\documentclass[twoside,11pt]{article}

\usepackage{jmlr2e_HongweiForDraft}

\usepackage{amsfonts}
\usepackage{textcomp}
\usepackage{amssymb}
\usepackage{marvosym}
\usepackage{amsmath}
\usepackage{verbatim}
\usepackage{pifont}

\usepackage{algorithm}
\usepackage{algorithmic}

\usepackage{hongwei}
\usepackage{booktabs}

\firstpageno{1}

\begin{document}

\title{Error Rate Bounds and Iterative Weighted Majority Voting for Crowdsourcing}


\author{ 
\AND
	\name Hongwei Li \email hwli@stat.berkeley.edu \\
    \addr Department of Statistics\\
    University of California\\
   Berkeley, CA 94720-1776, USA
    \AND
    \name Bin Yu \email binyu@stat.berkeley.edu \\
    \addr Department of Statistics \& EECS \\
    University of California\\
    Berkeley, CA 94720-1776, USA
}

\editor{}

\maketitle

\begin{abstract}

Crowdsourcing has become an effective and popular tool for human-powered computation to label large datasets. Since the workers can be unreliable, it is common in crowdsourcing to assign multiple workers to one task, and to aggregate the labels in order to obtain results of high quality.
In this paper, we provide finite-sample exponential bounds on the error rate (in probability and in expectation) of general aggregation rules under the Dawid-Skene crowdsourcing model. The bounds are derived for multi-class labeling, and can be used to analyze many aggregation methods, including majority voting,  weighted majority voting and the oracle Maximum A Posteriori (MAP) rule. 
We show that the oracle MAP rule approximately optimizes our upper bound on the mean error rate of weighted majority voting in certain setting. 
We propose an iterative weighted majority voting (IWMV) method that optimizes the error rate bound and approximates the oracle MAP rule. Its one step version has a provable theoretical guarantee on the error rate. The IWMV method is intuitive and computationally simple.  Experimental results on simulated and real data show that IWMV performs at least on par with  the state-of-the-art methods, and it has a much lower computational cost (around one hundred times faster) than the state-of-the-art methods.

\end{abstract}


\section{Introduction}

There are many tasks which can be easily carried out by people but that tend to be hard for computers, e.g., image annotation, visual design and video event classification.
When these tasks are extensive, outsourcing them to experts or well-trained people may be too expensive.
Crowdsourcing has recently emerged as a powerful alternative.
It outsources tasks to a distributed group of people (called workers) who might be inexperienced in these tasks. However, if we can appropriately aggregate the outputs from a crowd, the yielded results could be as good as the ones by experts
\citep{Smyth1995, Snow_emnlp08, Whitehill_nips09, Raykar_JMLR10, Welinder_nips10, Yan_icml10, Liu2012, Zhou2012}.

The flaws of crowdsourcing are apparent. Each worker is paid purely based on how many tasks that he/she has completed (for example, one cent for labeling one image).  No ground truth is available to evaluate how well he/she has performed on the tasks. So some workers may randomly submit answers independent of the questions when the tasks assigned to them are beyond their expertise. Moreover, workers are usually not persistent. Some workers may complete many tasks, while the others may  finish only very few tasks.

In spite of these drawbacks, is it still possible to get reliable answers in a crowdsourcing system? The answer is yes. In fact, majority voting (MV) has been able to generate fairly reasonable results \citep{ Snow_emnlp08}. However, majority voting treats each worker's result as equal in quality. It does not distinguish a spammer from a diligent worker. Thus majority voting can be significantly improved upon \citep{Karger_NIPS2011}.

The first improvement over majority voting dates back at least to \citep{Dawid_JRSS79}. They assumed that each worker is associated with an unknown confusion matrix, whose rows are discrete conditional distributions of input from workers given ground truth. Each off-diagonal element represents misclassification rate from one class to the other, while the diagonal elements represent the accuracy in each class. Based on the observed labels by the workers, the maximum likelihood principle is applied to jointly estimate unobserved true labels and worker confusion matrices. Although the likelihood function is non-convex, a local optimum can be obtained by using the Expectation-Maximization (EM) algorithm, which can be initialized by majority voting.

Dawid and Skene\rq{}s model \citep{Dawid_JRSS79} can be extended by assuming true labels are generated from a logistic model  \citep{Raykar_JMLR10}, or putting a prior over worker confusion matrices \citep{Liu2012}, or taking the task difficulties into account \citep{Bachrach_ICML12}. One may simplify the assumption made by \citet{Dawid_JRSS79} to consider a confusion matrix with only a single parameter \citep{Karger_NIPS2011, Liu2012}, which we call the \hds  model (Section \ref{sec:Notation}).


Recently, significant progress has been made for inferring the true labels of the items. 
\citet{Raykar_JMLR10} presented a maximum likelihood estimator (via EM algorithm) that infers worker reliabilities and true labels. 
\citet{Welinder_nips10} endowed each item (i.e., image data in their work) with features, which could represent concepts or topics, and workers have different areas of expertise of matching these topics. 
\citet{Liu2012} transformed label inference in crowdsourcing into a standard inference problem in graphical models, and applied approximate variational methods. 
\citet{Zhou2012} inferred the true labels by applying a minimax entropy principle to the distribution which jointly model the workers, items and labels. 
Some work also considers the problem of adaptively assigning the tasks to workers for budget efficiency \citep{Ho2013,Chen2013}.

All the previous work we mentioned above focused on applying or extending Dawid-Skene model, and inferring the true labels based on that. 
However, to understand the behavior and consequences of the crowdsourcing system, it is of great intension to investigate the error rate of various aggregation rules.
To theoretically analyze specific algorithm, 
\citet{Karger_NIPS2011} provided asymptotic error bounds for their iterative algorithm and also majority voting. It seems difficult to generalize their results to other \predrules in crowdsourcing or apply to finite sample scenario.
Very recently, \citet{Gao2014} studied the minimax convergence rate of the global maximizer of a lower bound of the marginal-likelihood function under a simplified \ds model (i.e., one coin model in binary labeling). Their results are on clustering error rate, which is different from the ordinary error rate, i.e., proportion of mistakes in final labeling. They focused on the mathematical properties of the global optimizer of a specific function for sufficiently large number of workers and items, and not on the behavior of rules/algorithms which find the optimizer or aggregate the results.

In this paper, we focus on providing finite sample bounds on the error rate of some general aggregation rules under crowdsourcing models of which the effectiveness on real data has been evaluated in \citep{Dawid_JRSS79,Raykar_JMLR10, Liu2012, Zhou2012}, and motivate efficient algorithms. Our main contributions are as follows:

\begin{enumerate}

\item 
We derived error rate bounds (in probability and in expectation) of a general type of aggregation rules with any finite number of workers and items under the Dawid-Skene model (with the \sds model and \hds model (Section \ref{sec:ProblemSetting}) as special cases). 

\item By applying the general error rate bounds to some special cases such as weighted majority voting and majority voting under specific models, we gain insights and intuitions. These lead to the \oborule for designing optimal weighted majority voting, and also the consistency property of majority voting.

\item We show that the oracle Maximum A Posteriori (MAP) rule approximately optimizes the upper bound on the mean error rate of weighted majority voting. The EM algorithm approximates the \omaprule, thus the error rate bounds can help us to understand the EM algorithm in the context of crowdsourcing.

\item 
We proposed a data-driven iterative weighted majority voting (IWMV) algorithm with performance guarantee on its one-step version (Section \ref{sec:MLE}). It is intuitive,  easy to implement and performs as well as the state-of-the-art methods on simulated and real data but with much lower computational cost. 


\end{enumerate}

To the best of our knowledge, this is the first work which focuses on the finite sample error rate analysis on general \predrules under the practical \ds model for crowdsourcing.  The results we obtained can be used for analyzing error rate and sample complexity of algorithms. It is also worth mentioning that most of the previous work done only focused on binary crowdsourcing labeling, while our results are based on multi-class  labeling, which naturally apply to the binary case. Meanwhile, we did not make any assumptions on the number of workers and items in the crowdsourcing, thus the results can be directly applied to the setting of real crowdsourcing data.

\section{Background and formulation}
\label{sec:ProblemSetting}
\label{sec:Notation}


As an example of crowdsourcing, we assume that a set of workers are assigned to perform labeling tasks, such as judging whether an image of an animal is that of a cat, a dog or a sheep, or evaluating if a video event is abnormal or not.


Throughout this paper, we assume there are $M$ workers and $N$ items for a labeling task with $\nL$ label classes. 
We denote  the set of workers $\M=\hua{1,2,\cdots, M}$, the set of items $\N=\hua{1,2,\cdots, N}$, and the set of labels $\Labset=\hua{1,2,\cdots, \nL}$ (called \emph{label set}). 
The extended label set is defined as $\Lextend=\Labset\cup \hua{0}= \hua{0,1,2,\cdots, \nL}$, where 0 represents the label is missing. However, in the case of $\nL=2$, we use the common convention of label set as $\hua{-1, +1}$ and extended label set as $\hua{0, -1, +1}$.

In what follows, we use $\yj$ as the true label for the $j$-th item, and $\hyj$ as the predicted label for the $j$-th item by an algorithm.\footnote{In this paper, any parameter with a hat $\hat{~}$ is an estimate for this parameter.}
Let $\pri_k= \P(\yj=k)$ denotes the prevalence of label ``$k$" in the true labels of the items  for any $ j\in [N]$ and $k\in\Labset$.

The observed \dataMatrix is denoted by $Z\in \Lextend^{M\times N}$, where $Z_{ij}$ is the label given by the $i$-th worker to the $j$-th item, and it will be $0$ if the corresponding label is missing (the $i$th worker did not label the $j$th item).
We introduce the indicator matrix $T=(\Tij)_{M\times N}$, where $\Tij=1$ indicates that entry $(i,j)$ is observed, and $\Tij=0$ indicates entry $(i,j)$ is unobserved.
Note that $T$ and $Z$ are observed together.

\begin{figure}[htb]
\centering
\includegraphics[width= 0.5\columnwidth]{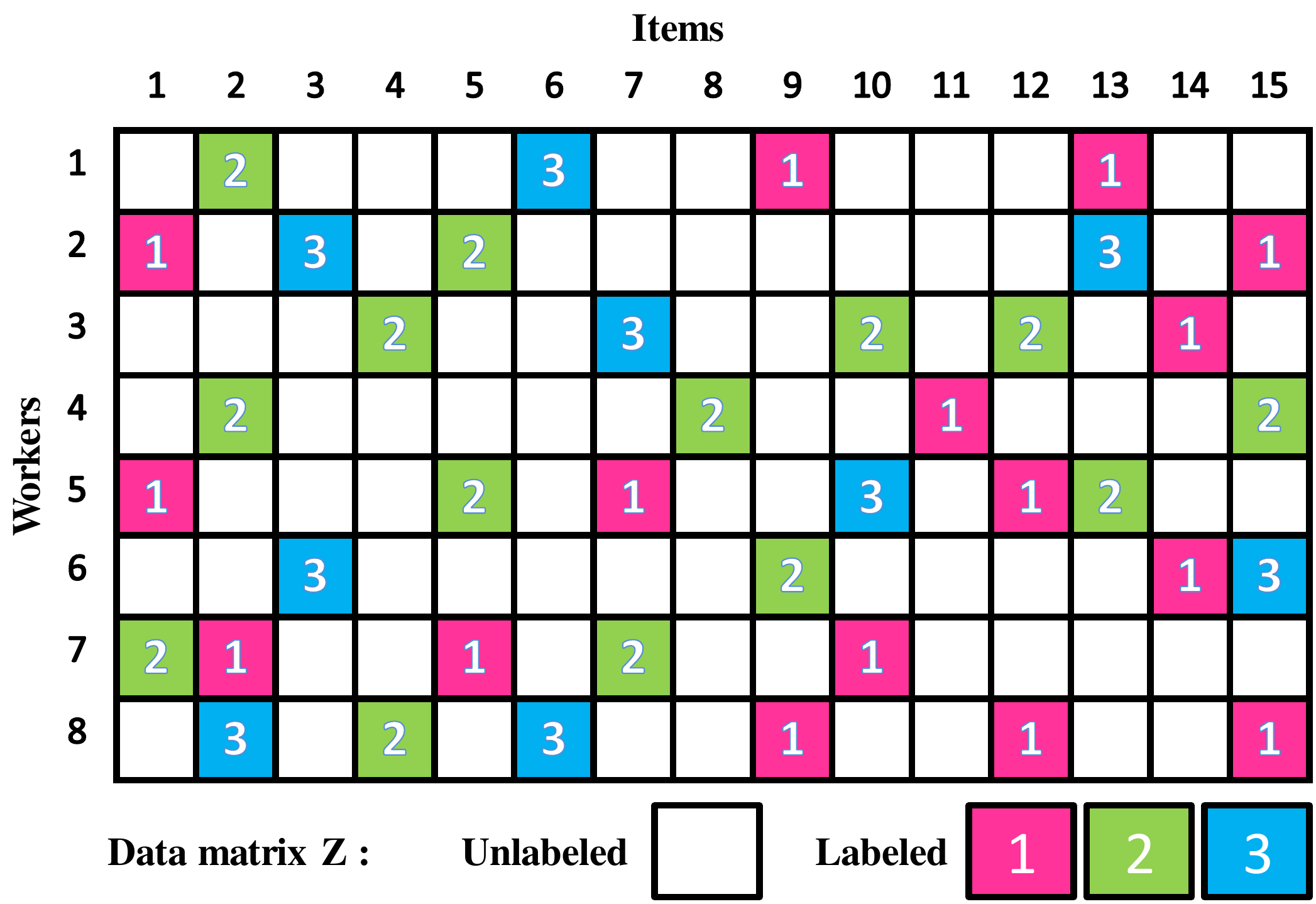}
\caption{Illustration of the input \dataMatrix. Entry $(i,j)$ is the label of $j$-th item given by  $i$-th worker. The set of labels is $\Labset= \hua{1,2,3}$.
}
\label{fig:labelMatrix}
\end{figure}

The process of matching workers with tasks (i.e., labeling items) can be modeled by a probability matrix $\Q= (\qij)_{M\times N}$, where $\qij= \P(\Tij=1)$ is the probability that the $j$th item was assigned to the $i$th worker (i.e., gets labeled). We call $\Q$ as \emph{\assign probability matrix}. 
Unlike the fixed assignment configuration in \citep{Karger_NIPS2011}, the \assign probability matrix is more flexible, and it does not require each worker label the same number of items, nor each item gets labeled by same number of workers. 
Hence, the \assign probability matrix covers the most general form of assigning items to workers, and there are special cases commonly adopted in literature, such as a worker has the same chance to label all items \citep{ Karger_NIPS2011, Liu2012}. More specifically, when $\qij= \qi\in(0,1], \forall i\in\M, j\in\N$, we call it the \emph{\assign probability vector $\qvec=(\q_1, \cdots, \q_M)$}. If $\qij=\qs \in(0,1], \forall i\in\M, j\in\N$, then we call it \emph{constant \assign probability $\q$.}  
The three assignment configurations above are referred to as the \emph{task assignment
\footnote{
The term \taskAssign seems to imply that workers are passive of labeling items --- they will surely label an item whenever they are assigned to.  This might not match the reality in the crowdsourcing platform. When a task owner distributed tasks to the crowd, it is likely that most of the workers will label a set of items and they can stop whenever they want to \citep{Snow_emnlp08}, unless they are required (by the owner) to complete a specific set of tasks for getting paid. Thus the process of matching tasks with workers might be determined by either workers (subjectively) or task owners (by enforcement), which depends on how the owners design and distribute tasks. If the workers have choice to select which item to label, it might be more proper to call the task-worker matching as \emph{task selection}, instead of \emph{\taskAssign}. However, both of the cases can be modeled by the probability matrix $\Q$ (or probability vector $\qvec$, or constant probability $\qs$). In what follows, we use the term \emph{\taskAssign} to represent the task-worker matching without introducing ambiguity. 
}
 is based on probability matrix $\Q$, probability vector $\qvec$ and constant probability $\qs$} , respectively.

  

Generally, we use $\pi, p,q$ as probabilities, and they might have indices according to the context. We denote $A$ and $C$ as constants which depend on other given variables, and $a$ as either a general constant or a vector depending on context.  
$\eta$ denotes likelihood probabilities in the context. $\para$ denotes a set of parameters. 
$\epsilon$ and $\delta$ are constants in $(0,1)$, where $\epsilon$ is used for bounding the error rate, and $\delta$ is used for denoting  a  positive probability.   $\H(\epsilon)$ denotes the natural entropy of Bernoulli random variable with parameter $\epsilon$, i.e., 
$\H(\epsilon)= -\epsilon\ln\epsilon - (1-\epsilon)\ln (1 - \epsilon)$. 
Operators $\minop$ and $\maxop$ denote the $\min$ operator and the $\max$ operator between two numbers, respectively. 
Meanwhile, throughout the paper, we will locally define each notation in the context before using them.

\subsection{Dawid-Skene models}\label{subsec:formulation}

We discuss three models covering all the cases that are widely used for modeling the quality of the workers \citep{Dawid_JRSS79, Raykar_JMLR10, Karger_NIPS2011, Liu2012, Zhou2012}. The first one, which is also the most general one, was originally proposed by  \citet{Dawid_JRSS79}:

\textit{\hwem{\gds model.~}}
In this model, the reliability of worker $i$ is modeled as a confusion matrix $\CMi=\kua{\cmikl}_{\nL\times\nL}\in [0,1]^{\nL\times \nL}$ , which is in a matrix form and represents a conditional probability table such that
\begin{eqnarray}\label{def:confusionMatrix}
\cmikl \defas \P\kua{\zij=l | \yj=k, \Tij=1}, \quad\forall k,l\in \Labset, \forall i\in\M.
\end{eqnarray}
Note that $\cmikk$ denotes the accuracy of worker $i$ on labeling an item with true label $k$ correctly, and $\cmikl, k\neq l$ represents the error probability of labeling an item with true label $k$ as $l$ mistakenly. The number of free parameters of modeling the reliability of a worker are $\nL(\nL-1)$ under the \gds model.

\begin{figure}[!htbp]
\centering
\includegraphics[width=0.4\columnwidth]{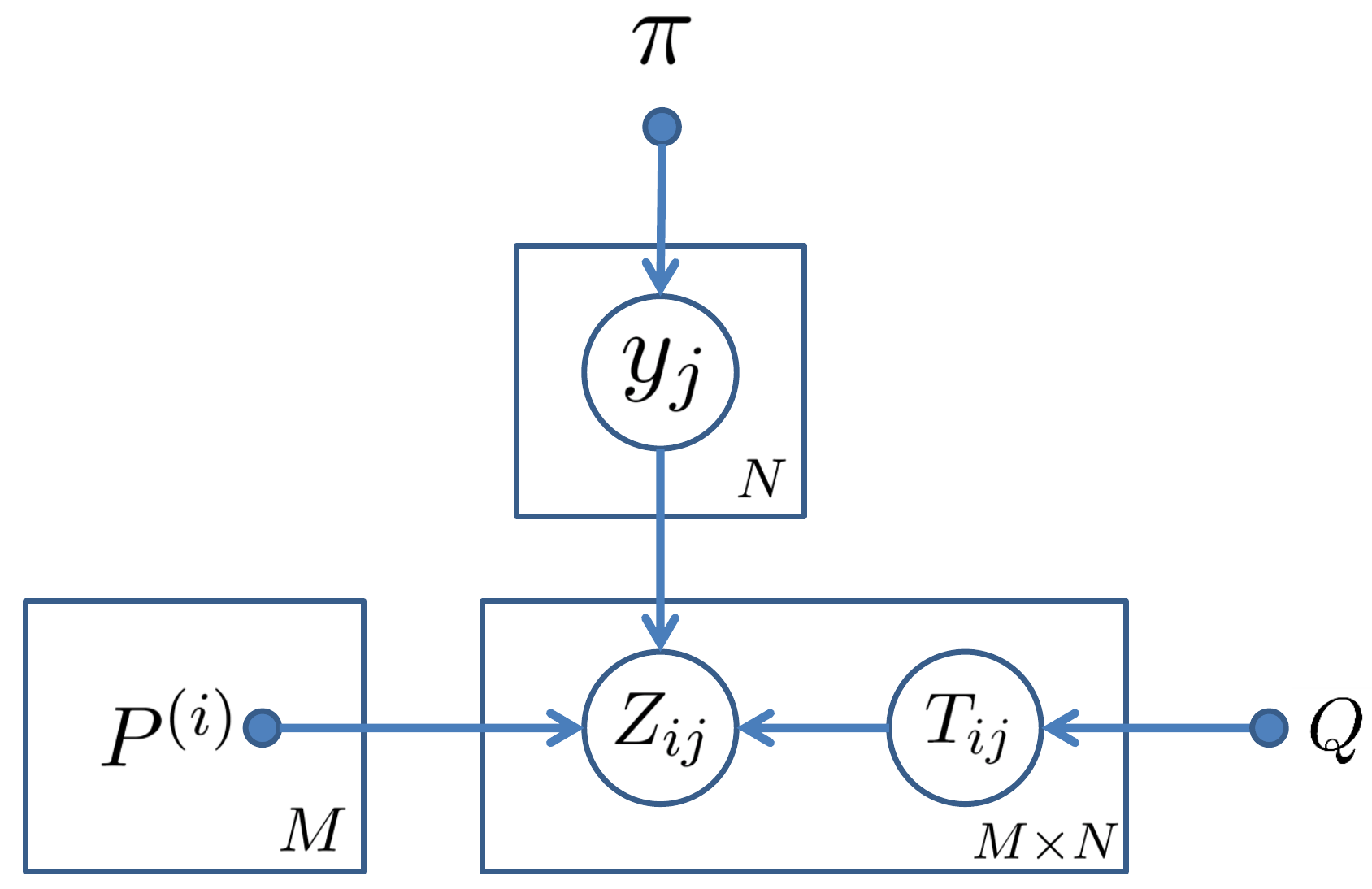}
\caption{The graphical model of the \gds model. Note that $\CMi$ is the confusion matrix of worker $i$ and if we change $\CMi$ to $\wi$, the graph will become the graphical model for the \hds model.}
\label{fig:graphModel}
\end{figure}

Since the \gds model has $\nL(\nL-1)$ degree of freedom to model each worker, it is flexible, but often leads to overfitting on small datasets. As a further regularization of the worker models, we consider another two models which are special cases of \gds model via imposing constraints on the worker confusion matrices.

\begin{itemize}

\item \emph{\sds model.~} In this model, the error probabilities of labeling an item with true label $k$ as label $l$ mistakenly for each worker are the same across different $l$. Formally, we have 
\begin{eqnarray}
\begin{cases}
\cmikk = \P\kua{\zij=k | \yj=k, \Tij=1}, & \quad\forall k\in \Labset, \forall i\in\M, 
\\
\cmikl = \frac{1-\cmikk}{\nL-1},  & \quad \forall k,l\in \Labset, l\neq k, \forall i\in\M.
\end{cases}
\end{eqnarray}
This model simplifies the error probabilities of a worker to be the same given the true label of an item. Thus, the off-diagonal elements of each row of the confusion matrix $\CMi$  will be the same. The number of free parameters to model the reliability of each worker under the \sds model is $\nL$.

\item \emph{\hds model.~}
Each worker is assumed to have the same accuracy on each class of items, and have the same error probabilities as well. 
Formally, worker $i$ labels an item correctly with a fixed probability $\wi$ and mistakenly with another fixed probability $\frac{1-\wi}{\nL-1}$, i.e., 
\begin{eqnarray}
\begin{cases}
\cmikk = \wi, & \qquad \forall k\in \Labset, \forall i\in\M,	
\\
\cmikl = \frac{1-\wi}{\nL-1}, & \qquad \forall k,l\in\Labset, k\neq l, \forall i\in\M.
\end{cases}
\end{eqnarray}

In this case, the worker labels an item with the same accuracy, independent of which label this item actually is. 
The number of parameters of modeling the reliability of each worker is 1 under the \hds model.

\end{itemize}

\begin{figure}[htb]
\centering
\includegraphics[width=0.6\columnwidth]
{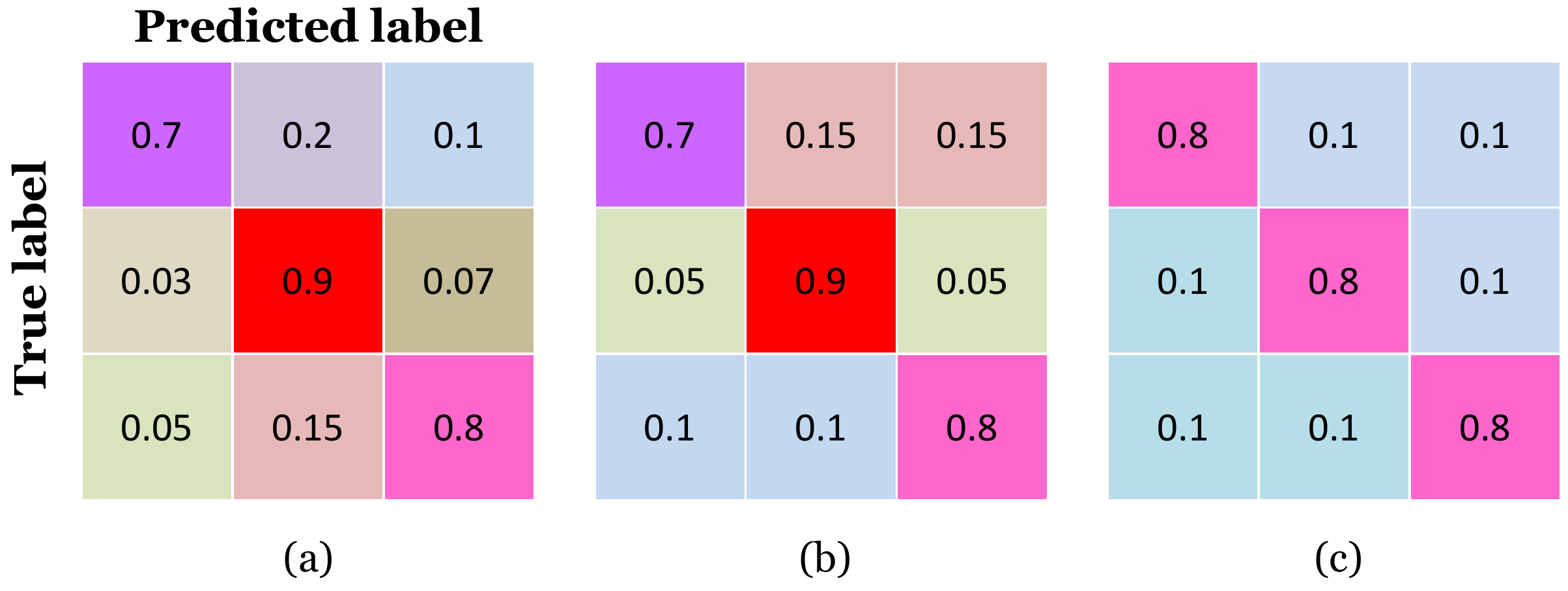}
\caption{ Toy examples of confusion matrices of different worker reliability models. We assume $\nL=3$ thus confusion matrices $\in [0,1]^{3\times 3}$. (a) \gds model. (b) \sds model. (c) \hds model. Vertical axis is actual classes, and horizontal axis is predicted classes. Different color corresponds to different conditional probability, and the diagonal elements are the accuracy of labeling the corresponding class of items correctly. }
\label{fig:confusionMatrix}
\end{figure}

Generally, the parameter set under all the three models can be denoted as 
$
\para = \hua{\hua{\CMi}_{i=1}^M, \Q, \pri}
$. Specifically, the parameter set of the \hds model can be denoted as
$
\para= \hua{\hua{w_i}_{i=1}^M, \Q, \pi}.
$

It is worth mentioning that when $\nL=2$, the \sds model and the \gds model are the same, which are referred to as the two-coin model, and the \hds model is referred to as the one-coin model in the literature \citep{Raykar_JMLR10, Liu2012}. 
In signal processing, the \hds model is equivalent to the random classification noise model \citep{Angluin_ML88}, and the \sds model is also referred to as the class-conditional noise model \citep{Natarajan_NIPS13}. We do not adopt the original term because the error comes from the limitations of workers' ability to label items correctly, not from noise which is in the context of signal processing. 


Binary labeling is the special case ($\nL=2$), and it is the major focus of the previous research in crowdsourcing \citep{Raykar_JMLR10, Karger_NIPS2011, Liu2012}. As a convention, we assume the set of labels is $\hua{\pm 1}$ instead of $\hua{1,2}$ when $\nL=2$.  For notation convenience, we defined the worker confusion matrix in a different way as follows: for $ i=1,2, \cdots, M$
\begin{eqnarray}
\begin{cases}\label{def:ppi}
\ppi = \P(\zijp1|\yjp1, \Tij=1), \\
\pni =	\P(\zijn1| \yjn1, \Tij=1).
\end{cases}
\end{eqnarray}
Then the parameter set will be
$
\para= \hua{\hua{\ppi,\pni}_{i=1}^M, \Q, \pi}
$ under this model. When we present the result of binary labeling, we will use $\ppi$ and $\pni$ without introducing any ambiguity.

Under the models above, the posterior probability of the true label of item $j$ to be $k$ is defined as:
\begin{eqnarray}
\rhojk= \P( \yj = k |Z,T,\para), \quad \forall j\in [N]. \label{def:rhojk}
\end{eqnarray}

For binary labeling, the posterior probability of the true label of item $j$ to be $k$ is defined as:
\begin{eqnarray}
\rhoj= \P( \yjp1 |Z,T,\para), \quad \forall j\in [N]. \label{def:rhojp}
\end{eqnarray}

\subsection{\Predrules}

After collecting the crowdsourced labels, the task owner could use an arbitrary rule to aggregate the multiple noisy labels of an item to a ``refined" label for that item.  The quality of final predicted labels depends not only on the input from the workers but also on the \predrule. It is hence of great importance to design a good \predrule. 

A natural \predrule is \emph{majority voting} \citep{Snow_emnlp08, Karger_NIPS2011}. For multiple labeling, majority voting can be written formally as 
\begin{equation}\label{def:mvRule}
\hyj = \argmax_{k\in \Labset} \sumi \I{\zij=k}, 
\end{equation}
which gives the $j$th item the majority label among all workers. 

Since workers have different reliabilities, it is inefficient to treat the labels from different workers with the same weight as in majority voting \citep{Karger_NIPS2011}. A natural extension of majority voting is \emph{weighted majority voting} (WMV), which weighs the labels differently. Formally, weighted majority voting rule can be written as
\begin{equation}\label{def:wmvRule}
\hyj = \argmax_{k\in \Labset} \sumi \vi\I{\zij=k},
\end{equation}
where $\vi\in \R$ is the weight associated with the $i$th worker. 

The idea behind weighted majority voting can be generalized in this way: given an item and the labels input from workers, the item can be potentially predicted to be any label class in $\Labset$. Suppose each class has a score which is computed based on the worker inputs, 
and we call it the \emph{aggregated score} for potential class
\footnote{Since we do not know the true label of an item, we use the term \emph{potential class} based on the fact that all the label class can potentially be the true label of this item.}
, then the aggregated label can be chosen as the class that obtains the highest score.
 
Based on the ideas above, we consider a general form of \predrule which is based on maximizing aggregated scores. 
The aggregated score of each label class can be decomposed into the sum of bounded prediction functions associated with workers, plus a shift constant. We refer this type of \predrule to \emph{decomposable \predrule}, which has the form
\begin{eqnarray}\label{def:predRuleFi}
\hyj= \argmax_{k\in\Labset} 
\sjk \connect \sjk\defas {\sumi \fikzij ~+ ~\ak},
\end{eqnarray}
where $\sjk$ is the \emph{aggregated score for potential class $k$ on item $j$},  $\ak \in \R$ and  $\f_i: \Labset \times \Lextend \rightarrow \R$ $~\forall i\in M$ is bounded, i.e., $|\f_i(k, h)|<\infty,$ $\forall k\in\Labset, h\in\Lextend$. Given $k$ and $h$, $\f_i(k,h)$ is a constant. Intuitively, $\f_i(k,h)$ is the score gained for the $k$th potential label class when the $i$th worker labels an item with label $h$. It is reasonable to assume that $\zij=0$ contributes no information to predict the true label of $j$th item, thus, we further assume that $\f_i(k,0)= \textrm{constant}, ~\forall i\in\M, k\in \Labset$. Without ambiguity, we will refer to $\hua{\f_1, \f_2, \cdots, \f_M}$ as \emph{score functions} in what follows. Note that score functions are usually designed by the task owners when they aggregate the noisy inputs into final predicted results. 

For illustration, majority voting is a special case of decomposable \predrule with $\fikzij= 1$ if $\zij=k$, and 0 if $\zij\neq k$. For weighted majority voting, $\fikzij= \vi$ if $\zij=k$, and 0 otherwise. Later, we will see more \predrules which can be expressed or approximated in this form (Section \ref{sec:MLE}).


\subsection{Performance metric}

Given an estimation or an \predrule, suppose that its predicted label for item $j$ is $\hyj$, then our objective is to minimize the error rate
\begin{eqnarray}
\errRate \= \inv{N}\sumj \I{\hyj\neq y_j}.
\end{eqnarray}
Since the error rate is random, we are also interested in its expected value (i.e., the \emph{mean error rate}). Formally, the mean error rate is:
\begin{eqnarray}
\E[\errRate] \= \inv{N}\sumj \P(\hyj \neq y_j).
\end{eqnarray}


The rest of the paper is organized as follows. In Section \ref{sec:Mainresults}, we present finite-sample bounds on the error rate of the decomposable \predrule in probability and in expectation under the \gds model.
Section \ref{sec:specialCases} contains the error rate bounds of some special cases, which the crowdsourcing community is widely concerned with. 
In Section \ref{sec:IWMV}, we propose an iterative weighted majority voting (IWMV) algorithm based on the analysis of the optimization of the error rate bounds, and provide performance guarantee for the one step verson of IWMV algorithm.
Experimental results on simulated and real-world dataset are presented in section \ref{sec:experiment}.
Note that the proofs are deferred to the supplementary materials.

\section{Error rate bounds}
\label{sec:Mainresults}
\label{sec:generalSetting}



In this section, finite sample bounds are provided for error rates (in high probability and in expectation) of the decomposable aggregation rule under the Dawid-Skene model.


Our main results will be focused on the setting which is as general as possible. We define the \emph{general setting} as follows:

\begin{itemize}

\item \emph{Worker modeling.} 
We focus on the \gds model, and then the results can be specialized for the \sds model and the \hds model straightforwardly. 


\item \emph{Task \assign.} We consider the \dataMatrix that is observed based on \assign probability matrix $\Q = (\qij)_{M\times N}$ ,where $\qij= \P(\Tij=1)$.
The results can be easily simplified to the scenarios where \taskAssign is based on probability vector $\qvec$ or constant probability $\qs$ according to the practical assignment process. 


\item \emph{\Predrule.} Our main results will be presented based on the decomposable \predrule (\ref{def:predRuleFi}).

\end{itemize}


\subsection{Some quantities of interest}\label{sec:quantities}

One important question we want to address is that how the error rate is bounded with high probability, and what quantities have impact on the bounds. 
Before deriving the error rate bound, we introduce some quantities of interest under the \emph{general setting} in this section. We shall bear in mind that all the quantities defined here serve the purpose of defining two measures $\tone$ and $\ttwo$, which play a central role in  bounding the error rate. 

The first quantity $\normalize$ is associated with  the score functions $\hua{\f_1, \f_2, \cdots, \f_M}$:
\begin{equation}
\normalize \defas \sqrt{\sumi \kua{ \max_{k,l,h\in\Labset,k \neq l}\abs{\muikh-\muilh} }^2  }. \label{eqn:defBeginGeneral}
\end{equation}
It measures the overall variation of the $\f_i$'s on their first argument (\ie, when the potential label class changes). Take weigthed majority voting (\ref{def:wmvRule}) as an example, $\muikh=\vi\I{h=k}$,  then $\normalize= \sqrt{\sumi \vi^2}= \vnorm$. For majority voting, $\normalize= \sqrt{\sumi 1}= \sqrt{M}$. Note that $\normalize$ is invariant to a translation of score functions, and is linear to scale of score functions. That is to say, if we design new score functions as $\f_i'= m\f_i + b$ for constant $m$ and $b$, and for all $i\in\M$, then the corresponding quantity $\normalize'= m\normalize$. 
Later on, we will see that it plays a role in normalization.

Another quantity $\gapjkl$ is defined as the \emph{expected gap of the aggregated scores} between two potential label classes $k$ and $l$, when the true label is $k$. Formally,  
\begin{equation}
\gapjkl \defas \E\braket{\sjk-\sjl|\yj=k} =\sumi\sumhL \qij\kua{\muikh-\muilh}\cmikh \+ \kua{\ak-\al}. \label{def:gapjkl}
\end{equation}
The larger this quantity is, the easier the \predrule identifies the true label as $k$ instead of $l$ (i.e., correctly predicted the label). Like $\normalize$, $\gapjkl$ is also invariant to translation of score functions, and linear to scale of score functions. 
Take weighted majority voting (\ref{def:wmvRule}) for illustration, the gap of aggregated scores under the \hds model is 
$
\gapjkl= \sumi\sumhL\kua{\qij\vi\wi -\qij\vi\frac{1-\wi}{\nL-1}}= \frac{1}{\nL-1}\sumi\qij\vi(\nL\wi-1),
$
because $\muikh=\vi\I{h=k}$ and $\cmikh=\wi$ if $h=k$, otherwise $\frac{1-\wi}{\nL-1}$.

The following two quantities serve as the lower bound and the upper bound of the normalized gap of aggregated scores for the $j$th item (i.e., the ratio of $\gapjkl$ and $\normalize$). 
\begin{equation}
\taujmin \defas \minkl \frac{\gapjkl}{\normalize}
\quad \connect \quad
\taujmax \defas \maxkl \frac{\gapjkl}{\normalize},\label{def:taujminmax}
\end{equation}
Both $\taujmin$ and $\taujmax$ are invariant to translation and scale of the score functions.

Now, we introduce the two most important quantities of interest --- $\tone$ and $\ttwo$, which are respectively the lower bound and the upper bound of the normalized gap of aggregated scores across all items. In our main results, $\tone$ is used to provide a sufficient condition for an upper bound on the error rate of crowdsourced labeling under the general setting. Meanwhile, $\ttwo$ is used to provide a sufficient condition for a lower bound of the error rate. 
\begin{equation}
\tone \defas \minj \taujmin
\quad \connect \quad
\ttwo \defas \maxj \taujmax, \label{def:toneAndttwo}
\end{equation}
Both $\tone$ and $\ttwo$ are invariant to translation and scale of the score functions $\hua{\f_1,\f_2,\cdots,\f_M}$.  $\tone$ and $\ttwo$ are related to how good a group of workers are, how the tasks are assigned and how well the \predrule is with respect to this type of labeling tasks.

Besides the quantities of interest above, we further introduce two notations which can capture the fluctuation of the score functions and the gap of aggregated scores. These two quantities will be used to bound the mean error rate of crowdsourced labeling. 

A quantity $\cH$ measures the maximum change amongst all the score functions when the potential label class changes from one label to another, and it is defined as
\begin{eqnarray}
\cH \= \inv{\normalize}\cdot\max_{i\in\M, k,l,h\in\Labset, k\neq l} \abs{\muikh-\muilh}. \label{def:cH}
\end{eqnarray} 
For weighted majority voting (\ref{def:wmvRule}), $\cH= \frac{\max_{i\in\M}|\vi|}{\vnorm}= \frac{\|\vweight\|_\infty}{\vnorm}$. And for majority voting (\ref{def:mvRule}), $\cH=\inv{\sqrt{M}}$.
Another quantity $\sigtwo$ relate to the variation of the gap of aggregated scores, and
\begin{eqnarray}
\sigtwo \= \inv{\normalize^2}\cdot\max_{j\in\N, k,l\in\Labset, k\neq l} \sumi\sumhL \qij\kua{\muikh-\muilh}^2 \cmikh ~~. \label{def:sigma2}
\end{eqnarray} 

Note that $0 < \cH\leq 1$, and $0<\sigtwo \leq \max_{i\in\M, j\in\N} \qij$. Both $\cH$ and $\sigtwo$ are translation and scale invariant of the score functions $\hua{\f_1,\f_2,\cdots, \f_M}$. 

In the next section, with $\tone$ and $\ttwo$, we will derive the bounds on the error rate of crowdsourced labeling with high probability, and together with $\cH$ and $\sigtwo$, we will derive the bounds on the mean error rate under the general setting.

\subsection{Main results}


In this section, we start with a main theorem to provide finite-sample error rate bounds for decomposable \predrules under the \gds model (\ref{def:confusionMatrix}). 

To lighten the notation, we define two functions as follows:
\begin{eqnarray}
&& \phi(x) = e^{-\frac{x^2}{2}}\qquad x\in\R , \label{def:phi}
\\
&& \D(x||y) =  x\ln{x\over y} + (1-x)\ln{{1-x}\over{1-y}} \qquad \forall x,y\in(0,1) .
 \label{eqn:defEndGeneral}
\end{eqnarray}
$\phi(\cdot)$ is the unnormalized standard Gaussian density function. 
$\D(x||y)$ is the Kullback-Leibler divergence of two Bernoulli distributions with parameters $x$ and $y$ respectively.


The following theorem provides sufficient conditions and the corresponding high probability bounds on the error rate under the \emph{general setting} as described in Section \ref{sec:generalSetting}. 

\begin{thm} \label{thm:MostGeneral_One}
\hwem{(Bounding error rate with high probability)}
Under the \gds model as in (\ref{def:confusionMatrix}), with the  prediction function for each item as in (\ref{def:predRuleFi}), and suppose the task assignment is based on probability matrix $\Q = (\qij)_{M\times N}$ where $\qij$ is the probability that the worker $i$ labels item $j$.

For $\forall \epsilon \in (0,1)$, with notations defined from (\ref{eqn:defBeginGeneral}) to (\ref{eqn:defEndGeneral}), we have:

(1) If \quad
$
\tone \geq \boundToneGeneral,
$
\quad
then 
\begin{eqnarray}
\errUpBoundEps \wpal
1- e^{- N D(\epsilon||(\nL-1)\phi(\tone))}.
\end{eqnarray}

(2) If \quad 
$
\ttwo \leq \boundTtwoGeneral,
$
\quad
then \vpull
\begin{eqnarray}
\errLowBoundEps \wpal
1- e^{- N D(\epsilon||1-\phi(\ttwo))}.
\end{eqnarray}
\end{thm}

\remark 
The high probability bounds on error rate require conditions on $\tone$ and $\ttwo$, which are related to the normalized gap of aggregated scores (Section \ref{sec:quantities}). Basically, if the scores of predicting an item as its true label (predicted correctly) are larger than the scores of predicting as a wrong label, then it is more likely that the error rate will be small, thus it is bounded from above with high probability. The interpretation of the lower bound and its condition is similar to the upper bound. 


To ensure the probability of bounding the error rate to be at least $1-\delta$, we have to solve the equation $D(\epsilon||(\nL-1)\phi(\tone)) = \inv{N}\ln\inv{\delta}$, which cannot be solved analytically. Thus, we need to figure out the minimum $\tone$ for bounding the error rate with probability at least $1-\delta$. The following theorem serves this purpose by slightly relaxing the conditions on $\tone$ and $\ttwo$ in Theorem \ref{thm:MostGeneral_One}.

Before presenting the next theorem, we define 
a notation $C$  
, which depends on parameters $\epsilon, \delta \in (0,1)$, for lightening notations in the theorem. 
\begin{eqnarray}
\C \= \constC, 
\end{eqnarray}
where $\H(\epsilon)= -\epsilon\ln\epsilon -(1-\epsilon)\ln(1-\epsilon)$, which is the natural entropy of a Bernoulli random variable with parameter $\epsilon$.

\begin{thm} \label{thm:MostGeneral_DeltaExplicit}
With the same notation as 
in Theorem \ref{thm:MostGeneral_One}, for $\forall \epsilon,\delta\in (0,1)$, we have:

(1) if 
$
\tone \geq \sqrt{2\ln\braket{(\nL-1)\C} },
$
then $\errorrate \leq \epsilon$ \wpal $1-\delta$, 

(2) if 
$
\ttwo \leq -\sqrt{2\ln \G},
$
then $\errorrate \geq \epsilon$ \wpal $1-\delta$. 

\end{thm}

\remark For any scenarios that can be formulated into the \emph{general setting} as in \ref{sec:generalSetting}, with $\tone$ and $\ttwo$ computed as in Section \ref{sec:quantities}, both Theorem \ref{thm:MostGeneral_One}  and  Theorem \ref{thm:MostGeneral_DeltaExplicit} can be applied. Therefore, in the rest of the paper, for any special case (Section \ref{sec:specialCases}) of the general setting , we will only present one of Theorem \ref{thm:MostGeneral_One}  and  Theorem \ref{thm:MostGeneral_DeltaExplicit}, and omit the other for clarity.



In practice, the mean error rate might be a better measure of performance because of its non-random nature. The method of evaluating the accuracy of a certain algorithm is often conducted by taking an empirical average of its performance in each trial, which is a consistent estimator of the mean error rate. Thus it will be of general interest to bound the mean error rate.

For the next theorem, we present the mean error rate bound under the general setting used in Theorem \ref{thm:MostGeneral_One}. 


\begin{thm}\label{thm:MostGeneral_meanErrorRate}
\hwem{(Bounding the mean error rate.)}
Under the same setting as in Theorem \ref{thm:MostGeneral_One}, with $\cH$ and $\sigtwo$ defined as in (\ref{def:cH}) and (\ref{def:sigma2}) respectively, 
\\
(1)
if~~$\tone \geq 0$,~~then \quad
$
\MER \leq (\nL-1)\cdot\min \hua{ \exp\kua{ - {\tone^2 \over 2}},~ \exp\kua{- \frac{\tone^2}{2\kua{\sigtwo + \cH\tone/3}} } },
$
\\
(2)
if~~$\ttwo \leq 0$,~~then \quad
$
\MER \geq 1- \min\hua{\exp\kua{-{\ttwo^2\over 2}},~ \exp\kua{- \frac{\ttwo^2}{2\kua{\sigtwo - \cH\ttwo/3}} } }.
$

\end{thm}

\remark The results above are composed of two exponential bounds, and neither is generally dominant over the other. Thus each component inside the min operator can be served as an individual bound for the mean error rate. Take the upper bound in (1) as an example, when $\tone$ is small (recall that both $\sigtwo$ and $\cH$ are bounded above by 1), the second component will be tighter than the first one. Thus the error rate bound behaves like $e^{-\tone}$. Otherwise, the first component will be tighter, and the mean error rate behaves like $e^{-\tone^2}$. 



All the proofs of the results in this section are deferred to Appendix \ref{app:proveTheorems}. In the following section, we demonstrate how the main results here can be applied to various settings of labeling by crowdsourcing, and provide theoretical bounds for the error rate correspondingly. 


\section{Apply error rate bounds to some typical scenarios}\label{sec:specialCases}


In this section, we apply the general results in Section \ref{sec:Mainresults} to some common  settings such as binary labeling, and \predrules such as majority voting and weighted majority voting. 
Specifically, we consider the following scenarios:


\begin{itemize}

\item Worker modeling: Recall that the \gds model is the same as the \sds model when $\nL=2$. We will cover the binary case with a certain \predrule under the \gds model. For $\nL>2$, 
we focus on multiclass labeling under the \hds model. 


\item Task assignment: We consider the scenario that the \taskAssign is based on probability vector $\qvec= (\q_1, \q_2, \cdots, \q_M)$ or a constant probability $\qs$.




\item \Predrule: We focus on  {weighted majority voting} (WMV) rule and  majority voting (MV) since they are intuitive and of common interest. In the case of binary labeling, we consider a general hyperplane rule. Later, we present results for Maximum A Posteriori rule. 



\end{itemize}

\subsection{Error rate bounds for hyperplane rule, WMV and MV}

It turns out that in practice, many prediction methods for binary labeling ($\yj\in\hua{\pm 1}$) can be formulated as a sign function of a  hyperplane, such as majority voting, weighted majority voting and the \omaprule (Section \ref{sec:MLE}). In this section, we are going to apply our results in Section \ref{sec:Mainresults} to discuss the error rate of the \predrule, whose decision boundary is a hyperplane in a high dimensional space that the label vector of each item (i.e., all the labels from the workers for this item) lies in. Formally, the \predrule is
\begin{eqnarray}\label{def:HyperplaneRule}
\hyj= \sign\kua{\sumi \vi \zij + \aconst}
\end{eqnarray}
This is called the \emph{general hyperplane rule} with unnormalized weights $\v= (\v_1, \cdots, \v_M)$ and shift constant $\aconst$. 
For binary labeling, majority voting is a special case with $\vi=1, \forall i\in \M$ and $\aconst=0$.



Theorem \ref{thm:MostGeneral_One} and Theorem \ref{thm:MostGeneral_meanErrorRate} can be directly applied to derive the error rate bounds of general hyperplane rule in the following corollary.

When \taskAssign is based on probability vector $\qvec$, the corresponding quantities of interest for the hyperplane rule (as in Section \ref{sec:quantities}) are defined as follows:

\begin{eqnarray}
&& \tone = \inv{\vnorm} \braket{ \kua{\sumi \qi\vi(2\ppi-1) + \aconst} \minop \kua{\sumi \qi\vi(2\pni-1) - \aconst} } \label{eqn:tone_hyperplane}
\\
&& \ttwo = \inv{\vnorm}\braket{ \kua{\sumi \qi\vi(2\ppi-1) + \aconst} \maxop \kua{\sumi \qi\vi(2\pni-1)- a} } 
\label{eqn:ttwo_hyperplane}
\\
&& \cH = \frac{||\vweight||_\infty}{\vnorm}
\connect
\sigtwo = \inv{\vnorm^2}\sumi \qi\vi^2
\label{eqn:sigtwo_hyperplane}, 
\end{eqnarray}
where $\minop$ and $\maxop$ are \emph{min} and \emph{max} operators respectively, $||\vweight||_\infty \defas \max_{i\in\M}|\vi|$, and functions $\phi(x)$ and $\D(x||y)$ are defined  in (\ref{def:phi}) and (\ref{eqn:defEndGeneral}).

\begin{cor} \label{cor:erbc_hyperPlane}
\emph{(Hyperplane rule in binary labeling)}
Consider binary labeling under the \gds model, i.e., $\yj\in \hua{\pm 1}$,  $\ppi$ and $\pni$ are defined as in (\ref{def:ppi}). 
 Suppose the \taskAssign is based on probability vector $\qvec = (\q_1, \q_2, \cdots, \q_M)$ where $\qi$ is the probability that worker $i$ labels an item,   
 \predrule is general hyperplane rule as in (\ref{def:HyperplaneRule}), and the quantities of interest are defined in (\ref{eqn:tone_hyperplane})-(\ref{eqn:sigtwo_hyperplane}), 
for $\forall \epsilon \in (0,1)$:
\\
(1) If \quad $\tone \geq 0$, then \quad 
$
\MER \leq \binMerUpBound. 
$

Furthermore, if ~
$
\tone \geq \boundTone,
$
~ then ~
$
\PErrUpExp .
$
\\
(2) If \quad $\ttwo \leq 0$, then \quad
$
\MER \geq \MerLowBound. 
$

Furthermore, if
$
\ttwo \leq \boundTtwo,
$
 then 
$ 
\PErrLowExp . 
$

\end{cor} 

\begin{proof}
The \predrule can be expressed as 
$
\hyj= \argmax_{k\in\hua{\pm 1}} \sumi \vi\I{\zij=k} + \ak, 
$
where $\aconst_+ = a$ and $\aconst_-=0$. 
Directly apply Theorem \ref{thm:MostGeneral_One} and Theorem \ref{thm:MostGeneral_meanErrorRate}, with $\muikk=\vi$ for $k\in\hua{\pm 1}$, $\muikl=0$ for $k\neq l$, and denote $\cmikl$ in terms of $\ppi$ and $\pni$, $\qij=\qi, \forall j\in\N$,  we can obtain the desired result immediately. 
\end{proof}

\remark From this corollary, we know that if $\tone$ is big enough, then the error rate will be upper bounded by $\epsilon$ with high probability. This means when the workers' reliabilities are generally good, it is very likely that we will have high quality aggregated results. 
Note that usually we can freely choose $\q$, $\v$ and $\aconst$, so the most important factors are worker reliabilities on positive samples ($\ppi$) and negative samples ($\pni$). 


If one needs to know, given $\epsilon$ and $\delta$, what are the results under the setting above corresponding to Theorem \ref{thm:MostGeneral_DeltaExplicit}, we can simply compute $\tone$ and $\ttwo$ as listed in Corollary \ref{cor:erbc_hyperPlane}, then the conditions and bounds in Theorem \ref{thm:MostGeneral_DeltaExplicit} hold.

As mentioned in the beginning of this section, weighted majority voting (WMV) is an important special case covered by our results in Section \ref{sec:Mainresults}. 
For the next several results, we focus on the mean error rate bound of WMV and MV under \hds model.

For simplicity, we consider the case where \taskAssign is based on constant probability $\qs$. Therefore, in the following corollary, the label of any entry $(i,j)$ is assumed to be revealed with constant probability $q$, and weighted majority voting (\ref{def:wmvRule}) is applied to obtain the aggregated labels in the end. 

\begin{cor}\label{cor:wmv_merBound}
\emph{(Weighted majority voting under \hds model)}
For weighted majority voting, whose prediction rule is $\hyj = \argmax_{k\in \Labset} \sumi \vi \I{\zij=k}$ with $\vweight=(\vweight_1, \vweight_2,\cdots, \vweight_M)$.  Assume the \taskAssign is based on constant probability $\qs\in (0,1]$, and assume the workers are modeled by the \hds model with parameter $\hua{\wi}_{i=1}^M$, then with
\begin{eqnarray}\label{def:tcsigtwoWMV}
\tsame = \frac{\qs}{(\nL-1)\vnorm}\sumi \vi(\nL \wi -1) 
,\quad
\cH= \frac{||\vweight||_\infty}{\vnorm}
\connect
\sigtwo = \qs
\end{eqnarray}

(1) if $\tsame \geq 0$, then 
$
\MER \leq (\nL-1)\min \hua{ \exp\kua{ - {\tsame^2 \over 2}},~ \exp\kua{- \frac{\tsame^2}{2\kua{\sigtwo + \cH\tsame/3}} } }
$

(2) if $\tsame \leq 0$, then 
$
\MER \geq 1- \min\hua{\exp\kua{-{\tsame^2\over 2}},~ \exp\kua{- \frac{\tsame^2}{2\kua{\sigtwo - \cH\tsame/3}} } }
$ 
\end{cor}

\begin{proof}
Under this setting, we have $\cmikk= \wi$ and $\cmikl= \frac{1-\wi}{\nL-1}$ when $k\neq l$. Meanwhile, $\muikk= \vi$ for any $k\in\Labset$ and $\muikl=0$ for $k,l\in\Labset, k\neq l$. Then $\normalize= \vnorm$. By plugging these above into the definitions of $\tone, \ttwo, \cH$ and $\sigtwo$ in Section \ref{sec:Mainresults}, we can then obtain the results.  
\end{proof}

\remark Note that $\tone = \ttwo =\tsame$ under the setting of this corollary. By replacing $\tone$ and $\ttwo$ with $\tsame$, the high probability bound on the error rate as in Theorem \ref{thm:MostGeneral_One} and Theorem \ref{thm:MostGeneral_DeltaExplicit} will hold as well. It is worth mentioning that the measure $t$ is of critical importance for the quality of the final aggregated results. It depends not only on the workers' reliability, but also on the weight vector $\vweight$, which can be chosen by us. This leaves room for us to study the best possible weight based on the bound we derived here. In Section \ref{sec:bound_optimal}, we will investigate the optimal weight in detail.

As a further special case of weighted majority voting and a commonly used prediction rule in crowdsourcing, the {majority voting} (MV) rule uses the same weight for each worker. It can be formally expressed as 
$
\hyj= \argmax_{k\in\Labset} {\sumi \I{\zij=k}}.
$
We can then directly obtain the error rate bounds of MV under the \hds model. 

\begin{cor}\label{cor:MV_mer_hds_q}
\emph{(Majority voting under \hds model)}
For majority voting with uniform random labeling distribution $\qs \in (0,1]$, and $\wbar= \inv{M}\sumi\wi$, 
if \quad $\wbar > \inv{\nL}$, \quad then 
\begin{eqnarray}\label{bound:MV1}
\MER \leq (\nL-1)\cdot\exp \hua{- \inv{2} \kua{\frac{\nL}{\nL-1}}^2 M \qs^2 \kua{\wbar - \inv{\nL}}^2 }.
\end{eqnarray}
Meanwhile, we  have
\begin{eqnarray}\label{bound:MV2}
\MER \leq (\nL-1)\cdot\exp \hua{- \frac{\inv{2} \kua{\frac{\nL}{\nL-1}}^2 M \qs \kua{\wbar - \inv{\nL}}^2 } 
{1+\frac{L}{3(\nL-1)}(\wbar-\inv{\nL})} }.
\end{eqnarray}
When $\qs < \frac{3}{4}$, the second upper bound (\ref{bound:MV2}) is tighter than the first one (\ref{bound:MV1}).

\end{cor}

\begin{proof}
For obtaining the upper bounds, we directly apply  Corollary \ref{cor:wmv_merBound} by letting $\vi=1$, then $\vnorm=\sqrt{M}$ and  note that $\sumi(\nL\wi-1)= \nL M (\wbar-\inv{\nL})$. By direct simplification, we get the desired result. 
The only difference in the two bounds is that the exponent in the second one is equal to the first one dividing by the factor $\alpha= \qs(1+\frac{L}{3(\nL-1)}(\wbar-\inv{\nL}))$. Since $\wbar\leq 1$, then $\frac{L}{3(\nL-1)}(\wbar-\inv{\nL}) \leq \inv{3}$, and $\alpha < 1$ when $\qs < \frac{3}{4}$. Therefore, $\qs < \frac{3}{4}$ implies that the second bound is tighter than the first one. 
\end{proof}
\remark (1) In real crowdsourcing applications, it is common that $\qs$ will be small enough due to the reasonable size of the available crowd \citep{Snow_emnlp08}. Thus the second bound will likely be tighter than the first one in practice. 
(2) The lower bound and its conditions can be easily derived similarly, so we omit them here since we are more interested in the ``possibilities" of controlling the error rate to be small. 


Due to the importance of majority voting in the crowdsourcing community, we are also interested in asymptotic properties of the bound. The following corollary discusses the case when $M \rightarrow \infty$, that is, when the number of workers who label items tends to infinity. 




\begin{cor}\label{cor:asympOneCoin}
\emph{(Majority voting in the asymptotic scenario)}
For the \hds model with \taskAssign based on constant probability $\qs\in (0,1]$, and let $\hyj$ be the predicted label for the $j$th item via majority voting rule (\ref{def:mvRule}), for any $N\geq 1$, 

(1) if \quad $\limM\wbar > \inv{\nL}$, \quad then \quad $\errorrate \goto 0$ in probability as $M \goto \infty$;

(2) if \quad$\limM\wbar < \inv{\nL}$, \quad then \quad $\errorrate \goto 1$ in probability as $M\goto \infty$;

(3) if \quad $\limM\wbar >\inv{\nL}$,~ then ~~ $\hyj \goto \yj$ in probablitliy as $M\goto \infty$, $\forall j\in \N$, i.e., majority voting is consistent. 
\end{cor}

\begin{proof}
The setting of this corollary is the same as in Corollary \ref{cor:MV_mer_hds_q}, and $\tone= \frac{\nL \qs}{\nL-1}\sqrt{M}(\wbar - \inv{\nL})$. When $\limM \wbar>\inv{\nL}$, there exists $\eta>0$ such that $\eta= \limM\wbar -\inv{\nL}$, which implies $\limM \tone = +\infty$, which further implies $\MER = 0$ for any $N\geq 1$ by Corollary \ref{cor:MV_mer_hds_q}. This is to say that $\limM \P(\hyj\neq \yj) = \limM \P(\I{\hyj\neq \yj}= 1)= 0$, which further implies that $\limM \P(\I{\hyj\neq \yj}= 0)= 1$. Note that for arbitrary $\epsilon>0$, we have $\P(\I{\hyj\neq \yj} < \epsilon)\geq \P(\I{\hyj\neq \yj}=0)$, then $\limM \P(\I{\hyj\neq\yj} \geq \epsilon)= 0$, i.e., $\I{\hyj\neq\yj} \goto 0, ~\forall j\in\N$ in probability when $M\goto \infty$ . By the properties of convergence in probability, $\errorrate \goto 0$ in probability when $M\goto \infty$. With the same argument, one can prove (2) as well.
To prove (3), note that $\P(\hyj-\yj=0)= \P(\I{\hyj\neq\yj}=0) \goto 0$ as $M\goto \infty$, which implies $\hyj\goto \yj$ in probability as $M\goto \infty$.
\end{proof}

\hwem{Remark:}
This corollary tells us that, if the average quality of the workers (i.e. the accuracy of labeling items) in the worker population, is better than random guess, then all the items can be labeled correctly with arbitrarily high probability when there are enough workers available.
The consistency property of majority voting for finite number of items ensures us that as long as there are enough reliable workers (better than random guess) available, we can achieve arbitrary accuracy even by a simple aggregating approach --- majority voting. 

The examples we covered in this section do not require the estimation of parameters such as worker reliabilities. In the next section, we discuss the maximizing likelihood  methods for inferring the parameters in crowdsourcing models, such as the celebrated EM (Expectation-Maximization) algorithm. Then we illustrate how our main results can be applied to analyze an underlining method that the ML methods approximate. 

\subsection{Error rate bounds for the Maximum A Posteriori rule}\label{sec:MLE}

If we know the label posterior distribution of each item defined in (\ref{def:rhojk}), then the Bayes classifier is
\begin{equation}\label{eqn:Bayes_Pred}
\hyj = \argmax_\kinL \rhojk,
\end{equation}
which is well known to be the optimal classifier \citep{Duda2012}. 

In reality, we do not know the true parameters of the model, thus the true posterior remains unknown. One natural way is to estimate the parameters of the model by Maximum Likelihood methods, further estimate the posterior of each label class for each item, and then build a classifier based on that. This is usually called the \emph{Maximum A Posteriori} (MAP) approach \citep{Duda2012}. 
After the EM algorithm estimating parameters, the Maximum A Posteriori (MAP) rule, which predicts the label of an item as the one that has the largest estimated posterior, can be applied. 
The prediction function of such a rule is
\begin{eqnarray}
\hyj =\argmax_\kinL \hrhojk,
\end{eqnarray}
where $\hrhojk$ is the estimated posterior. 
If the MAP rule is applied after the parameters are learned by the EM algorithm, then we call this method the \emph{EM-MAP rule}, and sometimes simply refer it to \emph{EM} without introducing any ambiguity in the context. 

However, the EM algorithm cannot guarantee convergence to the global optimum. Thus, the estimated parameters might be biased from the true parameters and the estimated posterior might be far away from the true one if it starts from a ``bad" initialization. 
Moreover, it is generally hard to study the solution of the EM algorithm, and thus it is relatively difficult for us to obtain the error rate for the EM-MAP rule. 

We consider the \emph{oracle MAP rule}, which assumes there is an oracle who knows the true parameters and uses the true posterior to predict labels.
Hence the oracle MAP rule is the Bayes classifier \eqref{eqn:Bayes_Pred}, and recall that its
prediction function is 
$
\hyj= \argmax_\kinL \rhojk,
$
where $\rhojk $ is the true posterior of $\yj=k$. 
Based on our empirical observations (Section \ref{sec:experiment}), the EM-MAP rule approximates the oracle MAP rule well in performance when most of the workers are good (better than random guess). In the next, we provide an error rate bound for the oracle MAP rule, which hopefully will help us understand the EM-MAP rule better. 


%


The following result is about the error rate bounds on the oracle MAP rule, and it can be straightforwardly derived from the main results in Section \ref{sec:Mainresults} since the \omaprule is a decomposable rule as in (\ref{def:predRuleFi}).

\begin{cor}\label{cor:oracleMAP_gds}
\emph{(Error rate bounds of the oracle MAP rule under the \gds model)} Suppose there is an oracle that knows the true parameters $\para=\hua{\hua{\CMi}_{i\in\M}, \Q, \pi}$ where $\CMi=\kua{\cmikh}_{k,h\in\Labset} \in(0,1]^{\nL\times\nL}$. The prediction function of the oracle MAP rule is $\hyj= \argmax_\kinL \rhojk$, where $\rhojk$ is the true posterior. 
All the error rate bounds in Theorem \ref{thm:MostGeneral_One}, \ref{thm:MostGeneral_DeltaExplicit} and \ref{thm:MostGeneral_meanErrorRate}  hold for the oracle MAP rule with $\muikh \defas \log\cmikh, ~\forall k,h\in\Labset$.
\end{cor}


As a special case of \gds model, the \hds model is relatively easy to visualize and simulate. Therefore, the results under this model will be useful in simulation.  The next corollary shows that the oracle MAP rule under the \hds model is weighted majority voting with class dependent shifts $\hua{\ak}_\kinL$ where $\ak=\log\prik$. For simplicity, we assume that the true labels of the items are drawn from uniform distribution (\ie, $\prik=\inv{\nL}$, balanced classes).

\begin{cor}\label{cor:msr_MAP_hds}
\emph{(The oracle MAP rule under \hds model)} Suppose the \taskAssign is based on a constant probability $\qs \in(0,1]$, and the prevalence of the true labels are balanced. Then the oracle MAP rule is a weighted majority voting rule under the \hds model with
\begin{eqnarray}
\hyj= \argmax_\kinL \sumi \vi\I{\zij=k} \text{~~~ where } \vi= \ln{\frac{(\nL-1)\wi}{1-\wi}}, ~~\forall i\in\M.
\end{eqnarray}
Let $\vweight=(\vweight_1, \vweight_2,\cdots,\vweight_M)$. The mean error rate of the oracle MAP rule is upper bounded without any conditions on $\tone$, i.e., for any $\hua{\wi}_{i\in\M}\in(0,1)^M$,  
$$
\MER \leq \LMerUpBound,
$$
where 
$
\tone =  \frac{\qs}{(\nL-1)\vnorm}\sumi\vi(L\wi-1),~~ ~~
 \cH= \frac{\|\vweight\|_\infty}{\vnorm}\connect \sigtwo= \qs.
$

\end{cor}


The results in the section above help us understand more about the practice of inferring the ground truth labels via maximum likelihood methods. The prominent \emmaprule approximates the \omaprule by estimating the parameters of crowdsourcing model and thus estimates the posterior distribution (\ref{def:rhojk}), then applies the MAP rule to predict the labels of items.  A further study on the error rate bounds of the \omaprule might be good for designing better algorithms with performance on par with the \emmaprule. This is the focus of the next section.

\section{Iterative weighted majority voting method}\label{sec:IWMV}

In this section, we first study  the mean error rate bound of weighted majority voting, then we minimize the bound to get the \oborule. Finally, we present its connection to the \omaprule. Based on the \oborule, we propose an iterative weighted majority voting method with performance guarantee on its one-step version. 

\subsection{The \oborule and the oracle MAP rule}
\label{sec:bound_optimal}

Here we explore the relationship between the oracle MAP rule and the mean error rate bound of WMV under the \hds model. We assume the \taskAssign is based on a constant probability $\qs$ for simplicity, and ignore the shift terms $\hua{\ak}$ in the \predrules (this is the case for the \omaprule when the label classes are balanced, Corollary \ref{cor:msr_MAP_hds}). 


The mean error rate bound of weighted majority voting (WMV) in Corollary \ref{cor:wmv_merBound} implies that if $\tone\geq 0$, then $\MER\leq \LMerUpBound$, where $\sigtwo= \qs$ and $\inv{\sqrt{M}}\leq \cH\leq 1$. Note that the impact of $\cH$ on the bound is marginal compared to that of $\tone$ since $\cH$ can be replaced by 1 to relax the bound slightly. 
At the same time, both functions $\exp\kua{ - {\tone^2 \over 2}}$ and $\exp\kua{- \frac{\tone^2}{2\kua{\sigtwo + \cH\tone/3} } }$ are monotonely decreasing w.r.t. $\tone \in [0,\infty)$. Thus the upper bound $\LMerUpBound$ is also monotonely decreasing w.r.t. $\tone\in[0,\infty)$. The mean error rate is bounded from above with the condition $\tone\geq 0$. Therefore,  maximizing $\tone$ will increase the chance of $\tone\geq 0$ being satisfied and reduce the bound to some extent. 
Recall that 
$$
\tone =\frac{\qs}{\nL-1} \sumi {\vi\over \vnorm}(\nL\wi - 1). 
$$ 
Since $\q$ is fixed now and we assume $\tone \geq 0$, so optimizing the upper bound is equivalent to maximizing $\tone$:
\begin{eqnarray}
&& \vstar
\quad=\quad  \argmax_{\v \in \R^M} \tone 
\quad=\quad   \argmax_{\v \in \R^M} \frac{\qs}{\nL-1} \sumi {\vi\over \vnorm}(\nL\wi - 1) 
\nn\\
&\then&  \hwem{The \oborule:} \quad \text{WMV with } \vstar_i\propto \nL\wi-1.
\label{eqn:chooseVA}
\end{eqnarray}
Therefore a bound-optimal strategy is to choose the weight for WMV as in (\ref{eqn:chooseVA}). This rule requires the information of the true parameters $\hua{\wi}_{i\in\M}$,  that is why we call it \emph{the \oborule}. In practice, we can estimate the parameters and plug $\hua{\hwi}_{i\in\M}$ into \eqref{eqn:chooseVA}, which we refer to as \emph{the \borule}.

By Corollary \ref{cor:msr_MAP_hds}, the oracle MAP rule under the \hds model is a weighted majority voting rule with weight 
$$
\vi^\ora= \log\frac{(\nL-1)\wi}{1-\wi}\approx \frac{\nL}{\nL-1}(\nL\wi-1).
$$
The approximation is due to the Taylor expansion around $x=\inv{\nL}$,
\begin{eqnarray}\label{eqn:tylor}
\ln{\frac{(\nL-1)x}{1-x}}= \frac{\nL}{\nL-1}\kua{\nL x - 1} + O\kua{\kua{x-\inv{\nL}}^2}.
\end{eqnarray}

Thus, \emph{the weight of the oracle bound-optimal rule is the first order Taylor expansion of the weight in the oracle MAP rule}.
Similar result and conclusion hold for the \sds model as well, but we omit them here for clarity.


By observing that the oracle MAP rule is very close to the oracle bound-optimal rule, the oracle MAP rule approximately optimizes the upper bound of the mean error rate. This fact also indicates that our bound is meaningful since the oracle MAP rule is the oracle Bayes classifier.

\subsection{Iterative weighted majority voting with performance guarantee}
\def \stepone{\hwem{ (Step 1) }}
\def \steptwo{\hwem{ (Step 2) }}
\def \stepthr{\hwem{ (Step 3) }}

Based on Section \ref{sec:bound_optimal}, the \oborule of choosing weights is $\vi\propto \nL (\wi -1)$. With this strategy, if we have an estimated $\wi$,  we can put more weights to the \lq\lq{}better\rq\rq{} workers and downplay the \lq\lq{}spammers\rq\rq{} (those workers with accuracy close to random guess). This strategy can potentially improve the performance of majority vote and result in a better estimate for $\wi$, which further improves the quality of the weights, and iterate. 
This inspires us to design an iterative weighted majority voting (IWMV) method as in Algorithm \ref{alg:IWMV}.

\begin{algorithm}[hbt] 
\caption{The iterative weighted majority voting algorithm (IWMV)}
\label{alg:IWMV}
\begin{algorithmic}

 \STATE {\bf Input:} { Number of workers= M; Number of items= N; \dataMatrix: $Z\in\Lextend^{M\times N}$;}
 \STATE {\bf Output:} {the predicted labels $\hua{\hat{y}_1, \hat{y}_2, ..., \hat{y}_N}$}

\STATE \hwem{Initialization:} $\vi= 1, ~~\forall i\in\M$; ~ $\Tij= \I{\zij\neq 0}, \forall i \in\M, \forall j\in\N$.
\REPEAT
\vspace{-8mm}
		\STATE 
		\begin{eqnarray*}
		&& \hyj \leftarrow \argmax_\kinL \sumi \vi\I{\zij=k}, \qquad\forall j\in\N. \\
		&& \hwi \leftarrow \frac{\sumj\I{\zij = \hyj}}{\sumj \Tij}, \qquad\forall i\in\M .\\
		&& \vi \leftarrow \nL\hwi - 1, \qquad \forall i\in\M .
		\end{eqnarray*}

\UNTIL{ converges or reaches $\nS$ iterations.}
\STATE \hwem{Output} the predictions $\hua{\hyj}_{j\in\N}$ by $\hyj=\argmax_\kinL \sumi\vi\I{\zij=k}$.
\end{algorithmic}
\end{algorithm}

 The time complexity of this algorithm is $O((M+L)N\nS)$, where $\nS$ is the number of iterations in the algorithm.
 Empirically, the IWMV method converges fast. But it also suffers from the local optimal trap as EM does, and is generally hard to analyze its error rate. However, we are able to obtain the error rate bound in the next theorem for a ``naive" version of it -- \emph{one-step WMV} (osWMV), which executes \stepone to \stepthr only once as follows:

\stepone Use majority voting to estimate labels, which are treated as the ``golden standard", i.e. $\hyjmv= \argmax_\kinL \sumi \I{\zij=k}$.

\steptwo Use the current ``golden standard" to estimate the worker accuracy $\hwi= \frac{\sumj\I{\zij=\hyjmv}}{\sumi \I{\zij\neq 0}}$ for all $i$ and set $\vi= \nL\hwi - 1$ for all $i$.

\stepthr Use the current weight $v$ in WMV to estimate an updated ``gold standard", i.e., $\hyj= \argmax_\kinL \sumi \vi\I{\zij=k}$.

For the succinctness of the result, we focus on the case where $\nL=2$, but the techniques used can be applied to the general case of $\nL$ as well.

\begin{thm}\label{thm:mer_IWMV_Bound}
\hwem{(Mean error rate bound of one step WMV for binary labeling)}
Under the \hds model, with label sampling probability $q=1$ and $\nL=2$, let $\yjwmv$ be the label predicted by one-step WMV for the $j$th item,  if $\wbar \geq \inv{2}+\inv{M}+\sqrt{\frac{(M-1)\ln2}{2M^2}}$, the mean error rate of one-step WMV 
\vspace{-2mm}
\begin{eqnarray} \label{eqn:oswmv}
\MERwmv \leq \finalUpBound,
\end{eqnarray}
\vspace{-2mm}
where $\pdiv= \sqrt{\inv{M}\sumi (\wpi-\inv{2})^2}$ and  $\seta= 2\halfseta$
\end{thm}
The proof of this theorem is deferred to Appendix \ref{app:oneStepWMVBound}. It is non-trivial to prove this theorem since the dependency among the weights and labels makes it hard to apply the concentration approach used in proving the previous results. Instead, a martingale-difference concentration bound has to be used.
\\

\hwem{Remarks:}
\begin{enumerate}
\item In the exponent of the bound, there are several important factors: $\pdiv$ represents how far away the accuracies of workers are from random guess, and it is a constant smaller than 1;  $\seta$ will be close to $0$ given a reasonable $M$.

\item The condition on $\wbar$ requires that $\wbar - \inv{2}$ is $\Omega(M^{-0.5})$, which is easier to satisfy with $M$ large if the average accuracy in the crowd population is better than random guess. This condition ensures that majority voting approximates the true labels.  Thus with more items labeled, we can get a better estimate of the workers\rq{} accuracies. The one-step WMV performance will then be improved with better weights.

\item We address how $M$ and $N$ affect the bound :
first, when both $M$ and $N$ increase but $\frac{M}{N}= r $ is a constant or decreases, the error rate bound  decreases. This makes sense because with the number of items labeled per worker increasing,   $\pih$ will be more accurate. The weights will be closer to the \oborule.
Second, when $M$ is fixed and $N$ increases, i.e., the number of items labeled increases, the upper bound on the error rate decreases.
Third, when $N$ is fixed and $M$ increases, the bound decreases when $M < \sqrt{N}$ and then increases when $M$ is beyond $\sqrt{N}$.
Intuitively, when $M$ is larger than $N$ and $M$ increases, the fluctuation of score functions, where $\pih$ is the estimated accuracy of the $i$th worker, will be large. This increases the chance of making more prediction errors. When $M$ is reasonably small (compared with $N$) but is increasing, i.e., more people label each item, the accuracy of majority voting will be improved according to Corollary \ref{cor:msr_MAP_hds}, then the gain on the accuracy of estimating $\pih$ results in the weights of the one-step WMV to be closer to the \oborule.

\end{enumerate}


As an alternative way of assigning weights to workers in each iteration (Alg.\ref{alg:IWMV}), we can also choose the weight of worker $i$ by plugging $\hwi$ into the weight in the \omaprule. That is, $\vi'= \log\frac{(\nL-1)\hwi}{1-\hwi}$. We refer this variant of IWMV to the \emph{\IWMVlog} algorithm.  From the practical point of view, however, $\vi'$ is unbounded and too large (or too small) if estimator $\hwi$ is close to 1 (or 0). 
Therefore \IWMVlog uses an aggressive way to weigh the workers, and it might be too risky when the estimates $\hua{\hwi}_{i\in\M}$ are noisy. 
Recall that given an estimate $\hwi$ of the reliability of worker $i$, the way that the IWMV algorithm chooses the weight is $\vi= \nL\hwi-1$. As a linearized version of $\vi'$ \eqref{eqn:tylor}, $\vi$ is more stable to the noise in the estimate $\hwi$. 
 Furthermore, IWMV is more convenient for theoretical analysis than the \IWMVlog  algorithm. In the next section, we will show some comparisons between IWMV and \IWMVlog by experiments. 
 


\section{Experiments}\label{sec:experiment}
\def \wideR{0.35}
\def \suo{-6mm}

In this section, we first compare the theoretical error rate bound with the error rate of \omaprule via simulation. Meanwhile, we compare IWMV with EM and \IWMVlog on synthetic data.  We then experimentally test IWMV and compare it with the state-of-art algorithms on real-world data. 
We implement majority voting, EM algorithm \citep{Raykar_JMLR10} with MAP rule (also referred as the EM-MAP rule in the experiments), and use public available code
\footnote{\tt http://www.ics.uci.edu/$\sim$qliu1/codes/crowd\_tool.zip}
 --- the iterative algorithm in \citep{Karger_NIPS2011} is referred to as KOS, and the variational inference algorithm from \citep{Liu2012} is referred to as LPI. 
All results are averaged over 100 random trials. 
All our experiments are implemented in Matlab 2012a, and run on a PC with Windows 7 operation system, Intel Core i7-3740QM (2.70GHz) CPU and 8GB memory. 

\subsection{Simulation}
\def \threeFigW{0.34}

The error rate of a crowdsourcing system is affected by variations of different parameters such as number of workers $M$, number of items $N$ and worker reliabilities $\hua{\wi}_{i\in\M}$ etc. To study how the error rate bound reflects the change of error rate when a parameter of the system changes, we conduct numerical experiments on simulated data for comparing the \omaprule with its error rate bound (Corollary \ref{cor:msr_MAP_hds}). We also measure the performance of the IWMV algorithm and compare it with the performance of \omaprule.


\begin{figure}[htb]
\begin{center}
\begin{tabular}{ccc}
\hspace{-3em}
\includegraphics[width=\threeFigW\columnwidth]{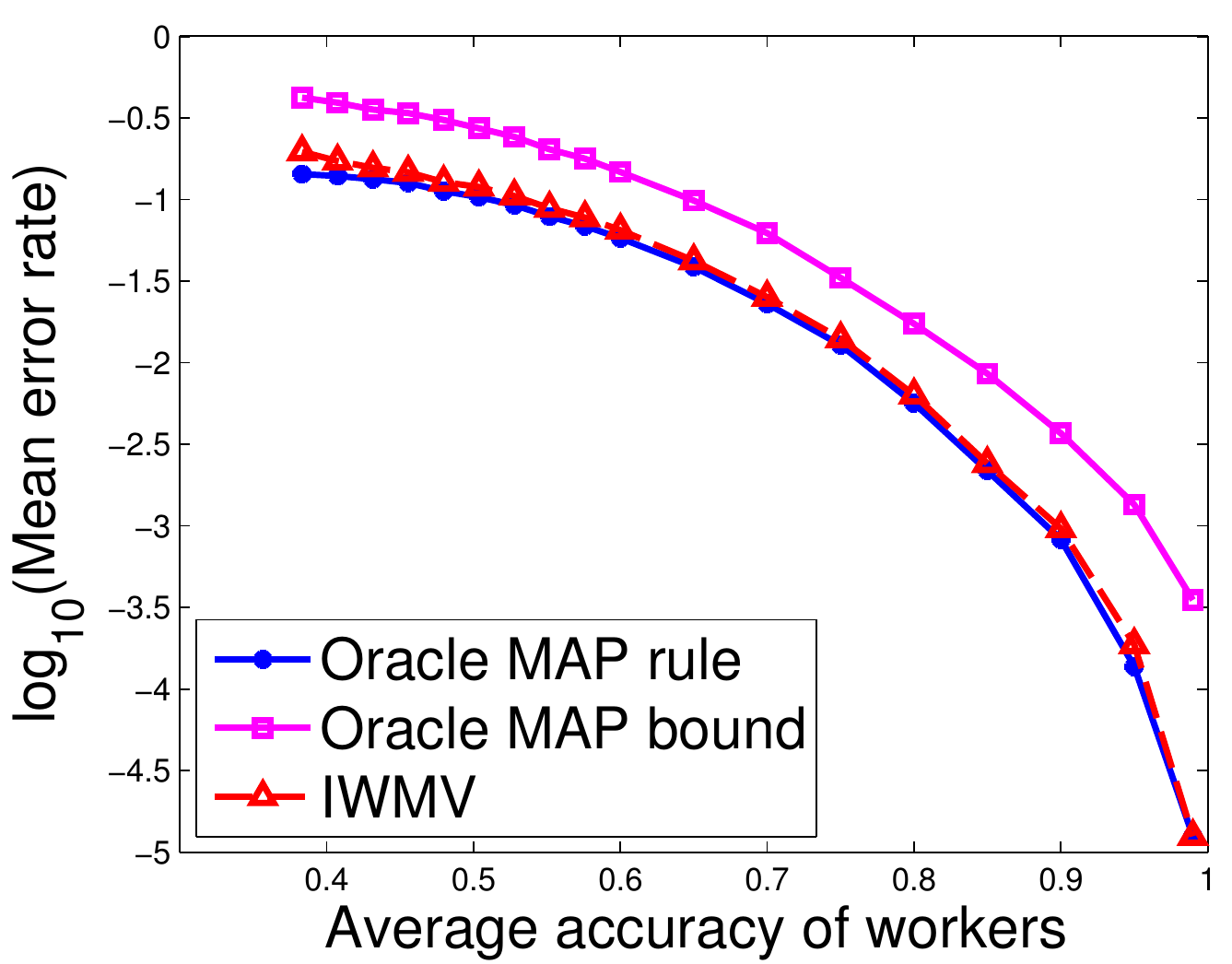} 
&
\includegraphics[width=\threeFigW\columnwidth]{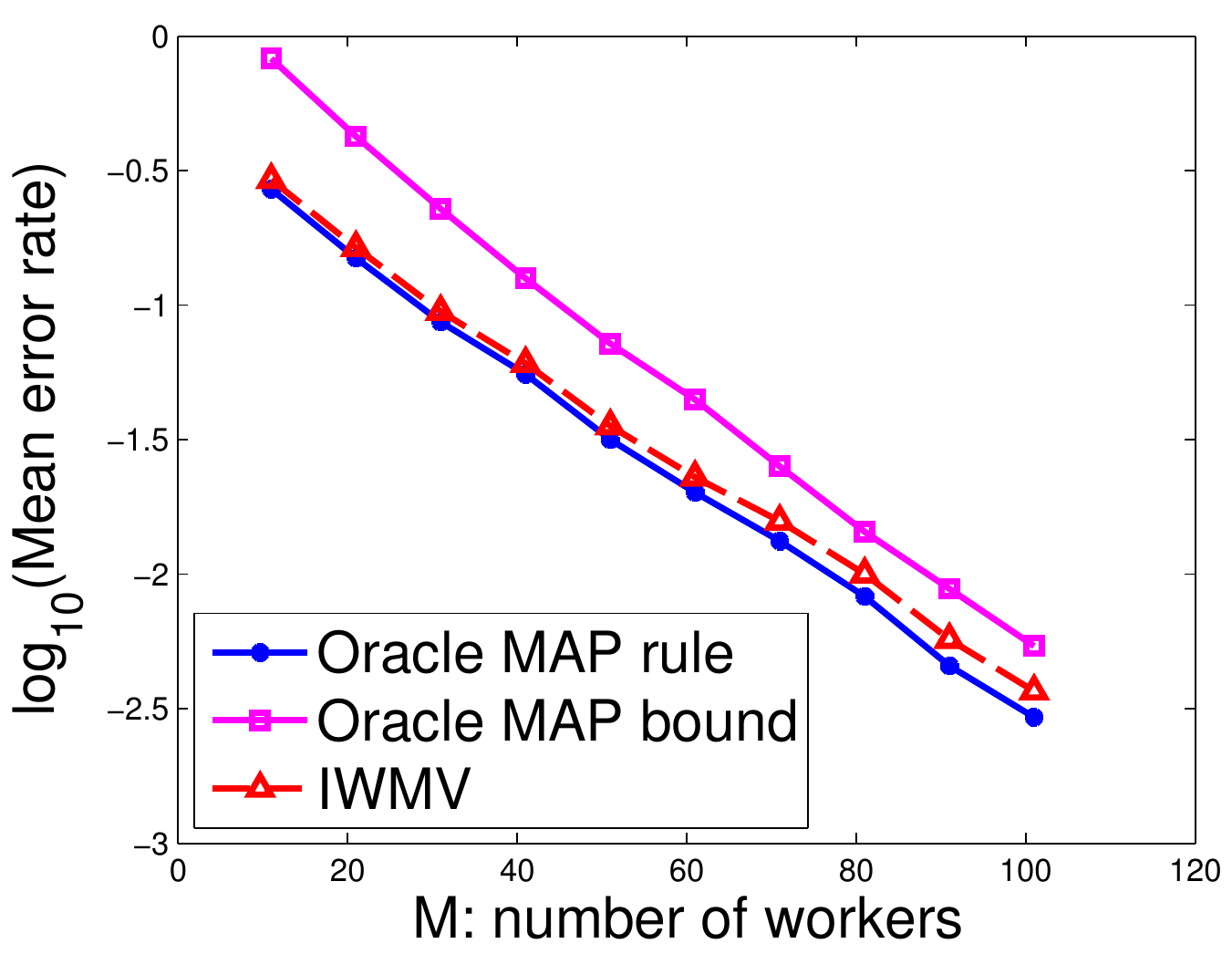}
&
\includegraphics[width=\threeFigW\columnwidth]{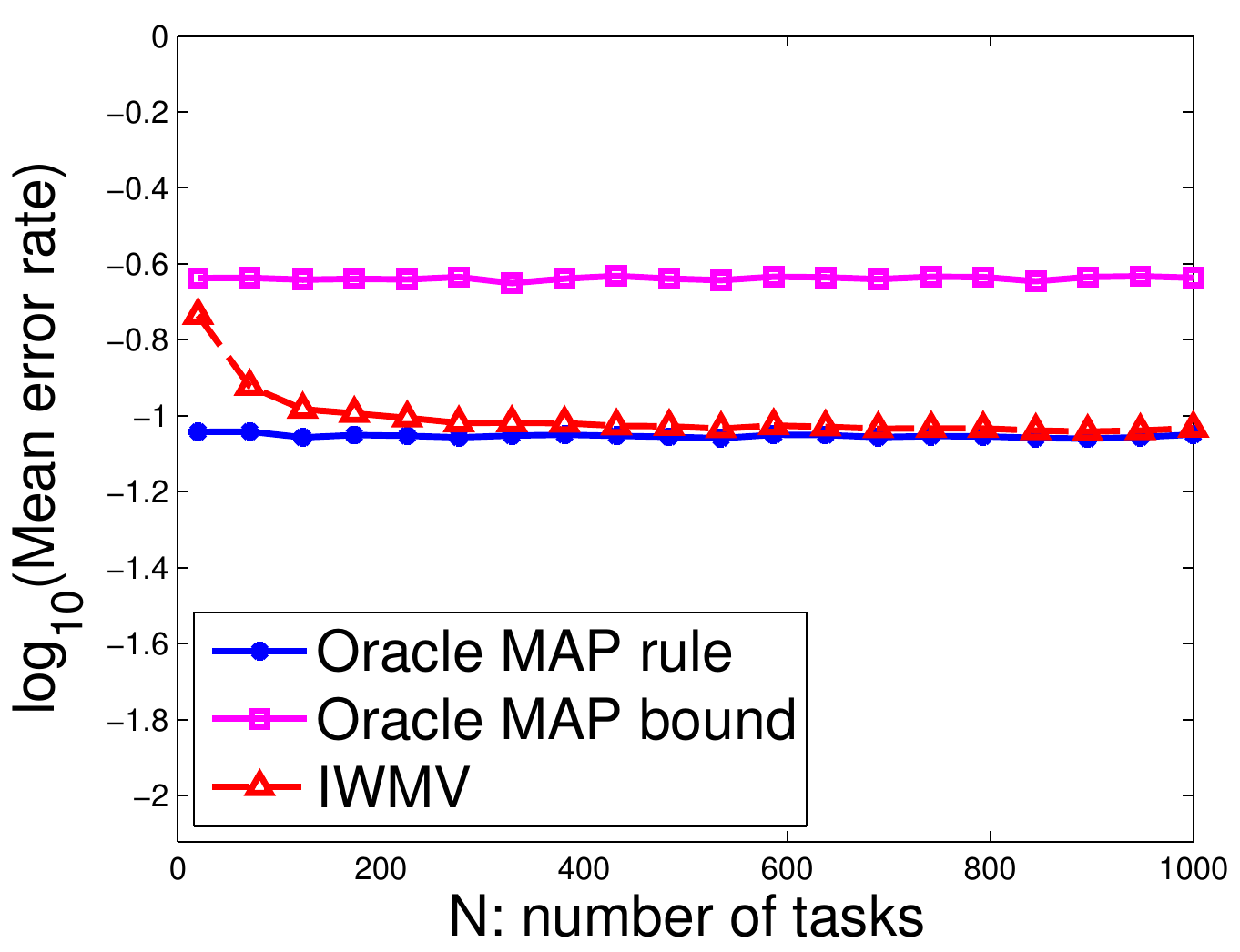} 
\\
(a) & (b) & (c)
\end{tabular}
\caption{ Comparing the \omaprule with its theoretical error rate bound by simulation. The performance of the IWMV algorithm is also imposed. These simulations are done under the \hds model with $\nL=3$ and $\qs=0.3$.  (a) Vary the average accuracy of workers and fix $M=31$ and $N=200$.(b) Vary $M$ and fix $N=200$. (c) Vary $N$ and fix $M=31$. The reliabilities of workers are sampled based on $\wi\sim $Beta(2.3, 2), $\forall i\in\M$ in (b) and (c). Note that all of them are in log scale and all the results are averaged across 100 repetitions. }
\label{fig:EM_MAPBound}
\end{center}
\end{figure}

The simulations are run under the \hds model.  Each worker has $\qs=$30\% chance to label any item which belongs to one of three classes ($\nL=3$). The ground truth labels of items are uniformly generated. The  accuracies of workers (\ie, $\hua{\wi}$) are sampled from a beta distribution \emph{Beta}$(a, b)$ with $b=2$. Given an expected average worker accuracy $\wbar$, we choose the paramater $a=\frac{2\wbar}{1-\wbar}$ so that the expected value for worker accuracies under distribution $Beta(a,2)$ matches with $\wbar$. In each random trial in the simulation, we keep sampling $M$ workers from this $Beta$ distribution until the average worker accuracy is within $\pm 0.01$ range from the expected $\wbar$. This is to maintain the average worker accuracy at the same level for each trial.  
 
First of all, we fix $M=31$ and $N=200$, then control the expected accuracy of the workers varies from 0.38 (slightly larger than random guess $1/L$) to 1 with a step size 0.05. 
The averaged error rates are displayed in Figure \ref{fig:EM_MAPBound}(a). Note that the error rate of the \omaprule is bounded by its mean error rate bound tightly (see Figure \ref{fig:EM_MAPBound}(a)). The bound follows the same trend of the true error rate of the \omaprule. The performace of IWMV converges to that of the \omaprule quickly as $\wbar$ increases.

By fixing $a=2.3$ and $\qs=0.3$, we then vary one of the two parameters--- number of items $N$ (default as 200) and number of workers $M$ (default as 31) --- with the other parameter maintained as the default. The corresponding results are presented in Figure \ref{fig:EM_MAPBound} (b) and (c), respectively. According to the results of simulation, the error rate bound of the \omaprule and its upper bound do not change when the number of tasks $N$ increases (Figure \ref{fig:EM_MAPBound}(c)), but they change log linearly when the number of workers $M$ increases (Figure \ref{fig:EM_MAPBound}(b)).
 Nevertheless, the performance of IWMV  changes whenever we increase $M$ or $N$. It behaves closely to the \omaprule when $M$ varies, but differently from the \omaprule when $N$ increases. This is because with more and more tasks done by the workers, the estimation of the reliability of the workers will be more accurate, and this can boost the performance of IWMV. However, the \omaprule knows the true worker reliability initially, so its performance will be independent with $N$.

\begin{figure}[htb]
\begin{center}
\includegraphics[width=0.8\columnwidth]{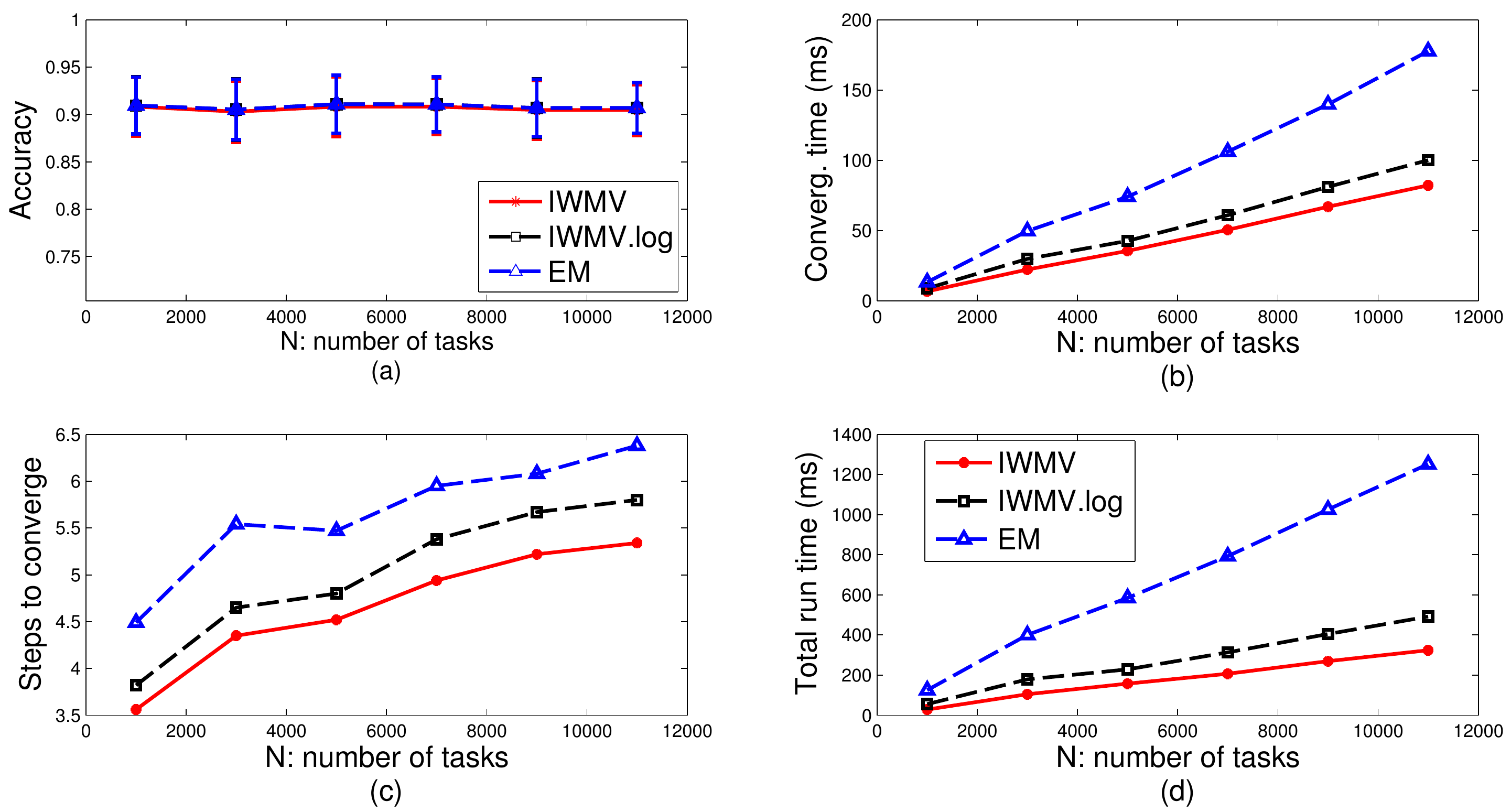}
\caption{Comparison between IWMV, \IWMVlog and the EM-MAP rule, with the number of items varying from 1000 to 11000. Simulation was performed with $\nL=3, M=31, q=0.3$ under the \hds model.  The reliabilities of workers are sampled based on $\wi\sim $Beta(2.3, 2), $i\in\M$.  All the results are averaged across 100 repetitions. (a) Final accuracy with error bar imposed. (b) The time until convergence, and we need to know the ground truth for measuring it. (c) Number of steps to converge. (d) Total run time is computed based on finishing 50 iterations.   }
\label{fig:results_IWMV_EM_changeN}
\end{center}
\end{figure}

Our next simulation (Figure \ref{fig:results_IWMV_EM_changeN}) shows that IWMV, its variant \IWMVlog and the EM algorithm achieve the same final prediction accuracy, while IWMV has the lowest computational cost. 
Specifically, we vary the number of tasks $N$, and compare the final accuracy with one standard deviation error bar imposed (Figure \ref{fig:results_IWMV_EM_changeN}(a)), the convergence time (Figure \ref{fig:results_IWMV_EM_changeN}(b)),
 the number of iterations (\ie steps) to converge (Figure \ref{fig:results_IWMV_EM_changeN}(c)) and the total run time for 50 iterations (Figure \ref{fig:results_IWMV_EM_changeN}(d)). With almost the same accuracies, IWMV converges faster and takes less steps to converge than EM and \IWMVlog, and the run time of IWMV is prominently lower than EM. Similar conclusions can be also confirmed when changing $M$, $\qs$ and $\nL$, thus we omit the similar results here.

\begin{figure}[htb]
\begin{center}
\begin{tabular}{cc}
\textit{\raisebox{.25\height}
{\includegraphics[width=0.35\columnwidth]{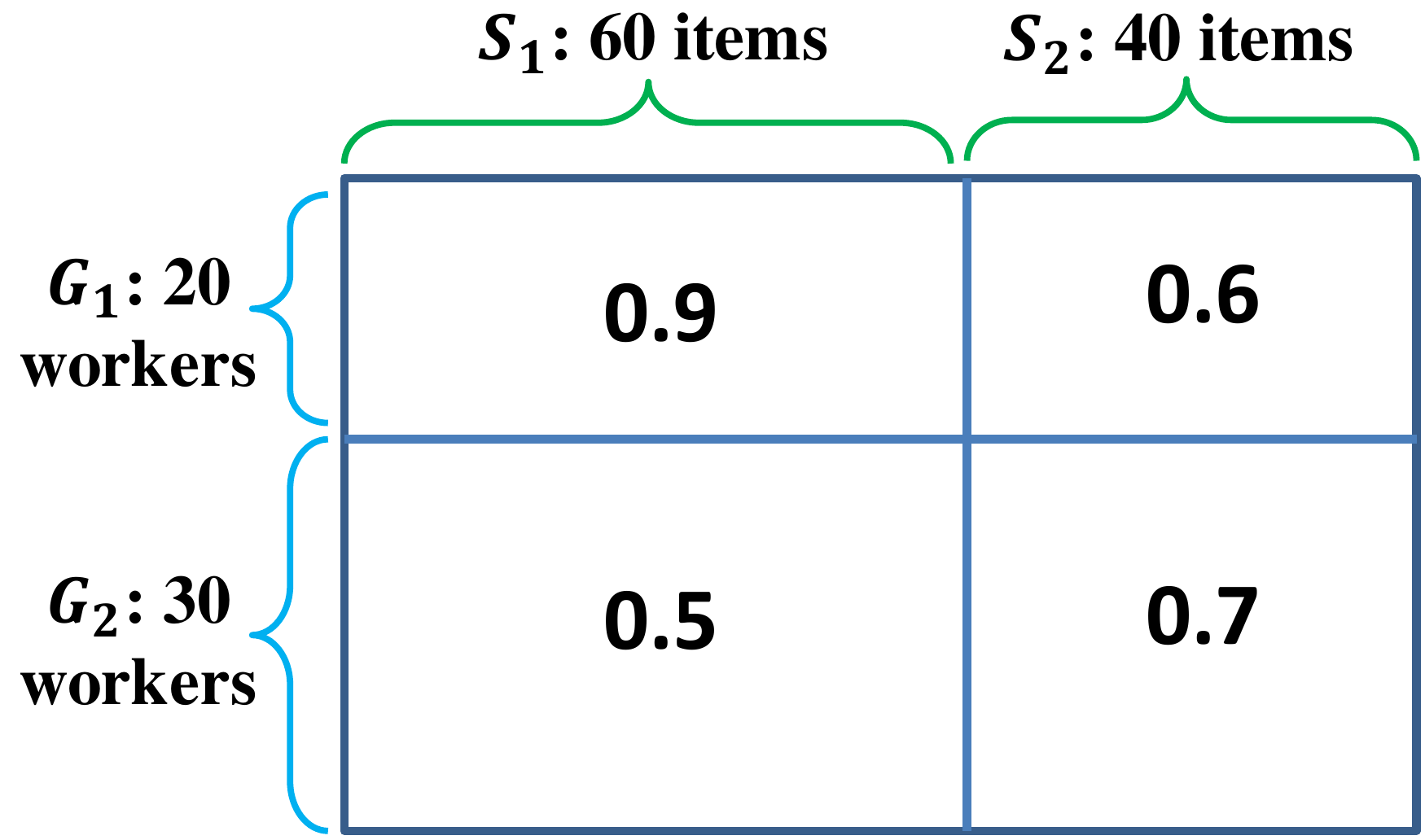}
}
}
&
\includegraphics[width=0.45\columnwidth]{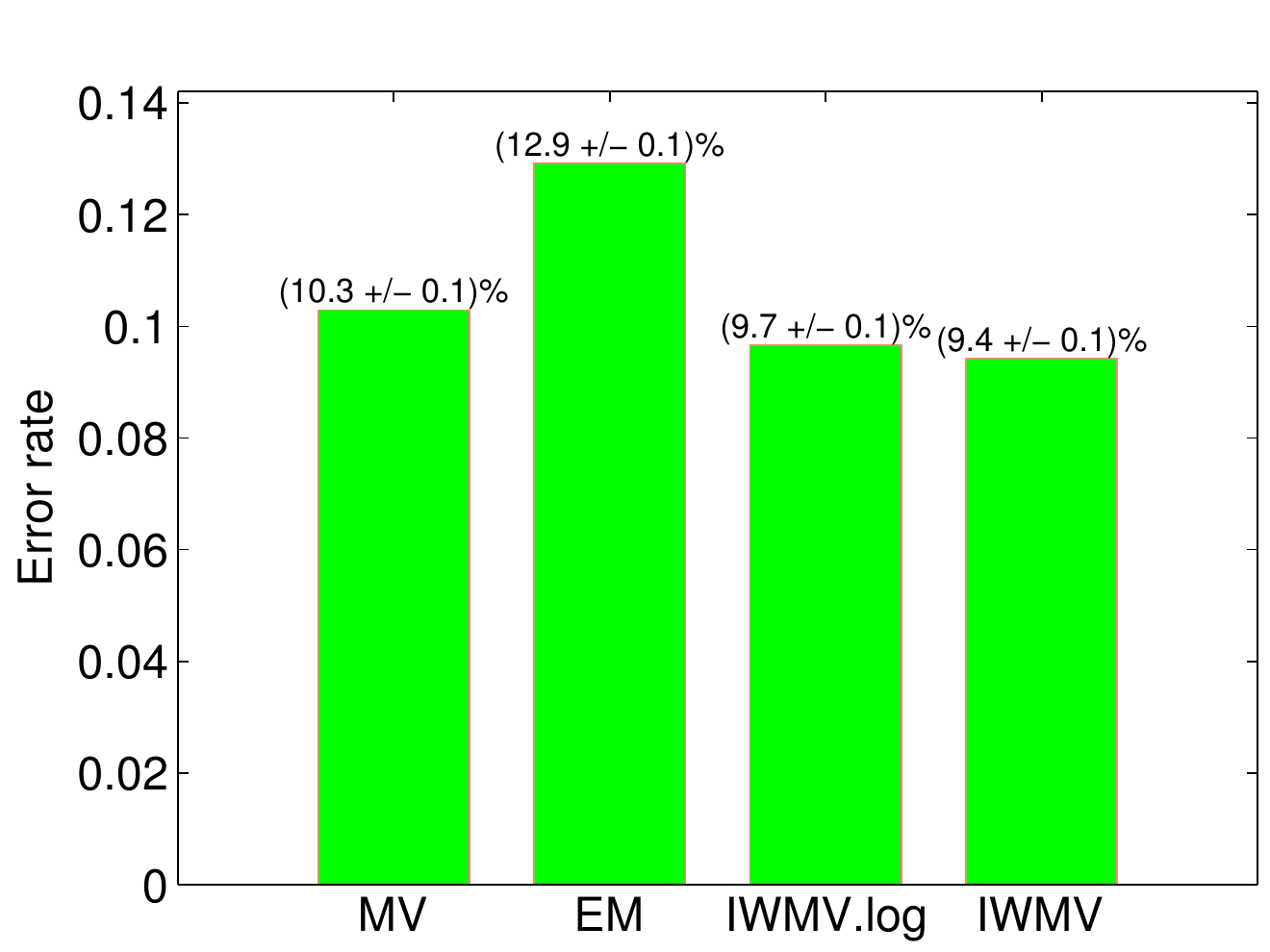}
\\
(a) & (b)
\end{tabular}
\caption{
(a) Setting of model misspecification. (b) Performance comparison under model misspecification setting in (a).}
\label{fig:model_misspecify}
\vspace{-3mm}
\end{center}
\end{figure}

The experiments above are strictly simulated based on the \hds worker model for its simplicity.
To compare IWMV with EM when the worker model is violated, we simulated toy data. 
See Figure \ref{fig:model_misspecify}(a) for the setup: suppose there are two group of workers $\Gone$ and $\Gtwo$, and two sets of items $\Sone$ and $\Stwo$. The true labels of items are generated uniformly from $\hua{\pm 1}$. The data matrix is generated as follows: 
$\P(\zij=\yj)= 0.9$ if $i\in \Gone, j\in\Sone$; $\P(\zij=\yj)= 0.6$ if $i\in \Gone, j\in\Stwo$; $\P(\zij=\yj)= 0.5$ if $i\in \Gtwo, j\in\Sone$; $\P(\zij=\yj)= 0.7$ if $i\in \Gtwo, j\in\Stwo$. 
We use $\qs=0.3$. 
The error rate (with one standard deviation) of MV, EM, \IWMVlog and IWMV are shown in Figure \ref{fig:model_misspecify} (b), which shows that IWMV achieves lower error rate than EM and \IWMVlog do in this model misspecification example. 
The results in Figure \ref{fig:model_misspecify} shows that IWMV are more robust than EM under model misspecification to some extent. Similar results can be obtained under other different configuations (as in Figure \ref{fig:model_misspecify}), and we omit them here. 

\subsection{Real data}

To compare our proposed \iwmv with the state-of-the-art methods \citep{Raykar_JMLR10, Karger_NIPS2011, Liu2012}, we conducted several experiments on real data (most of them are publicly available). We sampled the collected labels in the real data independently with probability $\samq$ which varies from 0.1 to 1, and see how the error rate and run time change accordingly. 
For clarity of the figures produced in this section, we omit the results of \IWMVlog since its performance is usually worse than IWMV in terms of both accuracy and computational time. Our focus will be the comparisons among IWMV, EM, KOS and LPI.  


\begin{table}[htb]
\caption{The summary of datasets used in the real data experiments. 
$\wbar$ is the average worker accuracy.
}\label{tab:datasets}
\begin{center}
\begin{tabular}{llllll}
\toprule
\multicolumn{1}{l}{Dataset}  
& \multicolumn{1}{l}{$\nL$ classes}
& \multicolumn{1}{l}{$M$ workers}
& \multicolumn{1}{l}{$N$ items}
& \multicolumn{1}{l}{\#labels}
& \multicolumn{1}{l}{$\wbar$}
\\ 
\midrule 
Duchenne	&2	&17		&159	&1221	&65.0\%	\\
RTE			&2	&164	&800	&8000	&83.7\%	\\
Temporal	&2	&76		&462	&4620	&84.1\%	\\
Web search	&5	&177	&2665	&15539 	&37.1\%	\\
\bottomrule
\end{tabular}
\end{center}
\end{table}

\emph{Duchenne dataset.} The first dataset is from \citep{Whitehill_nips09} on identifying Duchenne smile from non-Duchenne smile based on face images. In this data, 159 images are labeled with \emph{\{Duchenne, non-Duchenne\}} labels by 17 different Mechanical Turk\footnote{https://www.mturk.com} workers. In total, there are 1,221 labels, thus 1221/(17$\times$159)= 45.2\% of the potential task assignments are done (i.e., 45.2\% of the entires in the data matrix are observed). The ground truth labels are obtained from two certified experts and 58 out of the 159 images contain Duchenne smiles. The Duchenne images are hard to identify (the average accuracy of workers on this task is only 65\%).

\begin{figure}[htb]
\begin{center}
\begin{tabular}{ccc}
\hspace{\suo}
\includegraphics[width=\wideR\columnwidth]{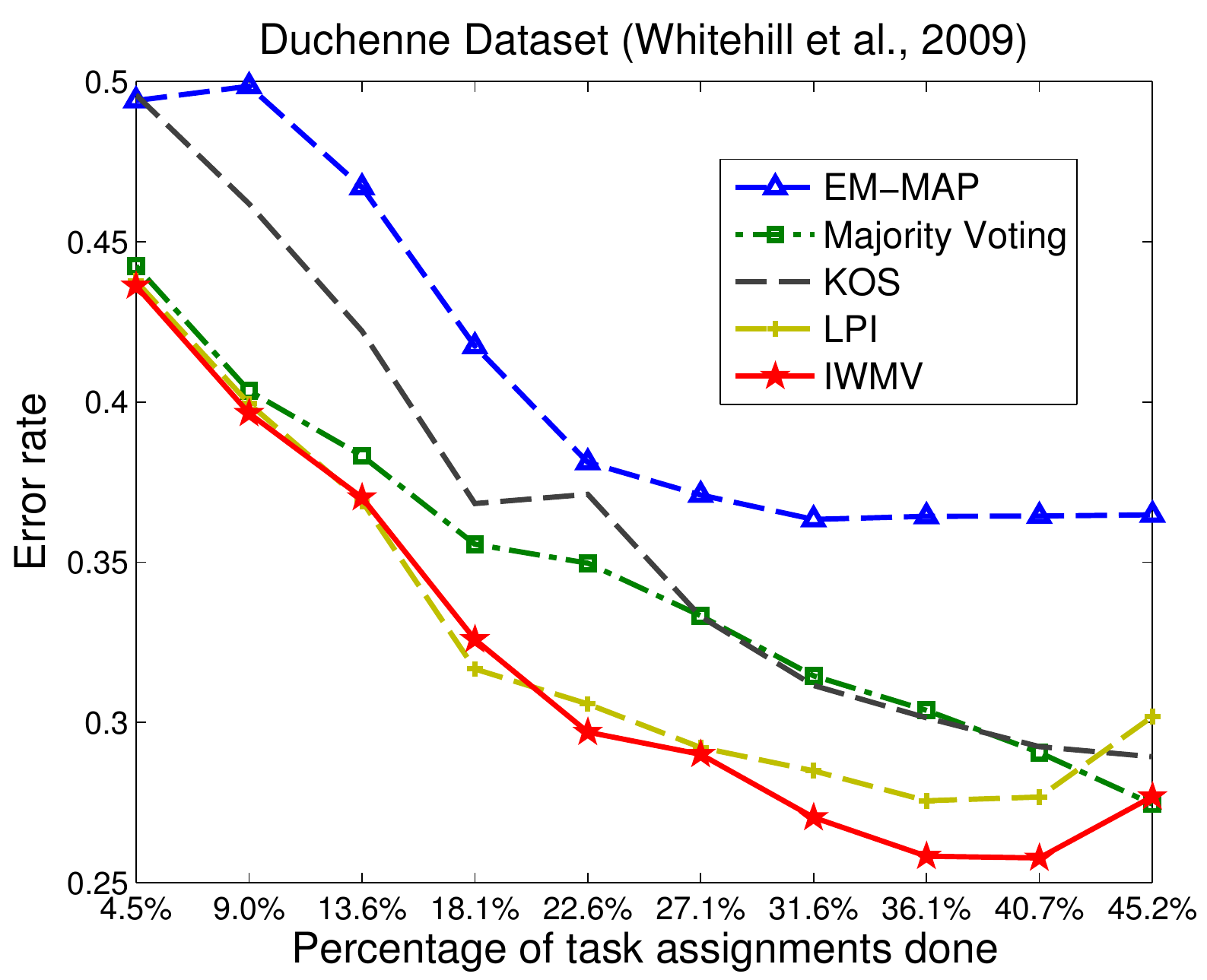}
&
\hspace{\suo}
\includegraphics[width=\wideR\columnwidth]{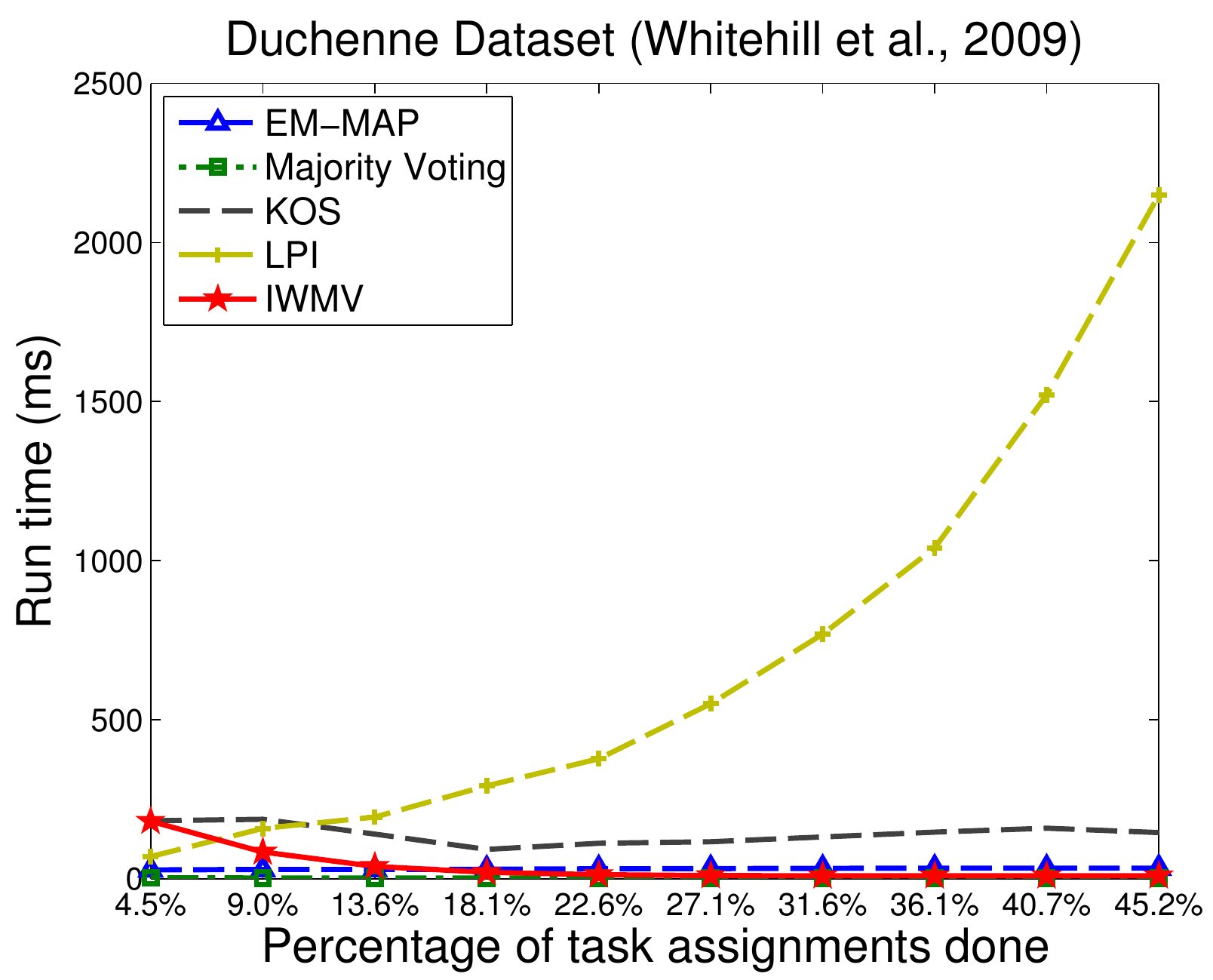}
&
\hspace{\suo}
\includegraphics[width=\wideR\columnwidth]{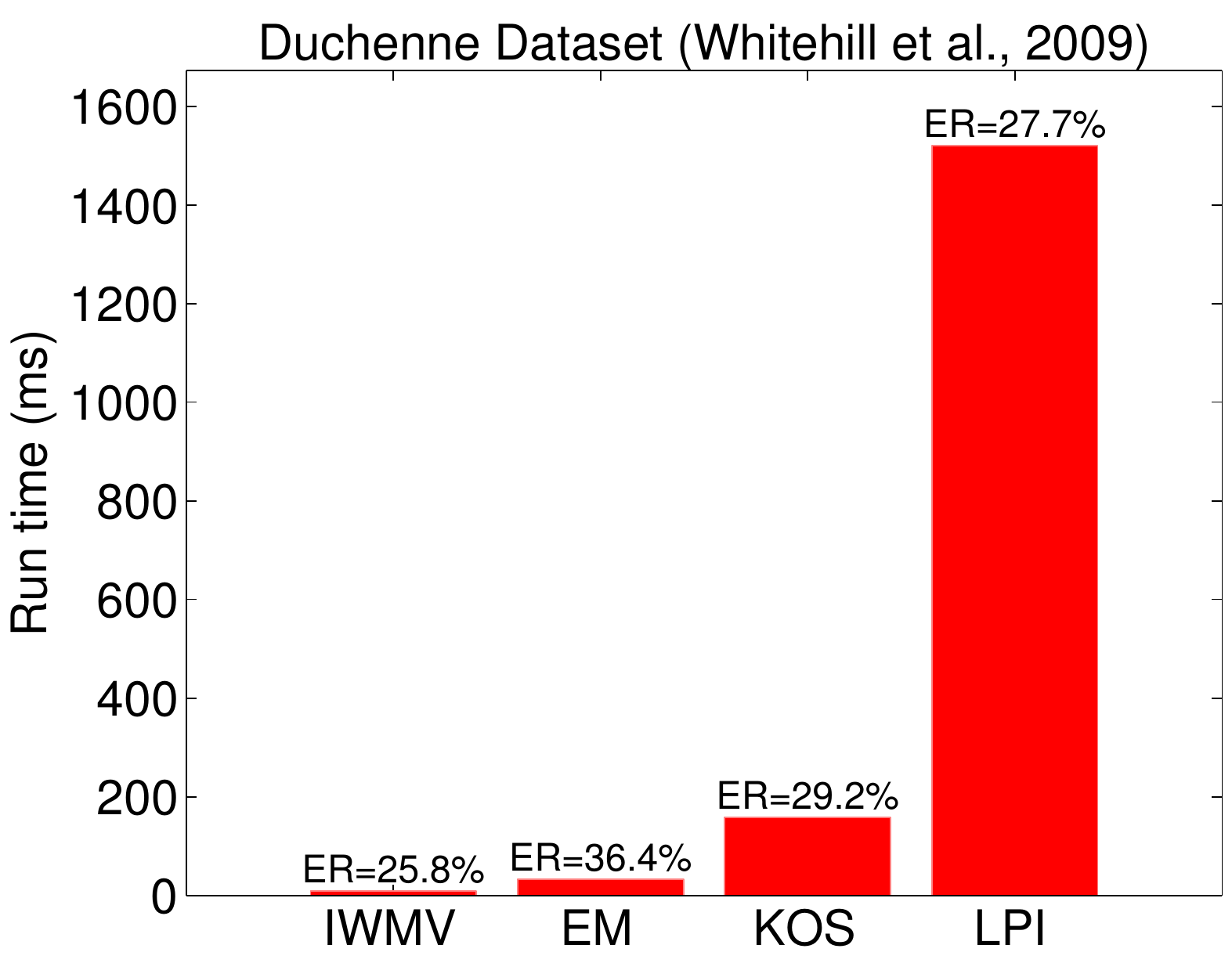}
\\
(a)& (b) & (c)
\end{tabular}
\caption{Duchenne smile dataset \citep{Whitehill_nips09}. (a) Error rate of different algorithms when the number of labels available increases. (b) Run time comparison. (c) A visualization with both run time and error rate when 40.7\% of the task assignments are done. 
}
\label{fig:results_Duchenne}
\vspace{-3mm}
\end{center}
\end{figure}

We conducted the experiments by sampling the labels independently with probability $\samq$ varying from 0.1 to 1. Thus the proportion of non-zero labels in the data matrix will vary from 0.6\% to 45.2\%. Note that based on our setting, the \taskAssign probability corresponding to $\samq$ is $\qs= \samq\times 45.2\%$. After sampling the data matrix from the original Duchenne dataset with a given $\samq$, we then run IWMV, majority voting, the EM algorithm \citep{Raykar_JMLR10}
, KOS (the iterative algorithm in \citep{Karger_NIPS2011}), and LPI (the variational inference algorithm from \citep{Liu2012}). The entire process will be repeated 100 times, and the results will be averaged. 
The comparison is shown in Figure \ref{fig:results_Duchenne}(a). 

From Figure \ref{fig:results_Duchenne}(a), We can see that when the available labels are very few, the performance of IWMV is as good as LPI, and these two generally dominate the other algorithms.  With more labels available, the error rate of IWMV is  around 2\% lower than LPI (Figure \ref{fig:results_Duchenne}(a)). At the same time, we compared the run time of each algorithm (Figure \ref{fig:results_Duchenne}(b)). With more labels, the run time of LPI increases fast (non-linearly), while the IWMV maintains a lower run time than EM, KOS and LPI. For a better visualization of comparing run time and error rate, we compared the run time of IWMV, EM, KOS and LPI by a bar plot with their error rates imposed on top. Figure \ref{fig:results_Duchenne}(c) shows the comparison when 40.7\% of the labels are available ($\samq=0.9$). IWMV is more than 100 times faster than LPI, and achieves the lowest error rate among these algorithms. 

An interesting phonomenon in Figure \ref{fig:results_Duchenne}(a) is that EM performs poorly --- it is even worse than MV. The major reason for this is that the workers reliabilities form a pattern similar to our model misspecification example in Figure \ref{fig:model_misspecify}(a): some workers are good at a set of images but bad at the complementary set of images, while the other workers are reversed. This is a real-data example of model misspecificatioin, and IWMV is more  robust to the model misspecification on this data than EM.

\begin{figure}[htb]
\begin{center}
\begin{tabular}{ccc}
\hspace{\suo}
\includegraphics[width=\wideR\columnwidth]{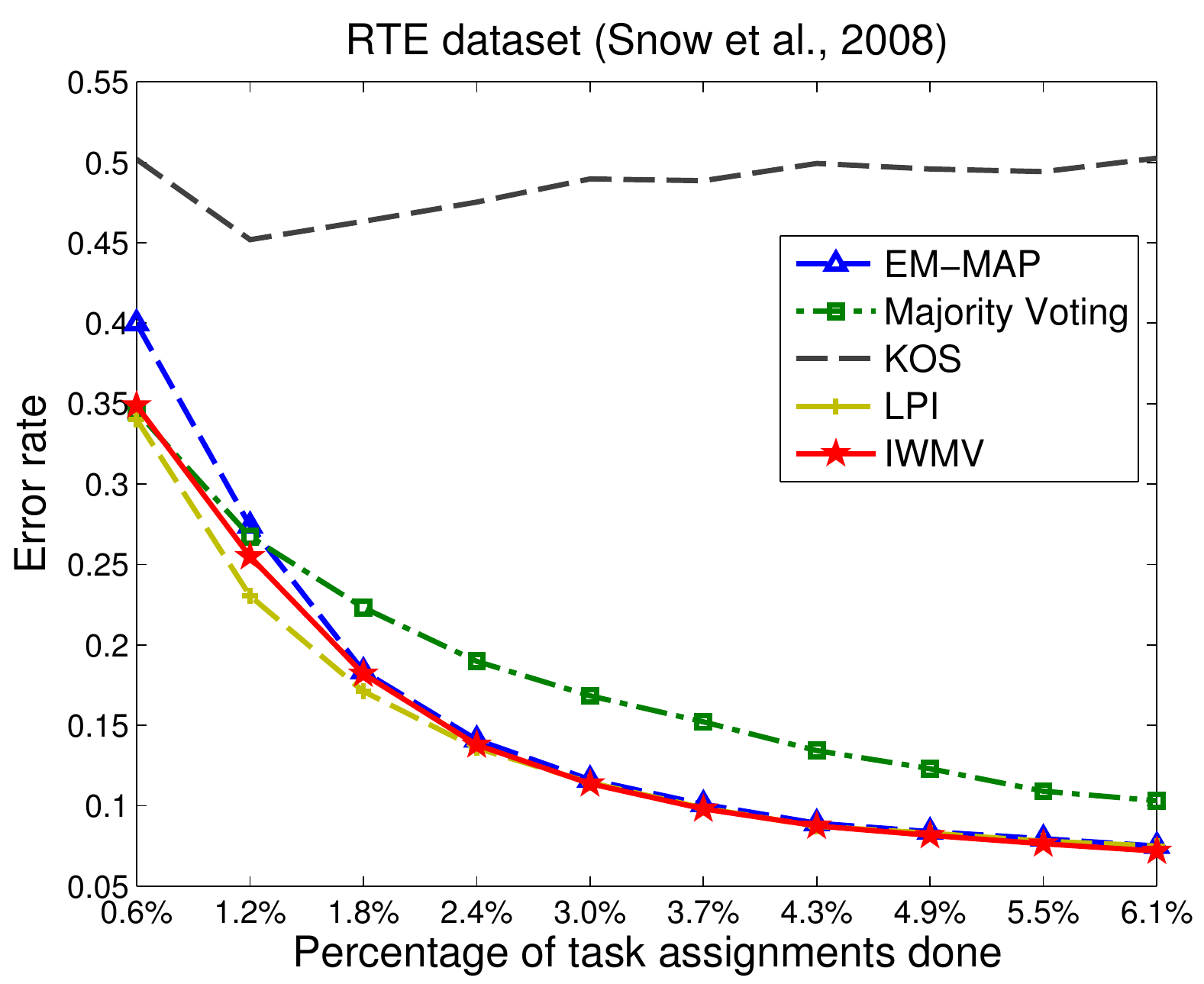}
&
\hspace{\suo}
\includegraphics[width=\wideR\columnwidth]{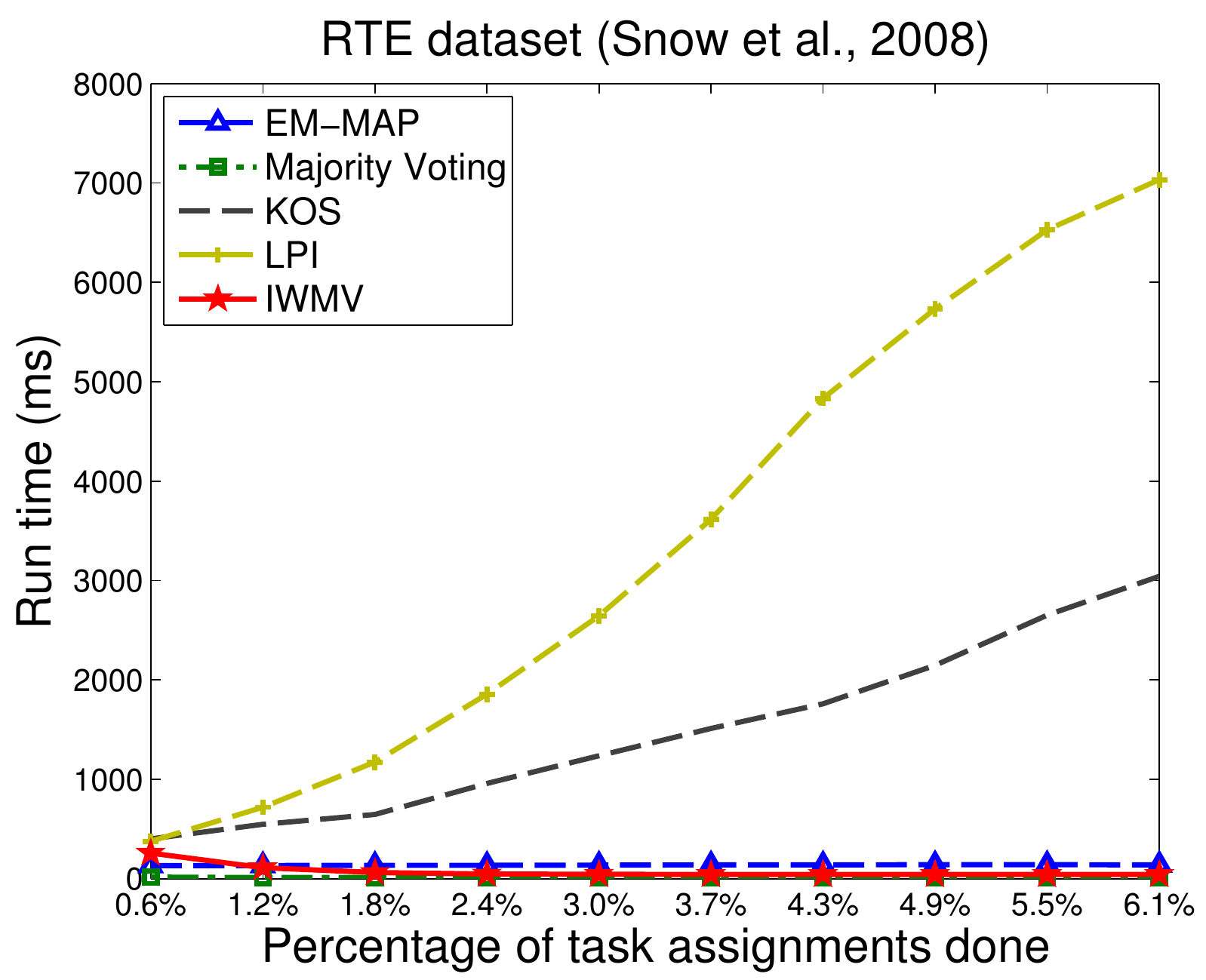}
&
\hspace{\suo}
\includegraphics[width=\wideR\columnwidth]{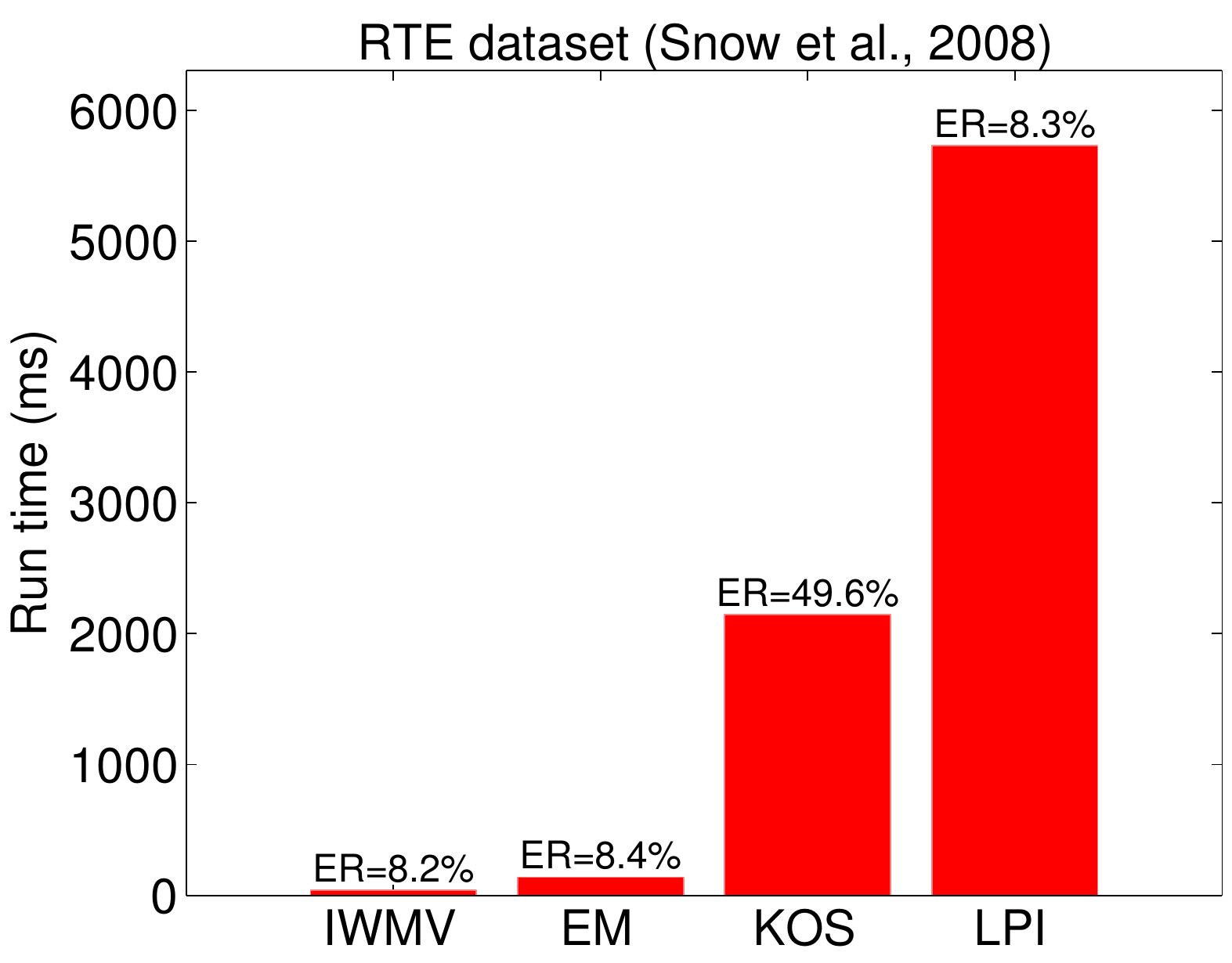}
\\
(a) & (b) & (c)
\end{tabular}
\caption{ RTE dataset \citep{Snow_emnlp08}. (a) Error rate of different algorithms.
(b) Run time when the percentage of the task assignments done increases.
(c) Run time comparison when 4.9\% of the task assignments is done. The error rates of each method are imposed on the top of the bar. 
}
\label{fig:results_RTE_realData}
\end{center}
\end{figure}


\emph{RTE dataset.} The RTE data is a language processing dataset from \citep{Snow_emnlp08}. The dataset is collected by asking workers to perform recognizing textual entailment (RTE) tasks, i.e.,  for each question the worker is presented with two sentences and given a binary choice of whether the second sentence can be inferred from the first. 


\emph{Temporal event dataset.}
This dataset is also a natural language processing dataset from \citep{Snow_emnlp08}. The task is to provide a label from \emph{\{strictly before, strictly after\}} for event-pairs that represents the temporal relation between them. 

We conducted similar experiments on the RTE dataset and the temporal event dataset as the one on the Duchenne dataset. The results on the RTE dataset are shown in Figure \ref{fig:results_RTE_realData}. 
Figure \ref{fig:results_RTE_realData}(a) is the performance curves of different algorithms when the percentage of task assignments done increases. 
Figure \ref{fig:results_RTE_realData}(b) is the run time of these algorithms, and it confirms the same observations as the results on the Duchenne dataset: the IWMV runs much faster than the other algorithms except majority voting, and it has similar performance to LPI which is the state-of-art method (Figure \ref{fig:results_RTE_realData}(c)). For clarity, we show the performance comparison on temporal event dataset in Figure \ref{fig:results_temp_web}(a) and omit the run time comparison. 


\begin{figure}[htb]
\begin{center}
\begin{tabular}{ccc}
\includegraphics[width=0.4\columnwidth]{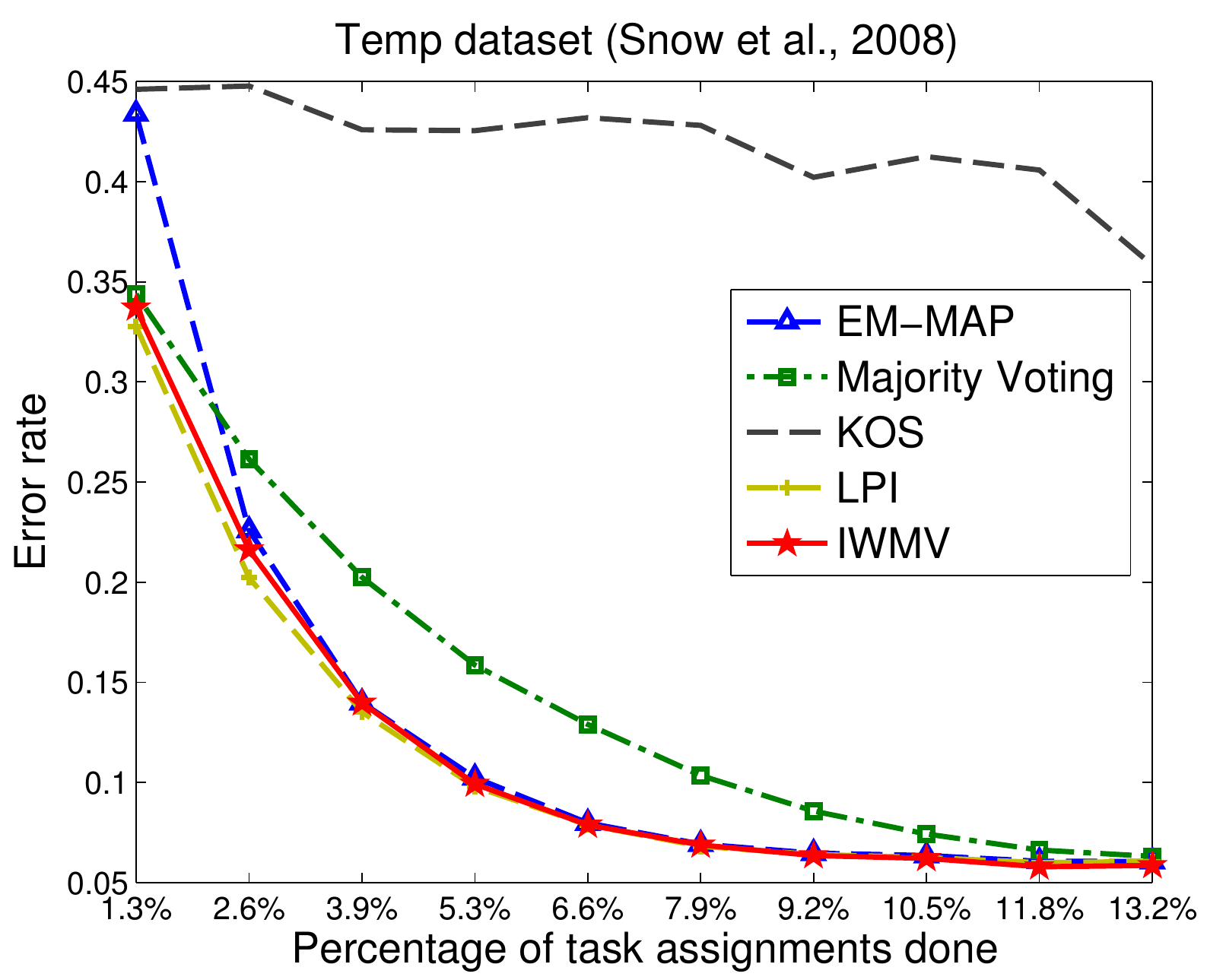}
&
\includegraphics[width=0.4\columnwidth]{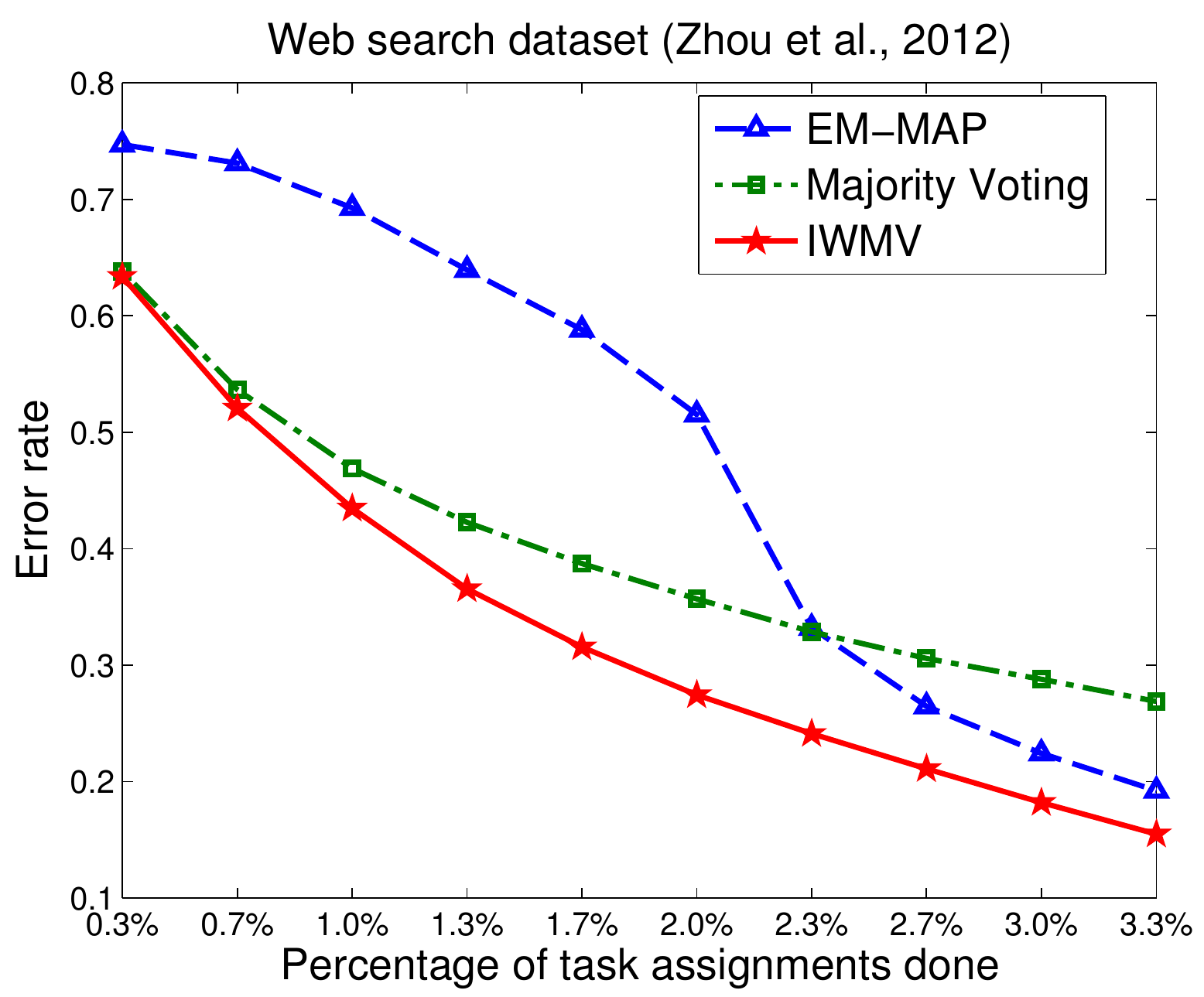}
\\
(a) & (b)
\end{tabular}
\caption{(a) Results on Temp dataset from \citep{Snow_emnlp08}. 
(b) Results on Web search dataset \citep{Zhou2012}. 
}
\label{fig:results_temp_web}
\vspace{-5mm}
\end{center}
\end{figure}

\emph{Web search dataset.} In this dataset \citep{Zhou2012}, workers were asked to rate query-URL pairs on a relevance rating scale from 1 to 5. Each pair was labeled by around 6 workers, and around 3.3\% of the entries in the data matrix are observed. 
The ground truth labels were collected via consensus from 9 experts. We treat the task as a multi-class labeling problem, thus $\nL=5$. 
We conduct the experiment in a similar setting to the experiment on the Duchenne dataset --- sampling the labels with probability $\samq$ and varied it to plot the performance curve (Figure \ref{fig:results_temp_web}(b)). Since LPI and KOS is constrained to binary labeling so far, we only compared IWMV with the \emmaprule and majority voting. 
The performance of IWMV generally outperforms the \emmaprule and majority voting by at least 4\%.

\section{Conclusions}

In this paper, we provided finite sample bounds on the error rate (in probability and in expectation) 
of decomposable \predrules 
 under the general \ds crowdsourcing model.
Optimizing the mean error rate bound under the \hds model leads to an \predrule that is a good approximation to the oracle MAP rule.  A data-driven iterative weighted majority voting is proposed to approximate the oracle MAP with a theoretical guarantee on the error rate of its one-step version.

Through simulations under the \hds model (for simplicity) and tests on real data, we have the following findings. 
\begin{enumerate}

\item The error rate bound reflects the trends of the real error rate of the \omaprule when some important factors in the crowdsourcing systems such as ($M,N,\wiAll$ etc.) change. 

\item The IWMV algorithm is close to the oracle MAP rule with superior performance in terms of error rate.

\item The iterative weighted majority voting method (IWMV) performs as well as the \emmaprule with much lower computational cost in simulation, and IWMV is more robust to model-misspecification than EM.  

\item On real data, IWMV achieved performance as good as or even better than that of the state-of-the-art methods with much less computational time. 


\end{enumerate}



In practice, if we want to obtain the error rate bounds for certain \predrule{}s that falls in the form of decomposable \predrule (\ref{def:predRuleFi}), what we can do should be similar to what we did in Section \ref{sec:specialCases}: (1) for the specific model and \taskAssign, compute the measure of $\tone$ and $\ttwo$ (and also $\cH$ and $\sigtwo$ if interested in bounds on the mean error rate) according to the descriptions in Section \ref{sec:quantities}. (2) Compute the corresponding error rate bounds according to the theorems in Section \ref{sec:Mainresults}. The quantities $\tone$ and $\ttwo$ can tell us if we can obtain upper bound or lower bound on the error rate in probability and in expectation. Note that though the mean error rate bounds (Theorem \ref{thm:MostGeneral_meanErrorRate}) are in a composite form of two exponential bounds, we can choose one of the two to use if the convenience of theoretical analysis is concerned. 

To the best of our knowledge, this is the first extensive work on error rate bounds for general \predrules under the practical \ds model for multi-class crowdsourced labeling. Our bounds are useful for explaining the effectiveness of different \predrules.

As a further direction for research, it would be interesting to obtain finite sample error bounds for the \predrules with random score functions, which depend on the data in a complicated manner. For example, the \emmaprule can be formulated as a weighted majority voting under the \hds model. However, the weights are estimated by EM algorithm \citep{Raykar_JMLR10} and depend on the data complicatedly. Hence the analysis of the EM algorithm is rather difficult. The IWMV and the EM algorithm share the similar iterative nature, and IWMV is simpler than EM. An error rate analysis of the final prediction of the IWMV will be helpful to understand the behavior of EM algorithm in the crowdsourcing context.

\section{Acknowledgement}
We thank Dengyong Zhou for his valuable suggestions and enlightening comments which lead to many improvements of this paper. 
We would like to thank Riddhipratim Basu and Qiang Liu for the valuable discussions. We would also like to thank Terry Speed for his helpful comments and suggestions.

\nocite{Dalvi2013, Dekel_icml09,Ertekin_nips09,Sheng2008,  Jin_nips02, Smyth1995,
Yan_icml11,
Bachrach_ICML12, Angluin_ML88,
Natarajan_NIPS13, Ho2013}


\bibliography{hongwei_NIPS14_reference}
\bibliographystyle{natbib}

\appendix
\section{Proofs of the main theorems}
\label{app:proveTheorems}

\def \thmGeneralERBoundProb{\ref{thm:MostGeneral_One}\xspace}
\def \thmPACDelta{\ref{thm:MostGeneral_DeltaExplicit}\xspace}
\def \thmGeneralMERBound{\ref{thm:MostGeneral_meanErrorRate}\xpace}
\def \thmGeneralMERIndiv{\ref{thm:MostGeneral_boundMER_indiv}\xspace}

\def \corOMAPgds{Corollary \ref{cor:oracleMAP_gds}\xspace}
\def \corMAPhds{Corollary \ref{cor:msr_MAP_hds}\xspace}

\def \thmosWMVBound{Theorem \ref{thm:mer_IWMV_Bound} }

Since the proof of Theorem \ref{thm:MostGeneral_One} requires other theorems in this paper, we will not present the proofs in the same order as in the paper. The order of our proofs will be: Proposition \ref{thm:MostGeneral_boundMER_indiv}, Theorem \ref{thm:MostGeneral_meanErrorRate},
Theorem \ref{thm:MostGeneral_One} and then
Theorem \ref{thm:MostGeneral_DeltaExplicit}. After proving the first four main results, we also prove \corOMAPgds and \corMAPhds. Since the proof of \thmosWMVBound requires different techniques and more efforts than the proofs in this section, we put its proof in a seperate section --- Appendix B. 

 
Before presenting the proofs, we would like to propose some notations for simplicity and clarity. 

\def \muikl{\weighti_{kl}}
\def \muikk{\weighti_{kk}}
\def \muikh{\weighti_{kh}}
\def \muilh{\weighti_{lh}}

We  simplify  $\f_i(k,h)$ to 
\begin{eqnarray}
\muikh \defas \f_i(k,h), \quad \forall k\in \Labset, h\in\Lextend,
\end{eqnarray}
then each worker is associated with a vote matrix $\weighti=(\muikh), k\in \Labset, h\in \Lextend$, where $\muikh$ is the voting score when worker $i$ labels item $j$ whose true label is $k$, as class $h$ if $h\neq 0$. Then the \predrule (\ref{def:predRuleFi}) is equivalent to
\begin{eqnarray}\label{def:generalPredBrief}
\hyj= \argmax_{k\in\Labset} \kua{\sumi \sumhL \muikh \I{\zij=h} ~+ ~\ak},
\end{eqnarray}

Note that 
\begin{eqnarray}\label{def:sjl}
\sjk = \sumi\sumhL \muikh\I{\zij=h} + \ak , \quad \forall k\in \Labset, j\in\N
\end{eqnarray}
is the aggregated score of label class $k$ on the $j$th item, and  the general \predrule is
\begin{eqnarray}
\hyj= \argmax_{k\in\Labset} \sjk . 
\end{eqnarray}

We will frequently discuss conditional probability, expectation and variance conditioned on the event $\hua{\yj=k}$. For simplicity of notations, we define:
\begin{eqnarray}
&& \Pk{~\cdot~} \defas \P(~\cdot~ | \yj=k) \\
&& \Ek{~\cdot~} \defas \E[~\cdot~ | \yj=k] \\
&& \Vark{~\cdot~} \defas \mathrm{Var}(~
\cdot~ | \yj= k) . 
\end{eqnarray}

Note that 
\begin{eqnarray}
\Ek{\sjl}= \sumi\sumhL \qij\muilh\cmikh + \al, ~~ \forall l,k \in \Labset. 
\end{eqnarray}

\subsection{Proof of Proposition \ref{thm:MostGeneral_boundMER_indiv}: bounding the mean error rate of labeling each  item}

\begin{prop}\label{thm:MostGeneral_boundMER_indiv}
\hwem{(Bounding the mean error rate of labeling each  item)}
Following the setting of Theorem \ref{thm:MostGeneral_One}, and with $\taujmin$ and $\taujmax$ defined as in (\ref{def:taujminmax})
, we have $\forall j\in \N
$, 
\\
(1)
if ~ $\taujmin \geq 0$,~  then ~
$
\P(\hyj \neq \yj) \leq (\nL-1)\cdot\min\hua{ \exp\kua{ - {\taujmin^2 \over 2}}, \exp\kua{-\frac{\taujmin^2}{2\kua{\sigtwo + \cH\taujmin/3}}} } $;
\\
(2)
if ~ $\taujmax \leq 0$,~  then ~ 
$
\P(\hyj \neq \yj) \geq 1 - \min\hua{\exp\kua{-\frac{\taujmax^2}{2}}, \exp\kua{-\frac{\taujmax^2}{2(\sigtwo - c\taujmax/3)} }}.
$

\end{prop}

\remark Proposition \ref{thm:MostGeneral_boundMER_indiv} provides the mean error rate bounds of labeling any specific item, and the bounds depend on the minimum and maximum values of $\hua{\gapjkl}_{k,l\in\Labset}$. Note that the subscript $j$ only comes from the \assign distribution $\qij$. If a specific worker has the same \assign probability to label all items, say $\qi$, then we can drop the subscript $j$ from $\taujmin$ and $\taujmax$, which means the error rate bounds of each item are eventually the same under that task assignment.

\begin{proof}
First of all, we expand the error probability of labeling the $j$-th item wrong in terms of the conditional probabilities:
\begin{eqnarray}\label{eqn:phyjDecomp}
\P(\hyj\neq \yj) \= \sumkinL \P(\yj=k)\P(\hyj\neq k | \yj=k) 
~=~ \sumkinL \pri_k \Pk{\hyj\neq k}.
\end{eqnarray}

Our major focus in this proof is to bound the term $\Pk{\hyj\neq k}$. 
Our approach will be based on the following events relations:
\begin{equation}
\bigcup_\lneqk \hua{\sjl>\sjk} 
\quad\subseteq\quad 
\hua{\hyj\neq k} 
\quad\subseteq\quad 
\bigcup_\lneqk\hua{\sjl\geq \sjk} . 
\end{equation}

(1). Assuming $\taujmin \geq 0$, we want to show the lower bound for $\P(\hyj\neq \yj)$. Note that
\begin{eqnarray}
\Pk{\hyj\neq k} &\leq& \Pk{\bigcup_\lneqk \hua{\sjl\geq \sjk}} 
\leq  \sum_\lneqk \Pk{\sjl\geq \sjk}.  \label{ineq:sum_prob_sl_geq_sk}
\end{eqnarray}

With $\sjl$ defined as in (\ref{def:sjl}), $\gapjkl$ defined as in (\ref{def:gapjkl}) and 
\begin{eqnarray}
&& \xikl \defas \sumhL (\muilh-\muikh)\I{\zij=h}, \\
&& \Ek{\xikl} =  \sumhL \qij\kua{\muilh-\muikh}\cmikh, 
\end{eqnarray}

we have
\begin{eqnarray}
\Pk{\sjl\geq \sjk} \= \Pk{\sumi\sumhL \kua{ \muilh-\muikh}\I{\zij=h} \geq \ak-\al} \nonumber \\
\= \Pk{\sumi \xikl \geq \ak-\al}
\nonumber \\
\= \Pk{\sumi \xikl - \sumi \Ek{\xikl} \geq (\ak-\al) - \sumi \Ek{\xikl}}, \quad 
\nonumber \\
\= \Pk{\sumi \xikl - \sumi \Ek{\xikl} \geq \gapjkl} \label{ineq:lowerBoundHub}
\end{eqnarray}

Note that $\hua{\xikl}_{i\in\M}$ are conditionally independent when given $\hua{\yj=k}$, and they are bounded given the voting weights $\hua{\muikh}$ are bounded. Therefore, we can apply the Hoeffding concentration inequality \citep{Hoeffding1956} to further bound $\Pk{\sjl\geq\sjk}$. 

We have that 
$
\min_{l,k,h\in\Labset,k\neq l}\hua{\muilh-\muikh} \leq \xikl \leq 
\max_{l,k,h\in\Labset,k\neq l}\hua{\muilh-\muikh}, 
$
and 
\begin{eqnarray*}
\sumi \braket{\max_{l,k,h\in\Labset,k\neq l}\hua{\muilh-\muikh} - \min_{l,k,h\in\Labset,k\neq l}\hua{\muilh-\muikh}}^2 
&\leq& \sumi\kua{2\max_{l,k,h\in\Labset,k\neq l}|\muilh-\muikh|}^2 
\\
\= 4\normalize^2 \label{ineq:range_xikl}
\end{eqnarray*}

When $\gapjkl\geq \taujmin\cdot\normalize\geq 0$, by applying the Hoeffding inequality to (\ref{ineq:lowerBoundHub}), we have
\begin{eqnarray*}
\Pk{\sjl\geq \sjk} &\leq& \Pk{\sumi \xikl - \sumi \Ek{\xikl} \geq \gapjkl} 
\\
&\leq& \exp\kua{-\frac{2\gapjkl^2}{\sumi \braket{\max_{l,k,h\in\Labset,k\neq l}\hua{\muilh-\muikh} - \min_{l,k,h\in\Labset,k\neq l}\hua{\muilh-\muikh}}^2 }} 
\\
&\leq& \exp\kua{- \frac{\gapjkl^2}{2\normalize^2}} \qquad\qquad (\text{because of (\ref{ineq:range_xikl})})
\\
&\leq & \exp\kua{-\frac{\taujmin^2}{2}}. \qquad\qquad (\text{based on the definition of $\taujmin$)}) 
\end{eqnarray*}

The right hand side of the last inequality does not depend on $k, l$ or $i$, then
\begin{eqnarray}
\Pk{\hyj\neq k} &\leq& \sum_\lneqk \Pk{\sjl\geq \sjk} 
\quad\leq\quad  (\nL-1)\exp\kua{-\frac{\taujmin^2}{2}}.
\end{eqnarray}
Because the RHS does not depend on $k$, we have
\begin{eqnarray}
\P(\hyj\neq \yj) \= \sum_{k\in\Labset}\pri_k\Pk{\hyj\neq k} \nn\\
&\leq& (\nL-1)\exp\kua{-\frac{\taujmin^2}{2}} \kua{\sum_{k\in\Labset} \pri_k} \nn\\
\= (\nL-1)\exp\kua{-\frac{\taujmin^2}{2}} 
\label{ineq:Hoeffding1}
\end{eqnarray}
The Hoeffding inequality does not take the variance information of the independent random variables into account, thus a ``stronger" concentration inequality can be applied when the fluctuation of $\xikl$ is available. 

Note the definition of $\cH$ and $\sigtwo$ are defined as
\begin{eqnarray*}
\cH \= \inv{\normalize}\cdot\max_{i\in\M, k,l,h\in\Labset, k\neq l} \abs{\muikh-\muilh},
\\
\sigtwo \= \inv{\normalize^2}\cdot\maxj \maxkl \sumi\sumhL \qij\kua{\muikh-\muilh}^2 \cmikh.
\end{eqnarray*}
The sum of the second moment of $\xikl$ can be bounded as
\begin{eqnarray*}
\sumi \Ek{\kua{\xikl}^2} \= \sumi \Ek{\kua{\sumhL (\muilh-\muikh)\I{\zij=h} }^2 }
= \sumi\sumhL \qij\kua{\muikh-\muilh}^2 \cmikh
\leq \sigtwo \normalize^2. 
\end{eqnarray*}

$\xikl$ can be bounded as 
$
|\xikl| 
\quad \leq \quad 
\max_{i\in\M, h\in\Labset} \abs{\muilh-\muikh}
\quad = \quad 
\cH\normalize. 
$ 

By applying the Bernstein-type concentration inequality (\citep{Chung_oldNewConIneq}, Theorem 2.8) with that $\gapjkl\geq \taujmin\normalize\geq 0$, 
\begin{eqnarray}
\Pk{\sjl\geq \sjk} &\leq& \Pk{\sumi \xikl - \sumi \Ek{\xikl} \geq \gapjkl} 
\nn\\
&\leq& \exp\kua{- \frac{\gapjkl^2}{2\kua{\sigtwo + \cH\normalize\gapjkl/3}}},  \nn \\
&\leq& \exp\kua{-\frac{\taujmin^2}{2\kua{\sigtwo+ \cH\taujmin/3}}}  \quad \quad (\text{because $\frac{\gapjkl}{\normalize}\geq \taujmin\geq 0$}), \qquad 
\end{eqnarray}
where the RHS does not depend on $k, l$. 
Then, we have
$$
\Pk{\hyj\neq k} \leq (\nL-1)\exp\kua{-\frac{\taujmin^2}{2\kua{\sigtwo + \cH\taujmin/3}}}.
$$
Furthermore, 
\begin{eqnarray}\label{ineq:Bernstein1}
\P(\hyj\neq \yj)= \sum_{k\in\Labset}\pri_k\Pk{\hyj\neq k} \leq (\nL-1)\exp\kua{-\frac{\taujmin^2}{2\kua{\sigtwo + \cH\taujmin/3}}}. 
\end{eqnarray}

Combining inequalities (\ref{ineq:Hoeffding1}) and (\ref{ineq:Bernstein1}) together, we can get the desired result in Theorem \ref{thm:MostGeneral_boundMER_indiv}.(1).

(2). Assuming that $\taujmax\leq 0$, we want to show the upper bound for $\P(\hyj\neq \yj)$. 

Using the same argument as in (1), we  provide a lower bound for $\Pk{\hyj\neq k}$. 

\begin{eqnarray*}
\Pk{\hyj\neq k} \geq \Pk{\bigcup_\lneqk \hua{\sjl > \sjk}}
~&\geq&~\max_\lneqk \Pk{\sjl > \sjk} \\
\= 1 - \min_\lneqk \Pk{\sjl\leq \sjk} 
\end{eqnarray*}

Given $\gapjkl \leq \taujmax\cdot\normalize \leq 0$, by applying the Hoeffding and the Bernstein inequality as in (1), we can obtain  
$
\Pk{\sjl\leq \sjk} \leq \exp\kua{-\frac{\gapjkl^2}{2\normalize^2}}
\leq \exp\kua{-\frac{\taujmax^2}{2}}, 
$
and 
$
\Pk{\sjl\leq \sjk} \leq \exp\kua{-\frac{\gapjkl^2}{2(\sigtwo\normalize^2 - c\normalize\gapjkl)}}
\leq \exp\kua{-\frac{\taujmax^2}{2(\sigtwo - c\taujmax)}}.
$
Since the RHS of the two inequalities do not depend on $k$ or $l$, 
$
\Pk{\hyj\neq k} \geq 1 - \min\hua{\exp\kua{-\frac{\taujmax^2}{2}}, \exp\kua{-\frac{\taujmax^2}{2(\sigtwo - c\taujmax)} }}. 
$

With (\ref{eqn:phyjDecomp}), we have 
$
\P\kua{\hyj\neq \yj} \geq 1 - \min\hua{\exp\kua{-\frac{\taujmax^2}{2}}, \exp\kua{-\frac{\taujmax^2}{2(\sigtwo - c\taujmax)} }}. 
$

\end{proof}

\subsection{Proof of Theorem \ref{thm:MostGeneral_meanErrorRate}}

\begin{proof}
Given $\sigtwo\geq 0 $ and $\cH> 0$, both functions $\exp\kua{-\frac{t^2}{2}}$ and $\exp\kua{-\frac{t^2}{2(\sigtwo + \cH t/3)}}$ are monotonely increasing on $t\in [0, \infty)$. On the other hand, both functions $\exp\kua{-\frac{t^2}{2}}$ and $\exp\kua{-\frac{t^2}{2(\sigtwo - \cH t/3)}}$ are monotonely decreasing on $t\in (-\infty, 0]$. 

Given $\tone \geq 0$, then $\taujmin\geq \tone \geq 0$. 
By Proposition \ref{thm:MostGeneral_boundMER_indiv}, 
\begin{eqnarray*}
\MER &\leq& \inv{N}\sumj (\nL-1)\min\hua{\exp\kua{-\frac{\taujmin^2}{2}}, \exp\kua{-\frac{\taujmin^2}{2\kua{\sigtwo + \cH\taujmin/3}}}} \\
&\leq& \frac{\nL-1}{N}\sumj \min\hua{\exp\kua{-\frac{\tone^2}{2}}, \exp\kua{-\frac{\tone^2}{2\kua{\sigtwo + \tone/3}}}}\\
\= (\nL-1)\min\hua{\exp\kua{-\frac{\tone^2}{2}}, \exp\kua{-\frac{\tone^2}{2\kua{\sigtwo + \tone/3}}}}.
\end{eqnarray*}

Thus, we have proved Theorem \ref{thm:MostGeneral_meanErrorRate}.(1).

With the same argument, we can straightforwardly prove Theorem \ref{thm:MostGeneral_meanErrorRate}.(2). 

\end{proof}

\subsection{Proof of Theorem \ref{thm:MostGeneral_One}}

So far, we have bounded the mean error rate, but we still need more tools for bounding the error rate in the practical case with high probability. The following lemma is another form of the Bernstein-Chernoff-Hoeffding theorem \citep{Ngo_LN2011}.

\def \accuracy{\inv{N}\sumj\I{\hyj\neq \yj}}

\begin{lem}\label{lem:BCH}
\hwem{(Bernstein-Chernoff-Hoeffding)} Let $\xi_i\in[0,1]$ be independent random variables where $\E\xi_i= p_i, i\in [n]$. Let $\xibar= \frac{1}{n}\sum_{i=1}^n \xi_i$ and $\pbar=\frac{1}{n}\sum_{i=1}^n p_i$. Then, 

(1) for any $m$ such that $\pbar \leq \frac{m}{n} < 1$, 
$
\P\kua{\xibar > m/n } \leq e^{-n \D(m/n || \pbar)}, 
$

(2) for any $m$ such that $0< \frac{m}{n} \leq \pbar$, 
$
\P\kua{\xibar < m/n} \leq e^{-n \D(m/n || \pbar)}.
$
\end{lem}

The proof of Theorem \thmGeneralERBoundProb is as follows:
\begin{proof} \emph{(Theorem \thmGeneralERBoundProb)}

\hwem{Proof of Theorem \thmGeneralERBoundProb (1) }
Let $\mu = \MER  $. By Theorem \ref{thm:MostGeneral_meanErrorRate}.(1), we have that 
$\mu \leq (\nL-1)e^{- \tone^2/2} = (\nL-1)\phi(\tone)$. Assume $\tone \geq \sqrt{2\ln\frac{\nL-1}{\epsilon}}$, then we can get $(\nL-1)\phi(\tone) \leq \epsilon$, which gives us $0 \leq \mu \leq (\nL-1)\phi(\tone) \leq \epsilon$. Then by the Bernstein-Chernoff-Hoeffding Theorem, i.e. Lemma \ref{lem:BCH}, we get 
$\P\kua{\accuracy > \epsilon} \leq e^{-N \D(\epsilon ||\mu)} \leq e^{-N \D(\epsilon ||(\nL-1)\phi(\tone))}$. Therefore, we have 
$$
\P\kua{\errorrate \leq \epsilon} \geq 1- e^{-N\D(\epsilon || (\nL-1)\phi(\tone))}. 
$$

\hwem{Proof of Theorem \thmGeneralERBoundProb (2) }
With the same argument as above, assuming $\ttwo \leq - \sqrt{2\ln\inv{1-\epsilon}}$, then $1\geq \mu \geq 1-\phi(\ttwo)\geq \epsilon$, which gives us
$$
\P\kua{\errorrate \geq \epsilon}  \geq 1- e^{-N\D(\epsilon || 1- \phi(\ttwo))}.
$$

Thus, we have proved Theorem \thmGeneralERBoundProb.

\end{proof}

\subsection{Proof of Theorem \ref{thm:MostGeneral_DeltaExplicit}}

Before proving Theorem \ref{thm:MostGeneral_DeltaExplicit}, we are going to  prove an important lemma for bounding the average of a group of independent Bernoulli random variables. 
The proof of this lemma relies on Hoeffding bounds and the Bernstein-Chernoff-Hoeffding theorem \citep{Ngo_LN2011}.  
\begin{lem}\label{lemma:probBound}
Suppose $\forall j\in \N$, $\xi_j \sim $ Bernoulli($\pj$) with $\pj\in (0,1)$, and $\xi_j$'s are independent of each other. Let $\xibar= \inv{N}\sumj \xi_j$  and $\pbar=\E \xibar= \inv{N}\sumj \pj$
Given any $\epsilon, \delta \in (0,1)$: 

\vspace{3mm}
(1) If~~
$
0<\pbar\leq \inv{\constC},
$
~~then~~
$\P(\xibar \leq \epsilon) \geq 1-\delta.$

(2) If~~
$
\inv{1+ \exp\kua{ - \inv{1-\epsilon}\constA}} \leq \pbar < 1,
$
~~then~~
$\P(\xibar \leq \epsilon) < \delta.$

\end{lem}
\begin{proof}
For simplicity let's define 
$
\A = \constA 
$

\hwem{The proof of Lemma \ref{lemma:probBound}.(1):} We will finish the proof in several steps:\vgap

Assume  $0<\pbar\leq \inv{\constC}$.

\vgap
\emph{Step 1.  we want to show $ \pbar < \epsilon $ :}

\begin{eqnarray*}
\exp\kua{ \frac{\A}{ \epsilon}} \= \exp\kua{ \epsilon\ln \inv{\epsilon} + (1-\epsilon)\ln\inv{1-\epsilon} + \inv{N}\ln \inv{\delta} \over \epsilon   } \\
\= \exp\kua{\ln \inv{\epsilon} + {1-\epsilon \over \epsilon} \ln\inv{1-\epsilon} + \inv{N\epsilon}\ln\inv{\delta}}\\
&>& \exp\kua{\ln\inv{\epsilon}} \qquad \kua{\because \epsilon, \delta\in (0,1), N>0}\\
\= \inv{\epsilon}
\end{eqnarray*}

\hgap\hgap $\then 	1+ \exp\kua{\A \over \epsilon} > \inv{\epsilon}$
$\then  \inv{1+\exp\kua{\A \over \epsilon}} < \epsilon$ 

Since $0<\pbar\leq \inv{1+\exp\kua{A/\epsilon}}$, then we have $\pbar < \epsilon$ 

\vgap
\emph{Step 2. We want to show $\P\kua{\xibar \leq \epsilon}\geq 1- e^{-N\cdot \D(\epsilon||\pbar)}$ :}

This is obtained by the Bernstein-Chernoff-Hoeffding Theorem (\citep{Ngo_LN2011, McDiarmid_1998}), which leads to:
\begin{eqnarray}
&& \text{If $0<\pbar\leq\epsilon$, then $\P(\xibar\leq \epsilon) \leq 1-e^{-ND(\epsilon||\pbar)}$} \label{eqn:BCH1}\\
&& \text{If $\epsilon\leq \pbar<1$, then $\P(\xibar\leq \epsilon) \leq e^{-ND(\epsilon||\pbar)}$ }  \label{eqn:BCH2}
\end{eqnarray}

Since we have shown in step 1 that $\pbar < \epsilon $, then we can get the desired result in this step easily.

\vgap
\emph{Step 3. We want to show  $e^{-ND(\epsilon||\pbar)} \leq \delta$ :}\\
Note:
\begin{eqnarray}
&&  e^{-ND(\epsilon||\pbar)} \leq \delta \nonumber\\
&\iff& D(\epsilon||\pbar) \geq \inv{N}\ln\inv{\delta} \nonumber\\
&\iff& \ln {\epsilon^\epsilon (1-\epsilon)^{1-\epsilon} \over \pbar^\epsilon (1-\pbar)^{1-\epsilon}} \geq \ln \kua{\inv{\delta}}^{\inv{N}} \nonumber \\
&\iff&  \pbar^\epsilon (1-\pbar)^{1-\epsilon} \leq \exp\kua{-\constA}= e^{- \A} \label{eqn:lem1_boundeND}
\end{eqnarray}

From the condition we have, 
\begin{eqnarray*}
&& \pbar\leq \inv{1+\exp( \A /\epsilon )} \hgap \then \hgap  \kua{\pbar \over 1-\pbar}^{\epsilon}  \leq e^{- \A }
\end{eqnarray*}

Note that 
$ \pbar^{\epsilon}(1-\pbar)^{1-\epsilon} = \kua{\pbar\over 1-\pbar}^{\epsilon} (1-\pbar) < \kua{\pbar \over 1-\pbar}^{\epsilon} \hspace{8mm} (\because 1-\pbar<1) $

$\then$  Inequality (\ref{eqn:lem1_boundeND}) holds: ~~$\pbar^\epsilon (1-\pbar)^{1-\epsilon} \leq e^{- \A}$
\hgap$\then$ $e^{-ND(\epsilon||\pbar)} \leq \delta$

By step 2 and step 3, we can easily get that 
if $\pbar\leq \inv{1+e^{\A / \epsilon}}$, then $\P(\xibar\leq \epsilon) \geq 1-\delta$, which is the results we want. 

\vgap
\hwem{The proof of Lemma \ref{lemma:probBound}.(2):} We will also finish the proof in several steps:

Assume 
\begin{eqnarray*}
\inv{\constF} \leq \pbar < 1
\end{eqnarray*}

\vgap
\emph{Step 1. We want to show \qquad $\pbar>\epsilon$ }

We show it by proving as follows:
\begin{eqnarray*}
&& \inv{\constF} > \epsilon\\
&\iff& \constF < \inv{\epsilon}\\
&\iff& \constA > (1-\epsilon)\ln{\epsilon\over 1-\epsilon}\\
&\iff& \epsilon\ln\inv{\epsilon} + (1-\epsilon)\ln\inv{1-\epsilon}+ \inv{N}\ln\inv{\delta} > (1-\epsilon)\ln\inv{1-\epsilon} - (1-\epsilon)\ln\inv{\epsilon}\\
&\iff& \ln\inv{\epsilon} + \inv{N}\ln\inv{\delta} > 0
\end{eqnarray*}
which is of course true since $\epsilon, \delta \in (0,1)$. Therefore, we have proved $\pbar>\epsilon$.

\vgap
\emph{Step 2. We want to show \qquad $\P(\xibar \leq \epsilon)\leq e^{-ND(\epsilon||\pbar)}$ }

By the Bernstein-Chernoff-Hoeffding Theorem, since $\epsilon < \pbar=\E\xibar $ and $\xi_j \sim$ Bernoulli($\pj$) independently,  we can directly prove this step.

\vgap
\emph{Step 3. we want to show \qquad $e^{-ND(\epsilon||\pbar)} < \delta$}

Note that 
\begin{eqnarray}
e^{-ND(\epsilon||\pbar)} < \delta 
&\iff& D(\epsilon||\pbar) > \inv{N}\ln\inv{\delta}\nonumber\\
&\iff& \pbar^\epsilon(1-\pbar)^{1-\epsilon} < \exp\kua{- \constA} = e^{-\A} \label{eqn:lem1_step3_1}
\end{eqnarray}

From the condition we have
\begin{eqnarray}
\pbar\geq \inv{1+ \exp{- {\A \over 1-\epsilon}}} 
&\then& \downratio{\pbar} \leq e^{- {\A \over 1-\epsilon}}  \nonumber \\
&\then& \kua{\downratio{\pbar}}^{1-\epsilon}	\leq \exp\kua{-\constA} \label{eqn:lem1_step3_2}
\end{eqnarray}

And note that 
\begin{eqnarray}
\pbar^\epsilon(1-\pbar)^{1-\epsilon} = \kua{\downratio{\pbar}}^{1-\epsilon} \cdot \pbar < \kua{\downratio{\pbar}}^{1-\epsilon}	\qquad (\because \pbar< 1) \label{eqn:lem1_step3_3}
\end{eqnarray}

By combining inequalities (\ref{eqn:lem1_step3_2}) and (\ref{eqn:lem1_step3_3}), we can prove inequality (\ref{eqn:lem1_step3_1}). 
Thus we obtained $e^{-ND(\epsilon||\pbar)} < \delta$
\emph{Finally, } by step 2 and 3, we get :
if $\pbar\geq \inv{\constF}$, then 
$\P\kua{\xibar \leq \epsilon}\ <  \delta$

\end{proof}

Now, we are going to prove Theorem \ref{thm:MostGeneral_DeltaExplicit} with the results we obtained in Lemma \ref{lemma:probBound}
\begin{proof}\hwem{of Theorem \ref{thm:MostGeneral_DeltaExplicit}}

Let $\zeta_j = \I{\hyj\neq \yj} \sim $ Bernoulli($1-\theta_j$) and let $\pbar= \E\bar{\zeta}= \inv{N}\sumj\E\zeta_j= 1-\thbar$.

\vgap
\hwem{The proof of Theorem \ref{thm:MostGeneral_DeltaExplicit}.(1):}

Assume that ${\tone } \geq \sqrt{2\ln\braket{(\nL-1)\C}}$, where $\C=\constC$,  then $\tone  \geq 0$. 

By Theorem \ref{thm:MostGeneral_meanErrorRate}, we have
\begin{eqnarray}
\thbar = 1 - \MER ~\geq~ 1 - (\nL-1)e^{-{\tone^2\over 2}} \label{ineq:general_6}
\end{eqnarray}

Let $\A=\constA$, then 
\begin{eqnarray}
&& {\tone} \geq \sqrt{2\ln\braket{(\nL-1)\C}} = \sqrt{2\ln\braket{(\nL-1)(1+\exp(\A/\epsilon))}}\nonumber \\
&\then& (\nL-1)\exp\kua{-{\tone^2\over 2}} \leq  \inv{1+\exp\kua{\A \over \epsilon}}\nonumber \\
&\then& \thbar \geq 1 -(\nL-1)\exp\kua{-{\tone^2\over 2}} \geq 1-\inv{1+\exp\kua{\A \over \epsilon}} \qquad(\because (\ref{ineq:general_6}))\nonumber \\
&\then& 1- \thbar \leq \inv{1+\exp\kua{\A \over \epsilon}}.
\label{ineq:general_3}
\end{eqnarray}

 By inequality (\ref{ineq:general_3}) and by Lemma \ref{lemma:probBound}, we have 
$$
\P(\bar{\zeta}\leq \epsilon) \geq 1-\delta
$$
which is to say,
$$
\PErrUpDelta
$$
Therefore, we have proved (1).

\vgap
\hwem{The proof of Theorem \ref{thm:MostGeneral_DeltaExplicit}.(2):}
\vspace{2mm}

Assume 
$
\ttwo \leq -\sqrt{2\ln \G} \leq 0
$
, where $\G= 1+\exp\kua{\inv{1-\epsilon}\constA}$. 
Then by Theorem \ref{thm:MostGeneral_meanErrorRate}.(2), we have 
\begin{eqnarray}
\thbar = \inv{N}\sumj \P(\hyj=\yj) \leq \exp\kua{-{\ttwo^2\over 2}} \label{ineq:general_7}
\end{eqnarray}

From the conditions in (2)
\begin{eqnarray*}
&&\ttwo \leq -\sqrt{2\ln\kua{\constG}}\\
&\then& \exp\kua{-{\ttwo^2\over 2}} \leq \inv{1+ \exp\kua{\A \over 1- \epsilon}}\\
&\then& 1- \thbar \geq 1- \exp\kua{-{\ttwo^2\over 2}}  \geq 1- \inv{1+ \exp\kua{{\A\over 1-\epsilon}}} = \inv{1+ \exp\kua{ - { \A\over 1-\epsilon}}}
\end{eqnarray*}

By Lemma \ref{lemma:probBound}.(2), we have 
$
\P(\bar{\zeta} \leq \epsilon) < \delta
$
which implies the desired result.

\end{proof}

\subsection{Proof of Corollary \ref{cor:oracleMAP_gds} (Error rate bounds of the oracle MAP rule) }

\begin{proof} The posterior distribution is

$$
\rhojk = \frac{\prik\likej_k}{\sumlL\pri_l\likej_l}, ~~\forall j\in\N, k\in\Labset,
$$
where 
$$
\likej_k= \prod_{i=1}^M \prod_{h=1}^\nL \kua{\cmikh}^{\I{\zij=h}}.
$$

For the oracle MAP classifier, 
\begin{eqnarray*}
\hyjoracle\= \argmax_\kinL \rhojk
= \argmax_\kinL \prik\likej_k
= \argmax_\kinL \log(\likej_k) + \log\prik
\\
\= \sumi \sumhL \kua{\log\cmikh} \I{\zij=h} + \log\prik
= \sumi\sumhL \muikh\I{\zij=h}+\ak,
\end{eqnarray*}
where $\muikh= \log\cmikh$ and $\ak=\log\prik$.
Therefore the oracle MAP rule is a form of the general \predrule (\ref{def:generalPredBrief}).
Thus all the results of error rate bounds  in Section \ref{sec:Mainresults} holds for the \omaprule.
\end{proof}

\subsection{Proof of Corollary \ref{cor:msr_MAP_hds}
\hwem{(The oracle MAP rule under \hds model)}}

\begin{proof}
The \hds model is the special case of the \gds model,  in which case we have $\cmikk=\wi$ and $\cmikh=\frac{1-\wi}{\nL-1}$ for all $k,h\in\Labset, k\neq h$. By Corollary \ref{cor:oracleMAP_gds}, we can replace $\muikk$ with  $\log\wi$, and $\muikh, k\neq h$ with $\log\frac{1-\wi}{\nL-1}$, and so we have 
\begin{eqnarray*}
\hyj \= \argmax_\kinL \sumi\sumhL \muikh\I{\zij=h} + \log\inv{\nL} \\
\= \argmax_\kinL \sumi \kua{ \I{\zij=k}\log\wi + \I{\zij\neq k,0}\log\frac{1-\wi}{\nL-1}} \\
\= \argmax_\kinL \sumi \kua{ \I{\zij=k}\log\wi + (\I{\zij\neq 0}-\I{\zij= k})\log\frac{1-\wi}{\nL-1}} \\
\= \argmax_\kinL \braket{ \sumi \I{\zij=k} \log\frac{(\nL-1)\wi}{1-\wi} + \sumi\I{\zij\neq 0}\log\frac{1-\wi}{\nL-1} } \\
\= \argmax_\kinL  \sumi \kua{\log\frac{(\nL-1)\wi}{1-\wi}}\I{\zij=k} \\
\= \argmax_\kinL \sumi \vi\I{\zij=k},
\end{eqnarray*}
where $\vi=\log\frac{(\nL-1)\wi}{1-\wi}$. Therefore the oracle MAP rule under the \hds model is a MWV rule. Thus the results from Corollary \ref{cor:wmv_merBound} can be directly applied here. i.e., 
\begin{eqnarray*}
&& \tone =  \frac{\qs}{(\nL-1)\vnorm}\sumi\vi(L\wi-1),~~ ~~
 \cH= \frac{\|\vweight\|_\infty}{\vnorm}\connect \sigtwo= \qs.
\end{eqnarray*}
And if $\tone \geq 0$, then 
$$
\MER \leq \LMerUpBound.
$$
Now all we need to show is that $\tone$ is always non-negative. We can see that if $\wi\geq \inv{\nL}$, then $(\nL\wi-1)\geq 0$ and $\vi= \frac{(\nL-1)\wi}{1-\wi}\geq 0$ for all $i\in\M$, then $\tone \geq 0$ in this case. $\wi<\inv{\nL}$, then $(\nL\wi-1)<0$ and $\frac{(\nL-1)\wi}{1-\wi}< 0$, thus $\tone >0$ in this case as well. All in all, $\tone\geq 0$ is always true for the \omaprule under the \hds model.
\end{proof}

\section{Proof of \thmosWMVBound: error rate bounds of one-step Weighted Majority Voting}
\label{app:oneStepWMVBound}

In the proof of this result, we focus on $\P(\Tij=1)=\qs=1, \forall i\in\M, j\in\N$, i.e., every worker labels any item with probability $\qs$. Meanwhile, we assume $\nL=2$, and the label set $\Labset \defas \hua{\pm 1}$.  It's not hard to generalize our results to $\qs\in(0,1]$ and general $\nL$ case, which is more practical, but the bound will be much more complicated. We omit it here for clarity. 

The prediction of the one-step Weighted Majority Voting for the $j$th item is 
\begin{eqnarray}\label{eqn:yjwmv}
\yjwmv= \sign\kua{\sumi (2\pih -1) \zij},
\end{eqnarray}
 where $\pih$ is the estimated worker accuracy by taking the output from majority voting as ``true" labels. That is to say, 
\begin{eqnarray}\label{eqn:pimv}
\pih = \inv{N}\sumj\I{\zij = \yjmv},
\qquad
\textrm{where}
\qquad
\yjmv=\sign\kua{\sumi\zij}.
\end{eqnarray}
 
Note that the average accuracy of workers is $\wbar= \inv{M}\sumi\wpi$.


In fact, \thmosWMVBound is a direct implication by the following result. 

\begin{prop}\label{res:BoundJthWMV}
If $\wbar \geq \inv{2}+\inv{M}+\sqrt{\frac{(M-1)\ln2}{2M^2}}$, then the mean error rate of one-step Weighted Majority Voting for the $j$th item will be:
\begin{eqnarray}
\P\kua{\yjwmv\neq \yj} \leq \finalUpBound,
\end{eqnarray}
where $\pdiv$ and  $\seta$ is as defined in \thmosWMVBound.

\end{prop}

In the next section, we will focus on proving this result first, then \thmosWMVBound can be obtained directly.

\subsection{The preparation for the proofs}
Before we prove Proposition \ref{res:BoundJthWMV}, we need to prove several useful results for our final proof of Proposition \ref{res:BoundJthWMV}. 

With the same notations as we have used for proving Theorem 1, we simplified some notations as follows for convenience and clearance:
\begin{eqnarray}
\Pjpos(~\cdot~) &\defas& \P(~\cdot~|\yj=\hpos)
\connect
\Pjneg(~\cdot~) \defas \P(~\cdot~|\yj=\hneg),\\
\Ejpos[~\cdot~] &\defas& \E[~\cdot~|\yj=\hpos]
\connect
\Ejneg[~\cdot~] \defas \E[~\cdot~|\yj=\hneg],
\end{eqnarray}
where \lq\lq{}$~\cdot~$\rq\rq{} denotes any event belonging to the $\sigma$-algebra generated by $Z$.  


The following lemma enable us to bound the probability of where the majority vote of the $j$th item agrees with the label given by the $i$th worker given the true label and $\zij$.   

\begin{lem}\label{res:AgreeMVZijCond}
$\forall j\in\N$ and $\forall i\in\M$, we have 
\\
(1) if $\wbar>\inv{2}$, then 
\begin{eqnarray}\label{bound:mvAgreeCondYjZij}
\P\kua{\yjmv=\hpos | \yj=\hpos, \zij=\hpos} \geq 1 - \expLowBound,
\end{eqnarray}
and the same bound holds for $\P\kua{\yjmv=\hneg|\yj=\hneg, \zij=\hneg}$.
\\
(2) if $\wbar \geq \inv{2}+\inv{M}$, then 
\begin{eqnarray}\label{bound:mvDisagreeCondYjZij}
\P\kua{\yjmv=\hneg | \yj=\hpos, \zij=\hneg} \leq \expUpBound,
\end{eqnarray}
and the same bound holds for $\P\kua{\yjmv=\hpos | \yj=\hneg, \zij=\hpos}$.

\end{lem}

\begin{proof}
(1) Notice that for any $i\in\M$ and $j\in\N$, given $\yj$,  $\zij$ is independent of $\hua{\zlj}_{l\neq i}$, ~ $\Ejpos\zlj= 2\wpl -1$ and $\Ejneg\zlj= -(2\wpl-1)$, then
\begin{eqnarray}\label{eqn:ejpos_zlj}
\sumlni\Ejpos\zlj +1 = \sumlni(2\wpl-1) +1 = 2M\kua{\wbar-\inv{2} + \frac{1-\wpi}{M}} >0 ,\quad 
\end{eqnarray} 
since $\wbar-\inv{2} + \frac{1-\wpi}{M}> \wbar-\inv{2}>0$. 
Therefore, we can apply the Hoeffding inequality to get:
\begin{eqnarray*}
\P\kua{\yjmv=\hpos | \yj=\hpos, \zij=\hpos} 
\= \Pjpos\kua{\yjmv=\hpos|\zij=\hpos} \\
\= \Pjpos\kua{\suml\zlj>0 \bigg|\zij=\hpos} \\
\= \Pjpos\kua{\sumlni\zlj - \sumlni\Ejpos\zlj > -(\sumlni\Ejpos\zlj + 1)} \\
&\geq& 1-\exp\kua{-\frac{[\sumlni\Ejpos\zlj + 1]^2}{2(M-1)}} \qquad (\textrm{by Hoeffding})\\
&=& 1 - \expLowBound \qquad (\textrm{by  (\ref{eqn:ejpos_zlj})})
\end{eqnarray*}

Note that with the same argument, 
\begin{eqnarray*}
\P(\yjmv=\hneg|\yj=\hneg, \zij=\hneg) 
\= \Pjneg\kua{\suml\zlj < 0 |\zij=\hneg} \\
\= \Pjneg\kua{\sumlni\zlj -\sumlni\Ejneg\zlj < -\sumlni\Ejneg\zlj + 1} \\
&\geq& 1 -\expLowBound ,
\end{eqnarray*}
provided that 
$
-\sumlni\Ejneg\zlj +1 = 2M\kua{\wbar-\inv{2}+\frac{1-\wpi}{M}} > 0,
$
which is satisfied by the assumption $\wbar>\inv{2}$.

(2) With the same argument as above, notice that 
$
\sumlni\Ejpos\zlj - 1= 2M\kua{\wbar - \inv{2} -\frac{\wpi}{M}} \geq 0
$
because $\wbar \geq \inv{2}+\inv{M}$.

By applying the Hoeffding inequality:
\begin{eqnarray*}
\P\kua{\yjmv=\hneg|\yj=\hpos,\zij=\hneg} 
\= \Pjpos(\sumlni \zlj< 0|\zij=\hneg) \\
\= \Pjpos\kua{\sumlni \zlj - \sumlni \Ejpos\zlj < - (\sumlni\Ejpos\zlj - 1)} \\
&\leq& \exp\kua{-\frac{[\sumlni\Ejpos\zlj-1]^2}{2(M-1)}} \\
\= \expUpBound
\end{eqnarray*}

Following the same argument, we can show that the same bound holds for \\
$\P\kua{\yjmv=\hpos |\yj=\hneg,\zij=\hpos}$. 

\end{proof}

Our next lemma will bound the probability that the label of item $j$ given by worker $i$ agrees with Majority Voting. 

\begin{lem}\label{res:SingleLabelAgreeMV}
Given $\wbar\geq \inv{2}+\inv{M} $, then $\forall j\in\N$, we have
\begin{eqnarray}\label{ineq:mvConditonMatch}
\wpi -\xione \leq \P(\zij=\yjmv \given \yj)\leq \wpi+\xitwo,  \quad
\end{eqnarray}
where 
$\xione= \wpi\expLowBound$ and $\xitwo=(1-\wpi)\expUpBound$
Furthermore, we have 
\begin{eqnarray}\label{ineq:mvMatch}
\wpi -\xione \leq \P(\zij=\yjmv)\leq \wpi+\xitwo,  \quad
\end{eqnarray}
\end{lem}

\remark This result implies that under mild conditions, the probability that the label of the $j$th item given by the $i$th worker matches the majority vote will be close to $\wpi$, i.e., the accuracy of this worker. As the number of workers increase, it will be closer and closer. Intuitively, this makes sense since if $M$ is large, majority voting will be close to the true label if $\wbar > 0.5$. 

\begin{proof}
\begin{eqnarray}\label{eqn:singleMV_1}
\P(\zij=\yjmv)= \pi\Pjpos(\zij=\yjmv) + (1-\pi)\Pjneg(\zij=\yjmv).
\end{eqnarray}

\begin{eqnarray*}
\Pjpos(\zij=\yjmv) \= \wpi\Pjpos(\yjmv=\hpos | \zij=\hpos)  + (1-\wpi)\Pjpos(\yjmv=\hpos|\zij=\hneg). 
\end{eqnarray*}
Applying $\Pjpos(\yjmv=\hpos | \zij=\hpos) \geq 1 - \expLowBound$ from Lemma \ref{res:AgreeMVZijCond}.(1) and $(1-\pi)\Pjneg(\zij=\yjmv)\geq 0$ we can get 
\begin{eqnarray}\label{eqn:singleMV_2}
\Pjpos(\zij=\yjmv) \geq \wpi - \wpi\expLowBound.
\end{eqnarray}
By applying $\Pjpos(\yjmv=\hpos | \zij=\hpos) 
\leq 1$ and $\Pjneg(\zij=\yjmv) \leq \expUpBound$ from Lemma \ref{res:AgreeMVZijCond}.(2), we can get
\begin{eqnarray}\label{eqn:singleMV_3}
\Pjpos(\zij=\yjmv) \leq \wpi + (1-\wpi)\expUpBound.
\end{eqnarray}
Similarly, we can obtain the same bounds for $\Pjneg(\zij=\yjmv)$, i.e., 
\begin{eqnarray}
\Pjneg(\zij=\yjmv) \geq \wpi-\wpi\expLowBound, \label{eqn:singleMV_4}\\
\Pjneg(\zij=\yjmv) \leq \wpi+(1-\wpi)\expUpBound \label{eqn:singleMV_5}
\end{eqnarray}

Since $\Pjpos(\zij=\yjmv)$ and  $\Pjneg(\zij=\yjmv)$ have the same bounds,  and 
$$
\P(\zij=\yjmv \given \yj)= \I{\yjp1}\Pjpos(\zij=\yjmv) + \I{\yjn1}\Pjneg(\zij=\yjmv),
$$
then (\ref{ineq:mvConditonMatch}) holds. 
Furthermore, since
$$
\P(\zij=\yjmv)= \pi\Pjpos(\zij=\yjmv) + (1-\pi)\Pjneg(\zij=\yjmv),
$$
 then (\ref{eqn:singleMV_2}) to (\ref{eqn:singleMV_5}) implies  (\ref{ineq:mvMatch}) . 

\end{proof}

The next lemma will be crucial for applying concentration measure results to bound $\P(\yjwmv\neq \yj \given \ykall)$.

For measuring the fluctuation of a function $\fj: \hua{0,\pm 1}^{M\times N} \rightarrow \R$ if we change one entry of the \dataMatrix, we define a quantity as follows:
\begin{eqnarray}\label{def:dijstar}
\dijstar \defas \inf\hua{d: |\fjz-\fjzpri|\leq d, \textrm{where $Z$ and $Z'$ only differ on ($\istar,\jstar$)}}. \quad
\end{eqnarray}
The constraints $Z$ and $Z'$ only differ on ($\istar,\jstar$), which means that $\zijstarpri$ is a independent copy of $\zijstar$, and $\zij=\zij'$ for $(i,j)\neq (\istar,\jstar)$.

\begin{lem}\label{res:fjZDiff}
Let $\fj(Z) \defas \fjexpr, ~\forall j\in\N$, where $Z$ is the \dataMatrix and $\pih$ is as defined in (\ref{eqn:pimv}),  with $\dijstar$ defined in (\ref{def:dijstar}), we have
\\
(1)if $\jstar\neq j$, \quad then\quad $\dijstar \leq \frac{2(M-1)}{N}  $;
\\
(2)if $\jstar=j$, \quad then\quad $\dijstar \leq \frac{2(M-1)}{N} + 2 $ .

\end{lem}

\begin{proof}
Since $\zijstarpri$ is an independent copy of $\zijstar$, $\zijstarpri=\zijstar$ or $\zijstarpri=-\zijstar$. When $\zijstarpri=\zijstar$, $Z'=Z$, then of course $|\fj(Z)-\fj(Z')|=0$, which satisfies the inequality trivially. 

Next, we focus on the non-trivial case $\zijstarpri=-\zijstar$.

Note that $\zij=\zij'$ when $(i,j)\neq(\istar,\jstar)$. Let $\yjmv$ be the majority vote of the $j$th column of $Z$, and $\yjmvpri$ be the majority vote by the $j$th column of $Z'$. 

Recall that  $|\fjz-\fjzpri|= \abs{\sumi(2\pih-1)\zij - \sumi(2\pih'-1)\zij'}$.

 If $\jstar \neq j$, then $\zij=\zij'$, $\forall i\in\M$, so, 
\begin{eqnarray}\label{eqn:fzdiff_1}
|\fj(Z) - \fj(Z')| = |\sumi(2\pih - 2\pih')\zij| 
\leq 2\sumi |(\pih-\pih')\zij| 
= 2\sumi |\pih-\pih'| .	 \quad 
\end{eqnarray}

If $\jstar= j$, we have $Z_{i\jstar}=Z_{i\jstar}'$ for $i\neq \istar$, and $\zijstar=-\zijstarpri$,  then,
\begin{eqnarray}
|\fjz-\fjzpri| \= \abs{\sum_{i\neq \istar} 2(\pih-\pih')\zij + 2(\pistarh+\pistarh'-1)\zijstar } \nonumber\\
&\leq& 2\sum_{i\neq \istar}|\pih-\pih'| + 2|\pistarh+\pistarh'-1|
\end{eqnarray}

We can see that the difference $\fjz-\fjzpri$ depends heavily on the two quantities $|\pih-\pih'|$ and $|\pistarh+\pistarh'-1|$.  

Next we bound $|\pih-\pih'|$:
\begin{eqnarray}\label{eqn:fzdiff_2}
|\pih-\pih'| \= \abs{\inv{N}\sumi\kua{\I{\zik=\ykmv} - \I{\zik'=\ykmvpri}}} \nonumber 
= \inv{N}\abs{\I{Z_{i\jstar}=\yjstarmv} - \I{Z_{i\jstar}'=\yjstarmvpri}},
\end{eqnarray}
because $\zik=\zik'$ and $\ykmv=\ykmvpri$ if $k\neq \jstar$.

(a). If $\yjstarmv=\yjstarmvpri$, then by (\ref{eqn:fzdiff_2})
\begin{eqnarray*}
|\pih-\pih'|=
\begin{cases}
0 & \text{ if } i\neq \istar, \\
\inv{N} & \text{ if } i=\istar.
\end{cases}
\end{eqnarray*}
In this case, 
\begin{eqnarray}
\sumi |\pih-\pih'| \= \inv{N}, \connect
\sum_{i\neq \istar}|\pih-\pih'| = 0.
\end{eqnarray}

(b). If  If $\yjstarmv\neq \yjstarmvpri$, then by (\ref{eqn:fzdiff_2})
\begin{eqnarray*}
|\pih-\pih'|=
\begin{cases}
\inv{N} & \text{ if } i\neq \istar, \\
0 & \text{ if } i=\istar.
\end{cases}
\end{eqnarray*}
In this case, 
\begin{eqnarray}
\sumi |\pih-\pih'| \= \frac{M-1}{N}, \connect
\sum_{i\neq \istar}|\pih-\pih'| = \frac{M-1}{N}.
\end{eqnarray}

Now, we are going to bound $|\pistarh+\pistarh'-1|$:
\begin{eqnarray}
|\pistarh+\pistarh'-1| \= \inv{N}\abs{\sumk\kua{\I{Z_{\istar k}=\ykmv} + \I{Z_{\istar k}' = \ykmvpri} -1 }} \nonumber \\
&\leq & \inv{N} \sumk\abs{ \I{Z_{\istar k}=\ykmv} + \I{Z_{\istar k}' = \ykmvpri} -1  }.
\end{eqnarray}

(c). If $\yjstarmv = \yjstarmvpri$, then
\begin{eqnarray*}
\abs{ \I{Z_{\istar k}=\ykmv} + \I{Z_{\istar k}' = \ykmvpri} -1  } =
\begin{cases}
1 & \text{ if } k\neq \jstar, \\
0 & \text{ if } k=\jstar. 
\end{cases}
\end{eqnarray*}

So in this case,
$
|\pistarh+\pistarh'-1| = \frac{N-1}{N}.
$

(d). If $\yjstarmv \neq \yjstarmvpri$, then
\begin{eqnarray*}
\abs{ \I{Z_{\istar k}=\ykmv} + \I{Z_{\istar k}' = \ykmvpri} -1  } =
\begin{cases}
1 & \text{ if } k\neq \jstar, \\
1 & \text{ if } k=\jstar. 
\end{cases}
\end{eqnarray*}

So in this case,
$
|\pistarh+\pistarh'-1| = 1.
$
Putting together all the results above, if $\zijstar \neq \zijstar'$, then we have,

(1') If $\yjstarmv= \yjstarmvpri$, 
\begin{eqnarray*}
|\fjz-\fjzpri|\leq
\begin{cases}
\frac{2}{N} & \text{ if } \jstar\neq j,\\
\frac{2(N-1)}{N} & \text{ if } \jstar = j.
\end{cases}
\end{eqnarray*}

(2') If $\yjstarmv \neq \yjstarmvpri$, 
\begin{eqnarray*}
|\fjz-\fjzpri|\leq
\begin{cases}
\frac{2(M-1)}{N} & \text{ if } \jstar\neq j,\\
\frac{2(N-1)}{N}+2 & \text{ if } \jstar = j.
\end{cases}
\end{eqnarray*}

The upper bound in the case when $\yjmv\neq \yjmvpri$ is also an upper bound for the case when $\yjmv=\yjmvpri$. By the definition of $\dijstar$ and noting that $Z$ can only take a finite number of values, we have 
\begin{eqnarray*}
\dijstar \leq 
\begin{cases}
\frac{2(M-1)}{N} & \text{ if } \jstar\neq j,\\
\frac{2(N-1)}{N}+2 & \text{ if } \jstar = j.
\end{cases}
\end{eqnarray*}

\end{proof}

\remark $\dijstar$ is the smallest upper bound on the difference between $\fjz$ and $\fjzpri$. From the proof, we can see that the bound we get is achievable, thus the bounds are tight. This result basically says that if we change only one entry of the \dataMatrix, the fluctuation of the prediction score function of one-step Weighted Majority Voting, i.e., $\fj(Z)$, will be large if $M$ increases and will decrease if $N$ increases.

\subsection{The proof of Proposition \ref{res:BoundJthWMV} and \thmosWMVBound }

In this section, we will use the lemmas we obtained in the last section to prove Proposition \ref{res:BoundJthWMV} and then easily derive the bounds on the expected error rate of one-step WMV from it. 

\begin{proof} \emph{of Proposition \ref{res:BoundJthWMV}}

\begin{eqnarray*}
\P(\yjwmv\neq \yj)= \sum_{y_1, \cdots, y_N} \P\kua{\yjwmv \neq \yj \given y_1, \cdots, y_N}\cdot \P(y_1, \cdots, y_N)
\end{eqnarray*}

If we can get an unified upper bound on $ \P\kua{\yjwmv=\yj \given y_1, \cdots, y_N}$, say $B$,  which is independent of $\ykall$, then this bound will also be an upper bound of 
$\P(\yjwmv\neq \yj)$ 
since
$$\sum_{y_1, \cdots, y_N} \P\kua{\yjwmv \neq \yj \given y_1, \cdots, y_N}\cdot \P(y_1, \cdots, y_N) \leq \sum_{y_1, \cdots, y_N} B \cdot \P(y_1, \cdots, y_N) = B.
$$

Note that the $\zij$'s are not independent of each other unless conditioned on all the true labels of  $y_1, \cdots, y_N$. For convenience of notation, we define the conditional probability and conditional expectation as follows:
\begin{eqnarray}
\Pcjp(~\cdot~) &\defas& \P\kua{~\cdot~|~\yj=\hpos,  \yknj}, \\
\Pcjn(~\cdot~) &\defas& \P\kua{~\cdot~|~ \yj=\hneg,  \yknj}, \\
\Ecjp[~\cdot~] &\defas& \E\braket{~\cdot~|~ \yj=\hpos,  \yknj}, \\
\Ecjn[~\cdot~] &\defas& \E\braket{~\cdot~|~ \yj=\hneg,  \yknj}, 
\end{eqnarray} 
where ``$\cdot$" denotes any event with respect to the $\sigma$-algebra generated by $Z$ and $\hua{y_j}_{j=1}^N$. Note that in these conditional notations, all true labels of item $k$,  $k\neq j$ remain unknown but are conditioned on, e.g., $\Ecjp[y_{k}]= y_{k}$ for $k\neq j$. 

Notice that 
\begin{eqnarray*}
\P\kua{\yjwmv=\yj \given y_1, \cdots, y_N} \= \I{\yj= \hpos}\cdot \P\kua{\yjwmv=\hneg\given \yj=\hpos, \yknj} \nonumber\\
&& + \I{\yj=\hneg}\cdot \P\kua{\yjwmv=\hpos\given \yj=\hneg, \yknj} \nonumber \\
\= \I{\yj= \hpos}\cdot \Pcjp\kua{\yjwmv=\hneg} \nonumber\\
&& + \I{\yj=\hneg}\cdot \Pcjn\kua{\yjwmv=\hpos} \nonumber \\
\= \I{\yj= \hpos}\cdot \Pcjp\kua{\fjz<0} \nonumber\\
&& + \I{\yj=\hneg}\cdot \Pcjn\kua{\fjz>0}, 
\end{eqnarray*}
where $\fjz= \sumi(2\pih-1)\zij$ and $\pih$ is defined as (\ref{eqn:pimv}). 

We want to provide the upper bound on both $\Pcjp\kua{\yjwmv=\hneg}=\Pcjp\kua{\fjz<0}$ and $\Pcjn\kua{\yjwmv=\hpos}=\Pcjp\kua{\fjz>0}$.

We complete our proof in several steps.

\hwem{Step 1. } Providing an upper bound on $\Pcjp\kua{\fjz<0}$

Once we condition on $\ykall$, all the entries in $Z$ will be independent of each other, and then we can apply McDiarmid Inequality \citep{McDiarmid_1998} to the probability $\Pcjp\kua{\fjz<0}$.  

From Lemma \ref{res:fjZDiff}, we get that if $Z$ and $Z'$ only differ on entry $(\istar,\jstar)$, $\zijstarpri$ is an independent copy of $\zijstar$, and so 
$$
|\fjz-\fjzpri| \leq \dijstar.
$$
Combining this with the results from Lemma \ref{res:fjZDiff} we have
\begin{eqnarray}
\sum_{\istar=1}^M\sum_{\jstar=1}^N \kua{\dijstar}^2 &\leq& M(N-1)\kua{\frac{2(M-1)}{N}}^2 + M\kua{2+ \frac{2(M-1)}{N} }^2 \nonumber\\
&\leq& MN\kua{\frac{2M}{N}}^2 + M\kua{2+\frac{2M}{N}}^2 \nonumber \\
&\leq& \frac{4M}{N^2}\braket{M^2N+(M+N)^2}
\end{eqnarray}

Applying the McDiamid Inequality, we get 
\begin{eqnarray}
\Pcjp\kua{\fjz<0} \= \Pcjp\kua{\fjz - \Ecjp[\fjz] < -\Ecjp[\fjz]} \nonumber \\
&\leq& \exp\kua{-\frac{2\kua{\Ecjp[\fjz]}^2}{\sum_{\istar=1}^M\sum_{\jstar=1}^N \kua{\dijstar}^2}} \nonumber \\
&\leq& \exp\kua{-\frac{N^2\kua{\Ecjp[\fjz]}^2}{2M\braket{M^2N+(M+N)^2}}}, \label{eqn:mainp_1}
\end{eqnarray}
provided $\Ecjp\fjz \leq 0$. 

Now, if we can provide a lower bound of $\Ecjp\braket{\fjz}$, then by replacing $\Ecjp\braket{\fjz}$ with that lower bound  in the last inequality, we can further bound $\Pcjp\kua{\fjz<0}$ from above.  Next, we aim at deriving a good lower bound of $\Ecjp\braket{\fjz}$. 

We can expand $\fjz$ so that
\begin{eqnarray*}
\Ecjp\braket{\fjz} = \Ecjp\braket{\sumi(2\pih-1)\zij} 
\= 2\sumi \Ecjp\braket{\pih\zij} - \sumi \Ecjp\zij\\ 
\= 2\sumi\Ecjp[\pih\zij]- \sumi(2\wpi -1),
\end{eqnarray*}
since $\Ecjp\zij= \E\braket{\zij \given \yj=\hpos, \yknj}=  \E[\zij|\yj=\hpos] = 2\wpi-1$. 

Note that for any $i\in\M$ and $j\in\N$, given $\yj$, $\hua{\zij}_{i=1}^M$ will be independent of $\hua{Z_{lk}}_{k\neq j}$ and $\yknj$. We will use this property for dropping all the irrelevant conditioned $y_k$'s.
\begin{eqnarray}
\Ecjp[\pih\zij] \= \Ecjp\braket{\zij\cdot\inv{N}\sumk \I{\zik=\ykmv}}
= \inv{N}\sumk\Ecjp\braket{\zij\I{\zik=\ykmv}} \nn\\
\= \inv{N}\braket{\sumknj\Ecjp\braket{\zij\I{\zik=\ykmv}} + \Ecjp\braket{\zij\I{\zij=\yjmv}}}
\end{eqnarray}  
When $k\neq j$, $\zik$ and $\ykmv$ are independent of $\zij$ given $\yj$. 
\begin{eqnarray*}
&& \Ecjp\braket{\zij\I{\zik=\ykmv}} 
= ~~\Ecjp\zij \cdot \Ecjp\I{\zik=\ykmv} \\
\= (2\wpi-1)\P(\zik=\ykmv\given y_k) \\
&\geq& \I{2\wpi-1\geq 0}\cdot (2\wpi-1)\wpi\braket{1-\expLowBound} \\
&& + \I{2\wpi-1<0}\cdot(2\wpi-1)\wpi\braket{1+\frac{1-\wpi}{\wpi}\expUpBound} \qquad (\text{By Lemma \ref{res:SingleLabelAgreeMV}})\\
&\geq& (2\wpi-1)\wpi\braket{\I{\wpi\geq \inv{2}} - \I{\wpi\geq \inv{2}} \expLowBound} \\
&& + (2\wpi-1)\wpi\braket{\I{\wpi<\inv{2}}+\I{\wpi<\inv{2}} \kua{\frac{1-\wpi}{\wpi}} \expUpBound} \\
\= (2\wpi-1)\wpi\braket{1+ \frac{1-2\wpi}{\wpi}\expUpBound} \qquad (\text{Because } \wbar>\inv{2}+\inv{M}) \\
\= (2\wpi-1)\wpi\braket{1 + \frac{1-2\wpi}{2\wpi}\seta_i},
\end{eqnarray*}
where $\seta_i= 2\expUpBound$. 

Furthermore,
\begin{eqnarray*}
&& \Ecjp\braket{\zij\I{\zij=\yjmv}} 
= ~~\Ejpos \bk{\zij\Ecjp\bk{\I{\zij=\yjmv\given \zij}}} \\
\= \wpi\P(\yjmv=\hpos\given \yj=\hpos, \zij=\hpos) - (1-\wpi)\P(\yjmv=\hneg\given \yj=\hpos, \zij=\hneg) \\
\= \wpi\bk{1-\expLowBound} - (1-\wpi)\expUpBound \quad (\text{By Lemma \ref{res:AgreeMVZijCond}}) \\
&\geq& \wpi\bk{1-\expUpBound} - (1-\wpi)\expUpBound \qquad (\text{As } \wbar>\inv{2}+\inv{M}) \\  
\= \wpi - \expUpBound \\
&\geq& \wpi - \seta_i/2 
\end{eqnarray*}

Combining the two bounds above, we obtained
\begin{eqnarray*}
\Ecjp[\pih\zij] 
\= \inv{N}\braket{\sumknj\Ecjp\braket{\zij\I{\zik=\ykmv}} + \Ecjp\braket{\zij\I{\zij=\yjmv}}} \\
&\geq& \inv{N}\bk{(N-1)(2\wpi-1)\wpi(1+\frac{1-2\wpi}{2\wpi}\seta_i) + \wpi - \frac{\seta_i}{2}} \\
\= \inv{2N}\bk{\kua{(N-1)(2\wpi-1)^2+1}(1-\seta_i) + N(2\wpi-1) } \\
&\geq& \inv{2N}\bk{N(2\wpi-1)^2(1-\seta_i) + N(2\wpi-1)} \qquad (\text{Because } 1\geq (2\wpi-1)^2) \\
\= \inv{2}(2\wpi-1)^2(1-\seta_i) + \inv{2}(2\wpi-1)
\end{eqnarray*}

Let $\seta = 2\halfseta$, so $\seta \leq \seta_i~\forall i\in\M$
\begin{eqnarray}
\Ecjp\fjz \= 2\sumi\Ecjp\bk{\pih\zij} - \sumi(2\pih-1) \nonumber  \\
&\geq& \sumi(2\wpi-1)^2(1-\seta_i)+ \sumi(2\pih-1)  - \sumi(2\pih-1)  \nonumber \\
&\geq& (1-\seta)\sumi (2\wpi-1)^2 \nonumber \\
\= 4M\pdiv^2(1-\seta), \label{eqn:mainp_2}
\end{eqnarray}
where $\pdiv= \sqrt{\inv{M}\sumi(2\wpi-1)^2}$.

Since $\wbar \geq \inv{2} +\inv{M} + \sqrt{\frac{(M-1)\ln 2}{2M^2}}$, so $\seta\leq 1$, which implies  $\Ecjp\fjz \geq 0$.

Then by (\ref{eqn:mainp_1}) and (\ref{eqn:mainp_2}) we have 
\begin{eqnarray}
\Pcjp{\fjz<0} &\leq& \exp\kua{-\frac{N^2\kua{\Ecjp[\fjz]}^2}{2M\braket{M^2N+(M+N)^2}}} \nonumber \\
&\leq& \finalUpBound. \label{eqn:mainp_3}
\end{eqnarray}

\hwem{Step 2. } With the same argument and following the same logic, we can obtain the same upper bound for $\Pcjn{\fjz>0}$.

\hwem{Step 3.} Combining the results we obtained from Step 1 and Step 2. 

Since 
\begin{eqnarray*}
&& \P\kua{\yjwmv=\yj \given y_1, \cdots, y_N} 
= \I{\yj= \hpos}\cdot \Pcjp\kua{\fjz<0} 
 + \I{\yj=\hneg}\cdot \Pcjn\kua{\fjz>0}
\end{eqnarray*}
and both $\Pcjp\kua{\fjz<0}$ and $\Pcjn\kua{\fjz>0}$ have the same upper bound, we have that 
$$
\P\kua{\yjwmv=\yj \given y_1, \cdots, y_N} \leq \finalUpBound.
$$
The upper bound above does not depend on the value of $\ykall$. By what we have discussed in the very beginning of the proof,
$$
\P(\yjmv\neq \yj) \leq \finalUpBound.
$$

\end{proof}

Now, we can directly prove \thmosWMVBound as follows:
\begin{proof} \hwem{(Proof of \thmosWMVBound)}\\
Since the upper bound of $\P(\yjmv\neq \yj) $ doesn't depend on $j$, it can directly imply that 
$$
\inv{N}\sumj\P(\yjmv\neq \yj) \leq \finalUpBound,
$$
which is the desired result in \thmosWMVBound. 

\end{proof}

\end{document}